\documentclass[letterpaper]{article} 
\usepackage{aaai2026}  
\usepackage{times}  
\usepackage{helvet}  
\usepackage{courier}  
\usepackage[hyphens]{url}  
\usepackage{graphicx} 
\usepackage{natbib}  
\usepackage{caption} 
\usepackage{algorithm}

\usepackage{newfloat}
\usepackage{listings}
\DeclareCaptionStyle{ruled}{labelfont=normalfont,labelsep=colon,strut=off} 
\floatstyle{ruled}
\newfloat{listing}{tb}{lst}{}
\floatname{listing}{Listing}
\usepackage{amsthm}
\usepackage{amsmath}

\newtheorem{theorem}{Theorem}
\newtheorem{proposition}{Proposition}
\newtheorem{corollary}{Corollary}
\newtheorem{lemma}{Lemma}
\newtheorem{definition}{Definition}
\newtheorem{assumption}{Assumption}
\newtheorem{remark}{Remark}

\usepackage{subcaption}
\usepackage{amssymb}
\usepackage{xcolor}     
\usepackage{algpseudocode}
\usepackage{bbding}
\usepackage{multirow}
\usepackage{booktabs}
\usepackage{pifont}
\usepackage{array}
\usepackage{algorithmicx}

\title{ECPv2: Fast, Efficient, and Scalable Global Optimization of Lipschitz Functions}
\author {
    Fares Fourati\textsuperscript{\rm 1},
    Mohamed-Slim Alouini\textsuperscript{\rm 1},
    Vaneet Aggarwal\textsuperscript{\rm 2}
}
\affiliations {
    \textsuperscript{\rm 1} KAUST\\
    \textsuperscript{\rm 2}Purdue University\\
    fares.fourati@kaust.edu.sa, slim.alouini@kaust.edu.sa, vannet@purdue.edu
}

\usepackage{bibentry}

\begin{document}

\maketitle

\begin{abstract}
We propose ECPv2, a scalable and theoretically grounded algorithm for global optimization of Lipschitz-continuous functions with unknown Lipschitz constants. Building on Every Call is Precious (ECP) framework, which ensures that each accepted function evaluation is potentially informative, ECPv2 addresses key limitations of ECP, including high computational cost and overly conservative early behavior. ECPv2 introduces three innovations: (i) an adaptive lower bound to avoid vacuous acceptance regions, (ii) a Worst-$m$ memory mechanism that restricts comparisons to a fixed-size subset of past evaluations, and (iii) a fixed random projection to accelerate distance computations in high dimensions. We theoretically show that ECPv2 retains ECP’s no-regret guarantees with optimal finite-time bounds and expands the acceptance region with high probability. We further empirically validate these findings through extensive experiments and ablation studies. Using principled hyperparameter settings, we evaluate ECPv2 across a wide range of high-dimensional, non-convex optimization problems. Across benchmarks, ECPv2 consistently matches or outperforms state-of-the-art optimizers, while significantly reducing wall-clock time.
\end{abstract}

\begin{links}
\link{Code}{https://github.com/fouratifares/ECP}
\end{links}

\section{Introduction}
\label{sec:introduction}

Global optimization is a long-standing and fundamental challenge in optimization~\cite{torn1989global, horst2013handbook, floudas2014recent, zabinsky2013stochastic, stork2022new}. The goal is to identify the global maximum of a function that may be non-convex, non-smooth, and accessible only through black-box evaluations. Such settings render gradient-based and local methods insufficient, motivating algorithms that can efficiently explore the search space while maintaining broad coverage.

The challenge is further exacerbated in many real-world applications, where each function evaluation may be costly in terms of time, energy, or monetary resources, significantly limiting the number of allowable queries. Nevertheless, global optimization remains critical across a wide range of domains, including robotics~\cite{antonova2023rethinking, d2024achieving}, hyperparameter tuning for machine learning models~\cite{lindauer2022smac3}, and query optimization for black-box large language models~\cite{pmlr-v235-chen24e, pmlr-v235-lin24r, kharrat2024acing}. These demands underscore the need for methods that are not only accurate and robust, but also highly \emph{efficient}.

A natural structural assumption in global optimization is \emph{Lipschitz continuity}, which posits that the objective function varies at a bounded rate. While many real-world functions are Lipschitz continuous~\cite{torn1989global}, the true Lipschitz constant is typically unknown. Prior methods such as DIRECT~\cite{jones1993lipschitzian}, AdaLIPO~\cite{malherbe2017global}, and AdaLIPO+~\cite{serre2024lipo+} address this by either deterministically partitioning the search space or estimating the Lipschitz constant through uniform random sampling. However, these approaches often spend substantial evaluation budget on regions unlikely to contain the maximum, thereby reducing sample efficiency.

To address this, Every Call is Precious (ECP) algorithm~\cite{fourati25ecp} was proposed as a conservative and principled alternative to previous strategies. ECP evaluates a point only if it lies within an adaptively defined region of potential maximizers. This acceptance region is initialized conservatively, possibly even empty, and is progressively expanded until a viable candidate is found. By avoiding uniformly random exploration, ECP ensures that every function evaluation is justified and potentially informative. The algorithm enjoys no-regret guarantees and has demonstrated strong empirical performance across a wide range of synthetic and real-world tasks, consistently outperforming state-of-the-art methods from Lipschitz, Bayesian, evolutionary, and bandit optimization paradigms.

Despite its strong theoretical guarantees and empirical performance, ECP has several practical limitations. Its per-iteration cost scales linearly with both the number of previously evaluated points and the problem dimension, which can become prohibitive in long-horizon or high-dimensional settings. Moreover, its optimistic acceptance rule may overly constrain exploration in early iterations, leading to the premature rejection of potentially promising candidates. These limitations motivate the refinements introduced in this work.

\subsection*{Contributions}

\begin{enumerate}
    \item We propose ECPv2, a scalable global optimization algorithm that enhances the efficiency and practical applicability of ECP~\cite{fourati25ecp}, particularly in high-dimensional settings and under modest or large evaluation budgets. ECPv2 introduces three key innovations:
    \begin{itemize}
        \item An \emph{adaptive lower bound} $\varepsilon_t^\oslash$, theoretically derived from the acceptance rule, which avoids infeasible searches and reduces unnecessary early rejections;
        \item A \emph{worst-\( m \) memory mechanism}, which compares candidates only against the \( m \) worst previous evaluations, reducing both computational and memory costs;
        \item A \emph{fixed random projection} applied in the acceptance criterion, reducing the dimensionality $d$ of distance calculations from \( \mathcal{O}(d) \) to \( \mathcal{O}(\log(n)) \), where \( n \) denotes the maximum number of evaluations.
    \end{itemize}

    \item We provide theoretical guarantees for ECPv2. We show that the adaptive lower bound avoids unnecessary rejections; the combined mechanisms define an acceptance region that strictly contains that of ECP with high probability; and the algorithm retains no-regret performance with optimal finite-time bounds, while offering significantly improved computational efficiency.

    \item We empirically validate our theoretical findings through extensive experiments and ablations. We derive principled hyperparameter settings and evaluate ECPv2 on a broad suite of high-dimensional, non-convex optimization problems. Across benchmarks, ECPv2 matches or outperforms state-of-the-art methods, while reducing wall-clock time by a substantial margin.
\end{enumerate}


\section{Related Work}
\label{sec:related}

Global optimization has long been a foundational problem in optimization. Classical non-adaptive methods such as grid search \cite{zabinsky2013stochastic} or Pure Random Search (PRS) \cite{brooks1958discussion, zabinsky2013stochastic} offer simple baselines by exhaustively or randomly exploring the domain. However, these methods are sample-inefficient as they ignore prior evaluations and the structure of the objective function.

Adaptive strategies like evolutionary algorithms such as CMA-ES \cite{hansen1996adapting, hansen2006cma, hansen2019pycma}, simulated Annealing \cite{metropolis1953equation, kirkpatrick1983optimization}, and Dual Annealing \cite{xiang1997generalized, tsallis1988possible} improve over random baselines through adaptive exploration. However, they lack no-regret guarantees and often require substantial tuning and large numbers of evaluations.

Bayesian optimization (BO) \cite{frazier2018tutorial, balandat2020botorch} represents another prominent line of work, constructing a probabilistic model (typically Gaussian Processes) to guide the search via acquisition functions. BO performs well when model assumptions hold, but is sensitive to hyperparameter choices. Hyperparameter-free variants such as A-GP-UCB \cite{JMLR:v20:18-213} and recent approaches such as SMAC3 \cite{JMLR:v23:21-0888} have addressed these issues, but may still struggle in high-dimensional or low-budget regimes. Notably, ECP \cite{fourati25ecp} has been shown to match or even surpass the performance of these approaches on several problems.

To leverage confidence bounds, NeuralUCB \cite{zhou2020neural}, have been adapted to global optimization via repeated sampling and retraining \cite{fourati25ecp}. However, their reliance on data-hungry neural models limits practical efficiency.

To leverage smoothness, adaptive approaches have been proposed to improve efficiency. Tree-based algorithms such as Zooming \cite{kleinberg2008multi}, Hierarchical Optimistic Optimization (HOO) \cite{bubeck2011x}, and Deterministic Optimistic Optimization (DOO) \cite{munos2011optimistic} require known smoothness parameters, whereas Simultaneous Optimistic Optimization (SOO) \cite{munos2011optimistic, preux2014bandits, kawaguchi2016global} and SequOOL \cite{pmlr-v98-bartlett19a} remove this requirement by performing optimistic exploration across a hierarchy of scales. However, these algorithms often underperform in practice due to their discrete space partitioning.

For continuous optimization under unknown Lipschitz constants, the DIRECT algorithm \cite{jones1993lipschitzian, jones2021direct} employs a deterministic space-partitioning strategy, iteratively refining regions with the highest potential upper bounds. While parameter-free, DIRECT is computationally demanding in high dimensions and may become overly conservative.

To address these limitations, AdaLIPO \cite{malherbe2017global} introduced an adaptive stochastic strategy that estimates the Lipschitz constant online via random sampling. It uses this estimate to construct upper confidence bounds and filter potentially optimal points, offering no-regret guarantees. AdaLIPO+ \cite{serre2024lipo+} improves practical performance by reducing exploration over time. However, their reliance on uniform exploration can result in inefficient use of evaluations.

Recently, Every Call is Precious (ECP) algorithm \cite{fourati25ecp} was proposed as a more conservative alternative: it evaluates points only when they are consistent with being global maximizers under a Lipschitz assumption. It avoids redundant evaluations by enforcing a strict geometric acceptance condition, achieving no-regret guarantees while often outperforming existing methods. However, ECP suffers from overly strict rejection of potentially valuable candidates. Moreover, like AdaLIPO and AdaLIPO+, it faces scalability bottlenecks due to the computational cost of distance calculations, which grows linearly with both the problem dimensionality and the evaluation budget. These limitations motivate our proposed algorithm, ECPv2, which improves the scalability and efficiency of ECP while preserving its theoretical guarantees.

In summary, ECPv2 extends the conservative and theoretically grounded framework of ECP by addressing its key scalability challenges. It offers a practical, efficient alternative to state-of-the-art global optimization approaches, suitable for problems where evaluations are expensive.

\section{Global Optimization of Lipschitz Functions}
\label{sec:problem_statement}

We consider the problem of global optimization for a black-box, deterministic, non-convex function 
$
f: \mathcal{X} \rightarrow \mathbb{R}
$
that is expensive to evaluate. The function is defined over a compact, convex subset \(\mathcal{X} \subset \mathbb{R}^d\) with non-empty interior. 

In many practical applications, each function evaluation carries substantial computational, temporal, energetic, or monetary cost, making it essential to extract maximal value from every query. We assume access only to \emph{zeroth-order} oracle calls, that is, we may evaluate \( f(x) \) at chosen points \(x\), but no gradient or higher-order information is available. The objective is to identify a global maximizer
\[
    x^{\star} \in \arg\max_{x \in \mathcal{X}} f(x)
\]
within a finite budget of \(n\) function evaluations.

For the theoretical analysis, same as ECP \cite{fourati25ecp}, we assume a single and minimal assumption:
\begin{assumption}
We assume that \(f\) is Lipschitz-continuous with an \textbf{unknown} finite Lipschitz constant \(k \geq 0\), i.e.,
\[
\forall x,x' \in \mathcal{X}, \quad |f(x) - f(x')| \leq k \cdot \|x - x'\|_2.
\]
\end{assumption}

Let $\text{Lip}(k)$ denote the class of Lipschitz-continious functions with a Lipschitz constant $k$, and let $\text{Lip} := \bigcup_{k \geq 0} \text{Lip}(k)$ be the space of all Lipschitz functions.

Starting from an initial query \(x_1 \in \mathcal{X}\) and observing its evaluation \(f(x_1)\), an adaptive global optimization algorithm proposes a subsequent query \(x_2\) based on this first evaluation. In general, at each step \(i \geq 2\), it suggests the next point \(x_i\) to evaluate using all previous evaluations. After \(n\) steps, it outputs an index \(\hat{\imath} \in \{1, \dots, n\}\) corresponding to the best evaluation observed:
\[
\hat{\imath} \in \arg\max_{i=1,\dots,n} f(x_i).
\]

To evaluate the performance of a global optimization algorithm \( A \in \mathcal{G} \), where \( \mathcal{G} \) is the class of global optimization algorithms, over a function \( f \), we consider its final \emph{regret} after \( n \) iterations:
\begin{equation}
\label{eq:regret}
\mathcal{R}_{A,f}(n) := \max_{x \in \mathcal{X}} f(x) - \max_{i=1,\dots,n} f(x_i),
\end{equation}
which measures the suboptimality of the best point queried so far relative to the unknown global maximum.

We are interested in algorithms that achieve \emph{no-regret} over the \text{Lip} space, i.e., whose regret converges to zero in probability as the number of evaluations increases.
\begin{definition}[No Regret]
\label{def:no-regret}
An algorithm \( A \in \mathcal{G} \) is said to be no-regret over a class of functions \( \mathcal{F} \) if:
\[
\forall f \in \mathcal{F}, \quad \mathcal{R}_{A,f}(n) \xrightarrow{p} 0 \quad \text{as } n \to \infty,
\]
where \( \mathcal{R}_{A,f}(n) \) denotes the regret of algorithm \( A \) with respect to function \( f \) after \( n \) rounds, and \( \xrightarrow{p} \) denotes convergence in probability.
\end{definition}

We also define the \emph{radius} of the domain,
\[
\operatorname{rad}(\mathcal{X})
:= \max \bigl\{ r > 0 : \exists\, x \in \mathcal{X} \text{ such that } B(x,r) \subseteq \mathcal{X} \bigr\},
\]
where \(B(x,r) = \{x' \in \mathbb{R}^d : \|x - x'\|_2 \le r\}\) is the Euclidean ball, and the \emph{diameter}
\[
\operatorname{diam}(\mathcal{X})
:= \sup_{x,x' \in \mathcal{X}} \|x - x'\|_2.
\]

As established in prior work~\cite{bull2011convergence}, there is a fundamental limit on the best achievable regret when optimizing Lipschitz functions, even when the Lipschitz constant is known. From Theorem~1 in~\cite{bull2011convergence}:
\begin{proposition}[Minimax Regret Lower Bound \cite{bull2011convergence}]
\label{prop:minimax}
For any $f \in \text{Lip}(k)$, any $k \geq 0$, and any $n \in \mathbb{N}^{\star}$,
\[
\inf_{A \in \mathcal{G}} \sup_{f \in \text{Lip}(k)} \mathbb{E}[\mathcal{R}_{A,f}(n)] \geq \Omega (\operatorname{rad}(\mathcal{X}) \cdot k \cdot n^{-1/d}).
\]
where $\operatorname{rad}(\mathcal{X})$ denotes the radius of the domain $\mathcal{X}$ and the expectation is taken over the sampling distribution induced by algorithm $\mathcal{A}$ when optimizing $f$.
\end{proposition}

This result underscores the fundamental challenge of high-dimensional global optimization, where the optimal regret decreases only at a rate of $n^{-1/d}$ \cite{bull2011convergence}. It highlights the crucial role of each function evaluation in driving meaningful progress, especially when the number of evaluations $n$ is limited and the dimensionality $d$ is large. 

In this context, the ECP algorithm \cite{fourati25ecp} embraces a conservative optimization philosophy: it \emph{evaluates only points that are likely to be potential maximizers}, employing Lipschitz-based acceptance conditions to filter candidate queries. This selection ensures that each evaluation has a potential to improve the current best solution. Notably, ECP achieves regret guarantees that match the theoretical lower bound, reflecting its strong theoretical foundation. However, its computational complexity grows linearly with both the number of evaluations and the problem dimensionality, which can significantly limit its scalability in large-scale settings. This limitation motivates the present work, which we discuss in detail in the following sections.

\section{Background}
\label{sec:ecp-motivation}

The ECP algorithm \cite{fourati25ecp} was proposed as a principled solution to global optimization when function evaluations are costly. It builds on the idea that each query to the objective function should be a potential optimizer, through a \emph{Lipschitz-based acceptance rule} that governs whether a candidate point \( x \in \mathcal{X} \) should be evaluated.

Concretely, given a current archive \( \{X_1, \dots, X_t\} \) of evaluated points and their values \( \{f(X_1), \dots, f(X_t)\} \), a new candidate \( x \) sampled from the uniform distribution \( \mathcal{U}(\mathcal{X}) \) is accepted for evaluation only if:
\begin{equation}
\label{eq:ecp-rule}
\min_{i=1,\dots,t} \left( f(X_i) + \varepsilon_t \cdot \|x - X_i\|_2 \right) \geq \max_{j=1,\dots,t} f(X_j),
\end{equation}
where \( \varepsilon_t > 0 \) is a time-dependent parameter.

Equation~\eqref{eq:ecp-rule} determines whether a candidate point $x$ remains \emph{compatible} with being the global maximizer under a Lipschitz assumption. If $\varepsilon_t$ were equal to the true (unkown) Lipschitz constant $k$, each quantity $f(X_i) + k\|x - X_i\|$ would represent the tightest upper bound on the unknown value $f(x)$, and the minimum over $i$ would give the most optimistic value that $x$ could still attain. The acceptance condition therefore evaluates $x$ only when this best-case estimate exceeds the current maximum, ensuring that no query is spent on points that are provably suboptimal. 

In practice, $\varepsilon_t$ begins as a small lower bound, yielding a conservative acceptance region, but increases over time toward $k$, progressively relaxing the test while guaranteeing that true maximizers are never discarded on the long horizon. The parameter $\varepsilon_t$ grows geometrically whenever the number of rejected candidates exceeds a threshold \( C > 1 \), by updating \( \varepsilon_t \gets \tau_{n,d} \cdot \varepsilon_t \), with \( \tau_{n,d} > 1 \). 

Rejections are tracked using a counter \( h_{t+1} \), which is initialized to zero and reset whenever \( \varepsilon_t \) is increased. At each iteration, \( h_t \) is set to the number of rejections from the previous step, and if the increase in rejections from one iteration to the next exceeds \( C \), the update is triggered. This mechanism enlarges the effective search space over time, allowing broader exploration. If a candidate fails the acceptance test, it is rejected without evaluation and a new point is sampled from \( \mathcal{U}(\mathcal{X}) \).

\begin{definition}
\label{def:potentially_accepted}
{\sc (ECP Acceptance Region)}
The set of points potentially accepted by ECP at any time step $t \geq 1$, for any value $\varepsilon_t > 0$, is defined as:
$
 \mathcal{A}_{\text{ECP}}(\varepsilon_t,t) \triangleq \{x \in \mathcal{X}: 
 \min_{i=1, \cdots, t}  (f(X_i) + \varepsilon_t \cdot \|x-X_i\|_2)
\geq \max_{j=1, \cdots, t}f(X_j) \}. 
$
\end{definition}

The ECP strategy offers several attractive properties. First, ECP achieves no-regret over Lipschitz functions, ensuring convergence to a global maximizer. Second, it is \emph{resource efficient}: by filtering out unpromising, ECP minimizes wasted calls to the objective function. Finally, ECP supports \emph{adaptive exploration}: its acceptance criterion dynamically adapts to the current state of knowledge, promoting a conservative yet effective exploration of the search space. However, despite these strengths, ECP suffers from some practical challenges that limit its scalability in high-dimensional or long-horizon settings:

\paragraph{Rejection Overhead:}  
When $\varepsilon_t$ is too small, the acceptance region may be empty, causing the algorithm to repeatedly reject proposals before adapting $\varepsilon_t$. While this is sufficient to guarantee that the acceptance region is not larger than enough, it delays progress and increases wall-clock time, especially in the early search stages.

\paragraph{Computational Inefficiency:}  
The acceptance condition in Equation~\eqref{eq:ecp-rule} requires computing distances to \emph{all} $t \leq n$ previous points, leading to at least $\Omega(n^2d)$ overall computational complexity. In high-dimensional spaces or with large evaluation budgets, this becomes a bottleneck.

Toward a scalable alternative and to overcome these limitations while preserving ECP’s conservative foundation, we propose ECPv2, which introduces key innovations, which make ECPv2 fast, efficient, and scalable, while retaining the theoretical guarantees. In the following sections, we formally describe ECPv2 and analyze its theoretical properties and empirical performance.

\section{ECPv2}

ECPv2, outlined in Algorithm~\ref{ALG:ECPv2_ALGORITHM}, takes as input a budget \( n \), a search space \( \mathcal{X} \), and an objective function \( f \). It builds upon the original ECP algorithm by retaining the core acceptance rule, including the geometric growth factor \( \tau_{n,d} > 1 \), the rejection threshold \( C > 1 \), and the initial parameter \( \varepsilon_1 > 0 \).

To improve scalability and performance in high-dimensional or complex settings, ECPv2 introduces three key modifications to the acceptance mechanism. These changes are briefly summarized below, with detailed explanations and theoretical guarantees provided in later sections.

\paragraph{(i) Adaptive lower bound $\varepsilon_t^\oslash \leq \varepsilon_t$.}
The size of the acceptance region at iteration $t$ in ECP is governed by the parameter $\varepsilon_t$, which appears on the left-hand side of the acceptance condition (Equation~\eqref{eq:ecp-rule}). When $\varepsilon_t$ becomes too small, the acceptance region may become vanishingly small or even empty, leading to an excessive number of rejections.

To prevent such overly conservative behavior, ECPv2 introduces a time-dependent lower bound on $\varepsilon_t$:
\begin{equation}
 \varepsilon_t^\oslash := \frac{\max_{i} f(x_i) - \min_{j} f(x_j)}{\operatorname{diam}(\mathcal{X})}.   
\end{equation}
This quantity ensures that the acceptance region is not trivially vacuous. At each step $t$, ECPv2 initializes $\varepsilon_t$ (line 14):
\[
\varepsilon_t = \max\left(\tau_{n,d} \cdot \varepsilon_{t-1}, \varepsilon_t^\oslash\right),
\]
where $\tau_{n,d} > 1$, same as in ECP, controls the growth of $\varepsilon_t$ over time. As shown in Lemma~\ref{lem:lb_eps}, this lower bound is necessary to guarantee that the acceptance region remains meaningfully sized, avoiding unnecessary rejections of promising candidates.  The acceptance condition becomes:
\begin{equation}
\label{new_acceptance_condition_with_lb}
\min_{i \in [t]} \left( f(x_i) + \max\{\varepsilon_t, \varepsilon_t^\oslash\} \cdot \|x - x_i\|_2 \right) \geq \max_{j\in [t]} f(x_j).
\end{equation}
The additional cost of maintaining $\varepsilon_t^\oslash$ is negligible, $\mathcal{O}(1)$ per iteration, since both $\max_i f(x_i)$ and $\min_j f(x_j)$ can be tracked incrementally during the optimization process. Lemma~\ref{lemma:lb_superset} further shows that the acceptance condition remains a superset of the previous acceptance condition.

\paragraph{(ii) Worst-$m$ memory mechanism.}
In addition to bounding $\varepsilon_t$, a major limitation of ECP is that the computational and memory cost of the acceptance condition grows linearly with the number of previous evaluations $t$. Moreover, as $t$ increases, the acceptance condition becomes increasingly conservative, often rejecting candidate points unnecessarily, particularly when $\varepsilon_t$ is still small.

To address these issues, and inspired by the stochastic maximization trick of \cite{pmlr-v235-fourati24a}, which reduces maximization cost by retaining only the most informative actions, ECPv2 restricts the acceptance test to the \(m\) worst-performing previously evaluated points. Formally, define:
\[
\mathcal{I}_t^{m} = \arg \min_{\substack{S \subseteq \{1, \dots, t\} \\ |S| = m}} \sum_{i \in S} f(x_i),
\]
where $\mathcal{I}_t^m$ indexes the $m$ points with the lowest objective values among the $t$ evaluated so far. The integer parameter $m \geq 1$ controls the trade-off between computational efficiency and conservatism: setting $m \geq t$ recovers the ECP condition, while smaller values of $m$ reduce both computational and memory costs.

The modified acceptance condition becomes:
\begin{equation}
\label{new_acceptance_condition_with_m}
\min_{i \in \mathcal{I}_t^{m}} \left( f(x_i) + \max\{\varepsilon_t, \varepsilon_t^\oslash\} \cdot \|x - x_i\|_2 \right) \geq \max_{j \in [t]} f(x_j).
\end{equation}

\paragraph{(iii) Introduce random projection.}
Instead of computing distances in the original space $\mathbb{R}^d$, ECPv2 uses a fixed random projection $\mathbf{P} : \mathbb{R}^d \to \mathbb{R}^{d'}$ to reduce computational cost, with $d'$ being function of the budget $n$ and a controlled maximal distortion $\delta \in [0,1)$. For $\delta = 0$ we consider an identity projection $\mathbf{I}_d$ (no projection) and revert to the case of ECP. For $\delta>0$, we can have an acceleration where projection is applied though a random projection matrix as:
\begin{equation}
\label{projection_definition}
 \mathbf{P} := \mathbf{1}_{\delta>0} \cdot \frac{1}{\sqrt{d'}} R^\top  + (1-\mathbf{1}_{\delta>0}) \cdot \mathbf{I}_d,  
\end{equation} 
where \( R \in \mathbb{R}^{d \times d'} \) is a matrix with i.i.d.\ entries drawn from \( \mathcal{N}(0, 1) \), with $d' = 8 \log(\beta n)/(\delta^2 - \delta^3)$. The choice of this form ensures that each projected vector \( x_i' = \mathbf{P} x_i \in \mathbb{R}^{d'} \) lies in a lower-dimensional subspace while preserving pairwise distances with high probability $(1-1/\beta^2)$. Here, \( \mathbf{1}_{\delta > 0} \) is the indicator function, which equals 1 if \( \delta > 0 \), and 0 otherwise.

The acceptance condition is then evaluated in the lower-dimensional projected space using a scaled parameter
$\tilde{\varepsilon}_t := \frac{\max\{\varepsilon_t,\varepsilon_t^\oslash\}}{\sqrt{1 - \delta}}.$ The modified acceptance rule becomes:
\begin{equation}
\label{new_acceptance_condition_with_proj}
\min_{i \in \mathcal{I}_t^{m}} \left( f(x_i) + \tilde{\varepsilon}_t \cdot \| \mathbf{P} x - \mathbf{P} x_i \|_2 \right) \geq \max_{j \in [t]} f(x_j),    
\end{equation}
where $\delta$ is the distortion due to projection and $\mathcal{I}_t^{m}$ is the Worst-$m$ memory set described above.

\begin{definition}
\label{def:potentially_accepted_v2}
{\sc (ECPv2 Acceptance Region)}
The set of points potentially accepted by ECPv2 at any time step $t \geq 1$, for any value $\varepsilon_t > 0$, for any controlled distortion $\delta \in [0,1)$, any value of $\beta>1$, resulting in a projection matrix $\mathbf{P}$ as define in Equation~\eqref{projection_definition}, and for any integer $m \geq 1$ is defined as:
$
 \mathcal{A}_{ECPv2}(\varepsilon_t, t, m, \mathbf{P}) \triangleq \{x \in \mathcal{X}: 
 \min_{i\in \mathcal{I}_t^{m}}  (f(X_i) + \frac{\max\{\varepsilon_t,\varepsilon_t^\oslash\}}{\sqrt{1 - \delta}} \cdot \|\mathbf{P}x-\mathbf{P}X_i\|_2)
\geq \max_{j=1, \cdots, t}f(X_j) \}. 
$
\end{definition}

The following sections detail the theoretical motivation and guarantees behind each modification.

\begin{algorithm}[t]
\caption{ECPv2}
\textbf{Input:} $n \in \mathbb{N}^{\star}$, $m \geq 1$, $\delta \in [0,1)$, $\beta>1$, $\varepsilon_1 > 0$, $\tau_{n,d} > 1$, $C > 1$, search space $\mathcal{X} \subset \mathbb{R}^d$, objective function $f$
\begin{algorithmic}[1]
\State Initialize projection matrix $\mathbf{P}$ as in Equation~\eqref{projection_definition}
\State Sample $x_1 \sim \mathcal{U}(\mathcal{X})$, evaluate $f(x_1)$
\State $\hat{x}_1 \gets \mathbf{P} x_1$, $t \gets 1$, $h_{1} \gets 1$, $h_{2} \gets 0$
\While{$t < n$}
\State Sample $x_{t+1} \sim \mathcal{U}(\mathcal{X})$, $h_{t+1} \gets h_{t+1} + 1$
\State $\hat{x}_{t+1} \gets \mathbf{P} x_{t+1}$ \Comment{(Projection \textbf{if} $\delta>0$)}
\If{$(h_{t+1} - h_t) > C$} \Comment{(Growth)}
\State $\varepsilon_t \gets \tau_{n,d} \cdot \varepsilon_t$, $h_{t+1} \gets 0$
\EndIf
\If{$x_{t+1} \in \mathcal{A}_{\text{ECPv2}}(\varepsilon_t, t, m, \mathbf{P})$} \Comment{(Acceptance)}
\State Evaluate $f(x_{t+1})$
\State $t \gets t + 1$, $h_t \gets h_{t+1}$
\State $\varepsilon^\oslash_{t+1} \gets \frac{\max_{i \le t} f(x_i) - \min_{i \le t} f(x_i)}{\operatorname{diam}(\mathcal{X})}$
\State $\varepsilon_{t+1} \gets \max\bigl(\tau_{n,d} \cdot \varepsilon_t, \varepsilon^\oslash_{t+1}\bigr)$, $h_{t+1} \gets 0$
\EndIf
\EndWhile
\State \Return $x_{\hat{i}}$ where $\hat{i} \in \arg\max_{i=1,\dots,n} f(x_i)$
\end{algorithmic}
\label{ALG:ECPv2_ALGORITHM}
\end{algorithm}

\section{Adaptive Lower Bound $\varepsilon_t^\oslash$}

In ECP algorithm, a new point $x \in \mathcal{X}$ is evaluated only if it satisfies the acceptance condition in Equation~\eqref{eq:ecp-rule}. This mechanism controls the exploration by controlling the admissibility of new samples. However, when $\varepsilon_t$ is too small, no candidate point may satisfy the acceptance condition, resulting in wasted proposals and inefficiency.

The ECP addresses this issue by gradually increasing $\varepsilon_t$ multiplicatively (i.e., $\varepsilon_t \leftarrow \varepsilon_t \cdot \tau_{n,d}$), which guarantees progress but can incur a large number of rejections before reaching a viable threshold. To reduce unnecessary overhead, we derive a principled lower bound on $\varepsilon_t$ to ensure that not every sampled point is trivially rejected.

\begin{lemma}[Lower Bound on $\varepsilon_t$]
\label{lem:lb_eps}
Let $\mathcal{X} \subset \mathbb{R}^d$ be a compact search space with bounded diameter $\operatorname{diam}(\mathcal{X})$.
Let $\{X_1, \dots, X_t\} \subset \mathcal{X}$ be the set of previously evaluated points. Define: $f_{\max_t} := \max_{i=1,\dots,t} f(X_i)$ and $f_{\min_t} := \min_{i=1,\dots,t} f(X_i).$
If there exists a point $x \in \mathcal{X}$ satisfying the ECP acceptance condition, then $\varepsilon_t$ must satisfy:
$\varepsilon_t \geq \varepsilon^\oslash_t =  \frac{f_{\max_t} - f_{\min_t}}{\operatorname{diam}(\mathcal{X})}.$
\end{lemma}

\noindent This bound (proved in Appendix C) provides for each time step $t$ a theoretically necessary initialization for $\varepsilon_t$ (for not being empty). Instead of relying solely on geometric growth, this lower bound can be used to re-initialize $\varepsilon_t$ adaptively, avoiding large rejection loops without sacrificing convergence or optimality.

For the same $\varepsilon_t$ it is shown (Appendix C) that the acceptance condition with the lower-bound is a super-set of ECP. 
\begin{lemma}
\label{lemma:lb_superset}
Let $\mathcal{A}_{\text{ECP}}(\varepsilon_t, t)$ denote the acceptance region defined by Equation~\eqref{eq:ecp-rule}, and let $\mathcal{A}_t(\varepsilon_t, t)$ denote the region defined by Equation~\eqref{new_acceptance_condition_with_lb}. Then, for any $\varepsilon_t$, we have
\[
\mathcal{A}_{\text{ECP}}(\varepsilon_t, t) \subseteq \mathcal{A}_t(\varepsilon_t, t).
\]
\end{lemma}
This result highlights that ECPv2 does not miss any potential maximizer that would have been accepted by ECP, and may even allow more potential maximizers to be accepted, ones that might have been rejected due to the overly restrictive value of \( \varepsilon_t \). Further empirical validation and discussion of this bound are provided in Appendix~F.

\section{Worst-$m$ Memory Mechanism}
\label{sec:worst-m-memory}

In the ECP algorithm, each new candidate $x \in \mathcal{X}$ is evaluated against \emph{all} previously queried points $\{X_1, \dots, X_t\}$ using the acceptance condition in Equation~\eqref{eq:ecp-rule}. While theoretically sound, this condition becomes increasingly restrictive and computationally costly as $t$ grows, particularly when many historical points are near-optimal or redundant. This scalability bottleneck can hinder both efficiency and exploratory behavior.

To mitigate these issues, and inspired by the stochastic maximization strategy introduced in~\cite{pmlr-v235-fourati24a}, we propose the \emph{Worst-$m$ memory mechanism}, which limits the acceptance check to only the $m$ \emph{worst-performing} points in the archive. Formally, let
\[
\mathcal{I}_t^{m} := \arg\min_{\substack{S \subseteq \{1,\dots,t\} \\ |S| = m}} \sum_{i \in S} f(X_i)
\]
denote the indices corresponding to the $m$ lowest-valued function evaluations among $\{f(X_1), \dots, f(X_t)\}$.

Intuitively, this strategy focuses rejection pressure on clearly suboptimal regions of the search space, thereby avoiding unnecessary constraints from high-performing points that might otherwise hinder exploration. By excluding these points, we obtain a broader acceptance region. Notably, when $m = n$, we exactly recover the ECP condition.

The Worst-$m$ memory rule \emph{relaxes} the original condition, inducing a larger acceptance region, as stated below.

\begin{lemma}
\label{lemma:acceptance-region-with-m}
Let $\mathcal{A}_t(\varepsilon_t, t)$ denote the acceptance region defined by Equation~\eqref{eq:ecp-rule}, and let $\mathcal{A}_t(\varepsilon_t, t, m)$ denote the region defined by Equation~\eqref{new_acceptance_condition_with_m}. Then, for any $m \geq 1$, we have
\[
\mathcal{A}_t(\varepsilon_t, t) \subseteq \mathcal{A}_t(\varepsilon_t, t, m).
\]
\end{lemma}

See Appendix~C for the proof and Appendix~G for empirical analysis of the Worst-$m$ strategy.

\section{Random Projection-Based Acceleration}
\label{sec:jl-theory}

In high-dimensional settings, computing distances between candidate points and the archive of previously evaluated points becomes a dominant computational bottleneck in ECP. To alleviate this, we introduce a \emph{projection-based acceleration} strategy that employs a fixed random projection $\mathbf{P}: \mathbb{R}^d \rightarrow \mathbb{R}^{d'}$ for which we show the following results that are proved in Appendix A.

\begin{lemma}
\label{lemma:distoriton_bound}
Let \( X = \{x_1, x_2, \ldots, x_n\} \subset \mathbb{R}^d \) be a set of \( n \) vectors. Fix a distortion parameter \( \delta \in [0, 1) \). Let the projection matrix $\mathbf{P}$ be generated as in Equation~\eqref{projection_definition}.
Let \( x_i' = \mathbf{P} x_i \) be the projection of each vector \( x_i \). If
$
d' \geq \frac{8 \log(\beta n)}{\delta^2 - \delta^3},
$
then with probability at least \( 1 - \frac{1}{\beta^2} \), the following inequality holds simultaneously for all \( i, j \in \{1, \ldots, n\} \):
\[
(1 - \delta)\|x_i - x_j\|_2^2 \leq \|x_i' - x_j'\|_2^2 \leq (1 + \delta)\|x_i - x_j\|_2^2.
\]
\end{lemma}

The lemma is obtained by adapting classical Johnson-Lindenstrauss (JL) results \cite{arriaga2006algorithmic} describing the concentration properties of random Gaussian projections. These results guarantee that norms and pairwise distances are preserved up to controlled distortion after dimensionality reduction. The specific restated and adapted forms used in our analysis appear in Appendix A.

We now formalize the relationship between the acceptance regions in the original and projected spaces.

\begin{lemma}
\label{lemma:jl-projection-superset}
Let $\mathcal{A}_{ECPv2}(\varepsilon_t, t, m, \mathbf{P})$ denote the acceptance region of ECPv2 in Definition~\ref{def:potentially_accepted_v2} and let $\mathcal{A}_t(\varepsilon_t, t, m)$ denote the region defined by Equation~\eqref{new_acceptance_condition_with_m}. Let $\mathbf{P}$ be the random projection matrix with reduction $d'$ as in Lemma~\ref{lemma:distoriton_bound}. Then, for any $m \geq 1$, we have with probability at least $1-1/\beta^2$:
\[
\mathcal{A}_t(\varepsilon_t, t, m) \subseteq \mathcal{A}_{ECPv2}(\varepsilon_t, t, m, \mathbf{P}). 
\]
\end{lemma}

Extensive empirical analysis of the projection mechanism, including its sensitivity to the parameters $\delta$ and $\beta$, is provided in Appendix~D, while principled values for these parameters are derived in a subsequent section.

Combining Lemmas~\ref{lemma:lb_superset}, \ref{lemma:acceptance-region-with-m}, and \ref{lemma:jl-projection-superset}, we obtain:

\begin{corollary}
\label{corollary:ecp2_superset_of_ecp}
Let $\mathcal{A}_{ECPv2}(\varepsilon_t, t, m, \mathbf{P})$ denote the acceptance region of ECPv2, as defined in Definition~\ref{def:potentially_accepted_v2}, and let $\mathcal{A}_{ECP}(\varepsilon_t, t)$ denote the region defined in Definition~\ref{def:potentially_accepted}. Let $\mathbf{P}$ be a random projection matrix with reduced dimension $d'$, as in Lemma~\ref{lemma:distoriton_bound}. Then, for any $m \leq n$, with probability at least $1 - 1/\beta^2$, the following inclusion holds:
\[
\mathcal{A}_{ECP}(\varepsilon_t, t) \subseteq \mathcal{A}_{ECPv2}(\varepsilon_t, t, m, \mathbf{P}). 
\]
\end{corollary}

Corollary~\ref{corollary:ecp2_superset_of_ecp} establishes that the acceptance region of ECPv2 strictly contains that of ECP with high probability. In particular, the proposed techniques, including the lower bound, the worst-\( m \) mechanism, and the projection-based acceleration, guarantee that any point accepted by ECP will also be accepted by ECPv2 with high-probability, ensuring that no potential maximizers are erroneously discarded due to an overly restrictive acceptance region. 

Importantly, this enlargement is controlled: the distortion \( \delta \) introduced by the projection can be made arbitrarily small with high probability by choosing a sufficiently large projection dimension, specifically \( d' = \mathcal{O}\left( \log(\beta n) / (\delta^2 - \delta^3) \right) \), while the worst-\( m \) mechanism can be adjusted via the choice of \( m \). Consequently, the worst-\( m \) and projection techniques provide mechanisms for accelerating ECP while preserving its conservative decision criteria. In high-dimensional settings, this yields substantial computational gains.

\begin{table*}[h]
\centering
\begin{tabular}{lccc}
\toprule
\textbf{Lipschitz Method} 
& \textbf{Memory} 
& \textbf{Runtime} 
& \textbf{Regret Upper Bound} \\
\midrule
AdaLIPO      
& $\mathcal{O}(nd)$ 
& $\Omega(n^2 d)$ 
& $\mathcal{O}_{1 - \frac{1}{\xi}}\left(
    c^{\ddag} \cdot  p^{-\frac{1}{d}}
    \cdot \left( \ln(3 / \xi) \right)^{\frac{1}{d}} \cdot k \cdot n^{-\frac{1}{d}}
  \right)$
\\

AdaLIPO+    
& $\mathcal{O}(nd)$ 
& $\Omega(n^2 d)$ 
& --- \\

ECP          
& $\mathcal{O}(nd)$ 
& $\Omega(n^2 d)$ 
& $\mathcal{O}_{1 - \frac{1}{\xi}}\left(
    c^\star
    \cdot \left( \ln(1 / \xi) \right)^{\frac{1}{d}} \cdot k \cdot n^{-\frac{1}{d}}
  \right)$
 \\

\textbf{ECPv2 (this work)}  
& $\mathcal{O}((m + d)\log(\beta n))$ 
& $\Omega(n (m + d)\log(\beta n))$ 
& $\mathcal{O}_{1 - \frac{1}{\beta^2} - \frac{1}{\xi}}\left(
    c^\star 
    \cdot \left( \ln(1 / \xi) \right)^{\frac{1}{d}} \cdot k \cdot n^{-\frac{1}{d}}
  \right)$ \\
\midrule
Lower Bound \cite{bull2011convergence} 
& --- 
& --- 
& $\Omega\left(\operatorname{rad}(\mathcal{X}) \cdot k \cdot n^{-\frac{1}{d}}
  \right)$ \\
\bottomrule
\end{tabular}
\caption{
Comparison of Lipschitz global optimization methods with unknown Lipschitz constants. 
$d$: dimension, $n$: evaluation budget, $m \ll n$: number of worst-performing points retained, 
$\beta$: ECPv2 hyperparameter, 
$p$: AdaLIPO sampling probability. \\
$c^\star = \operatorname{diam}(\mathcal{X}) \cdot \log_{\tau_{n,d}}\left(\frac{k}{\varepsilon_1}\right)^{1/d}$, 
$c^{\ddag} = \operatorname{diam}(\mathcal{X}) \cdot \left( 
5 + \frac{2 \ln(\xi / 3)}{\ln(1 - \Gamma(f, k_{i^*-1}))} \right)^{1/d}$.
}
\label{tab:ecp-comparison}
\end{table*}

\section{Computational Complexity Gains}
\label{sec:ecp-complexity}

A central motivation for ECPv2 is to reduce the high computational overhead of Lipschitz-based global optimization, especially in high-dimensional settings or under large evaluation budgets. Table~\ref{tab:ecp-comparison} compares the theoretical complexity of ECPv2 with existing methods, including ECP, AdaLIPO, and AdaLIPO+. The analysis highlight substantial reductions in both runtime and memory requirements.

\paragraph{Runtime Complexity Gains.}

In ECP, AdaLIPO, and AdaLIPO+, each iteration requires checking the acceptance condition against all previously evaluated points. This involves computing $t \leq n$ distances in $\mathbb{R}^d$, leading to a per-iteration complexity of $\Omega(td)$. Over $n$ evaluations, this results in an overall runtime of
$
\Omega(n^2 d),
$
which becomes prohibitive as either $n$ or $d$ increases. Rather than considering all $t = \mathcal{O}(n)$ past evaluations, ECPv2 computes the acceptance condition using only the $m$ worst-performing archive points, reducing the number of distance computations from $t$ to $m \ll n$. Furthermore, before computing distances, the candidate is projected from $\mathbb{R}^d$ to $\mathbb{R}^{d'}$ and the distance is then computed in $\mathbb{R}^{d'}$. Projection incurs a cost of $\Omega(dd')$ per iteration, and each projected distance is computed in $\Omega(d')$ time. With these techniques, the per-iteration complexity becomes
$\Omega(dd' + m d'),$
where we set $d' = \mathcal{O}(\log(\beta n))$. The total runtime complexity over $n$ evaluations is then
$\Omega\bigl(n (m + d) \log(\beta n)\bigr),$
which is linear in $n$ (up to a log factor) and only linear in $d$ when $m$ and $\beta$ are held constant.

\paragraph{Memory Complexity Gains.}

In addition to runtime savings, ECPv2 also provides substantial improvements in memory efficiency. Prior Lipschitz-based methods such as ECP, AdaLIPO, and AdaLIPO+ store the entire evaluation history, each of dimensionality $d$, yielding a memory complexity of $\mathcal{O}(nd)$. ECPv2, by contrast, only stores: A projection matrix of size $\mathcal{O}(d \log(\beta n))$ and the $m$ worst-performing projected points, requiring $\mathcal{O}(m \log(\beta n))$ memory.
Thus, the total memory complexity of ECPv2 is
$\mathcal{O}((m + d)\log(\beta n)),$
which scales  and only linearly with the dimension $d$ and only logarithmically with the budget $n$.

\section{No-Regret Guarantees}
\label{sec:regret_analysis}

Although the proposed mechanisms in ECPv2 significantly reduce computational cost and minimize unnecessary rejections, the algorithm retains theoretical guarantees. In particular, ECPv2 preserves the no-regret behavior over Lipschitz objectives. We present both asymptotic and finite-time guarantees below, with full proofs available in Appendix~B.

\begin{theorem}[Asymptotic No-Regret]
\label{thm:proba_convergence}
Let \( f \in \text{Lip}(k) \) for some unknown Lipschitz constant \( k > 0 \), and let ECPv2 be run with any \( \delta \in [0,1) \), \( \varepsilon_1 > 0 \), integer \( m \geq 1 \), $\beta>1$, \( \tau_{n,d} > 1 \), and \( C > 1 \). Then the simple regret satisfies:
\[
\mathcal{R}_{\text{ECPv2}, f}(n) \xrightarrow{p} 0.
\]
\end{theorem}

\begin{theorem}[Finite-Time Regret Bound]
\label{thm:ecpupperbound}
Let \( f \in \text{Lip}(k) \) be a non-constant function over a compact domain \( \mathcal{X} \subset \mathbb{R}^d \), and let ECPv2 be tuned with any \( \delta \in [0,1) \), integer \( m \geq 1 \), $\beta>1$, \( \varepsilon_1 > 0 \), \( \tau_{n,d} > 1 \), and \( C > 1 \). Then for any budget \( n \in \mathbb{N}^\star \) and confidence level \( \xi \in (0, 1 - \frac{1}{\beta^2}) \), with probability at least \( 1 - \frac{1}{\beta^2} - \xi \), the simple regret satisfies:
\[
\mathcal{R}_{\text{ECPv2}, f}(n)
\leq
k \cdot
\operatorname{diam}(\mathcal{X}) \cdot \log_{\tau_{n,d}}\left(\frac{k}{\varepsilon_1}\right)^{\frac{1}{d}}
\left(  \frac{\ln(1/\xi)}{ n } \right)^{\frac{1}{d}}.
\]
\end{theorem}

\paragraph{Discussion.}  
Theorem~\ref{thm:proba_convergence} establishes that ECPv2 is no-regret. Theorem~\ref{thm:ecpupperbound} further provides a finite-time performance guarantee, bounding the simple regret with high probability. Specifically, ECPv2 achieves the minimax optimal convergence rate of \( \mathcal{O}(k n^{-1/d}) \), matching the lower bound \( \Omega(k n^{-1/d}) \) in Proposition~\ref{prop:minimax}. Notably, this is obtained with significantly reduced runtime and memory cost compared to the previous Lipschitz-based optimizers (see Table~\ref{tab:ecp-comparison}).

In fact, ECPv2 attains the same regret upper bound as ECP, but with a slightly lower confidence level. This trade-off is explicitly governed by the parameter \(\beta\), which balances three factors: computational complexity, memory usage, and the confidence level of the regret bound. A larger \(\beta\) improves the confidence guarantee (bringing it closer to 1), but potentially increases the overhead dealing with larger dimensions. Conversely, a smaller \(\beta\) accelerates execution at the cost of a looser high-probability guarantee.

\begin{figure}[t]
    \centering
    \includegraphics[width=0.95\linewidth]{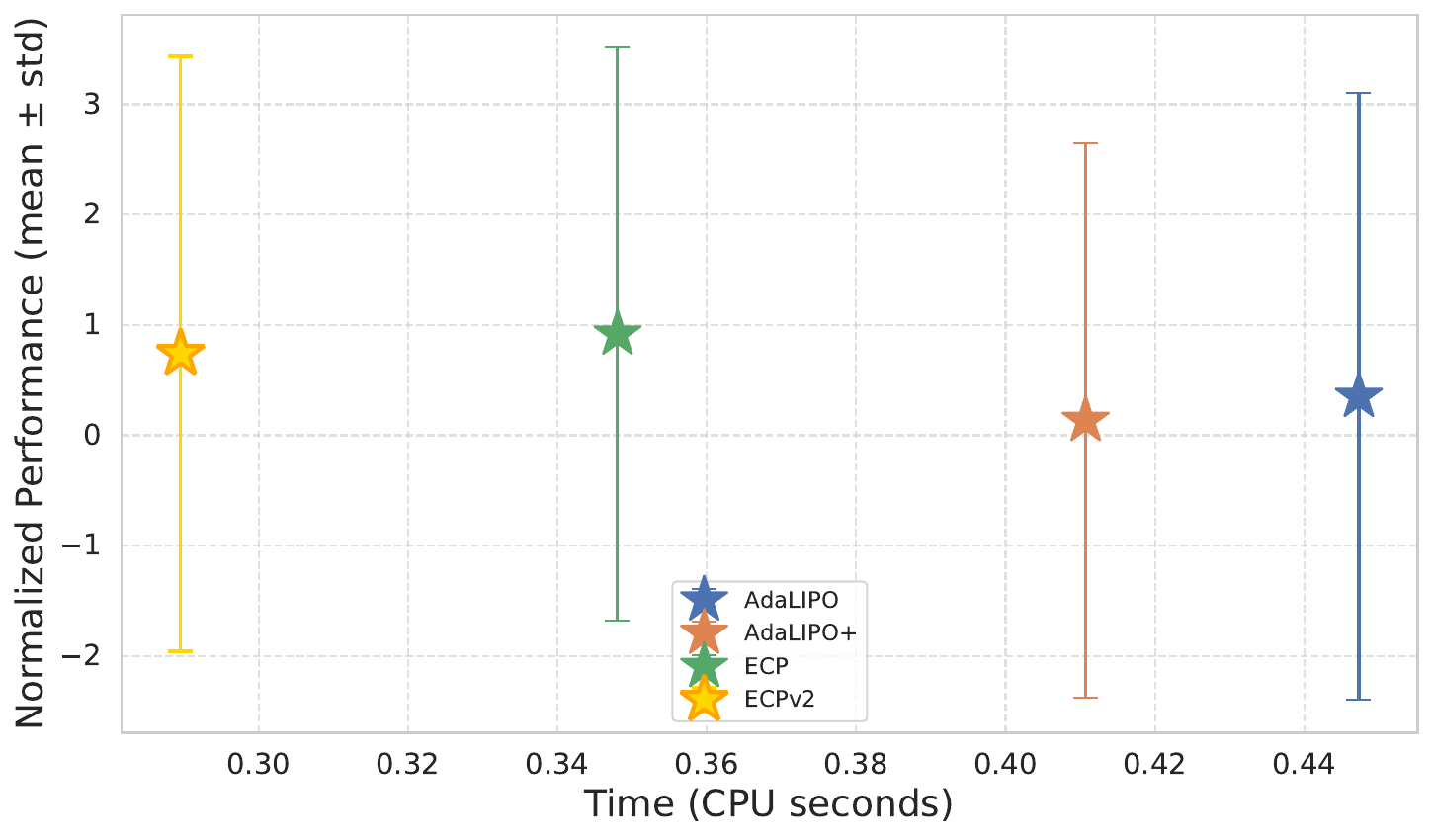}
    \caption{
        Comparison of Lipschitz optimization methods on Rosenbrock with $d \in \{3, 100, 200, 300, 500\}$. Each star shows the mean performance across dimensions after $n=200$ evaluations, averaged over 100 runs. ECPv2 uses hyperparameters ($\beta = 5$, $\delta = 2/3$, $m = 8$).
    }
    \label{fig:rosenbrock_ecpv2_comparison}
\end{figure}

\begin{figure*}[t]
    \centering

    \begin{subfigure}[t]{0.46\textwidth}
        \centering
        \includegraphics[width=\linewidth]{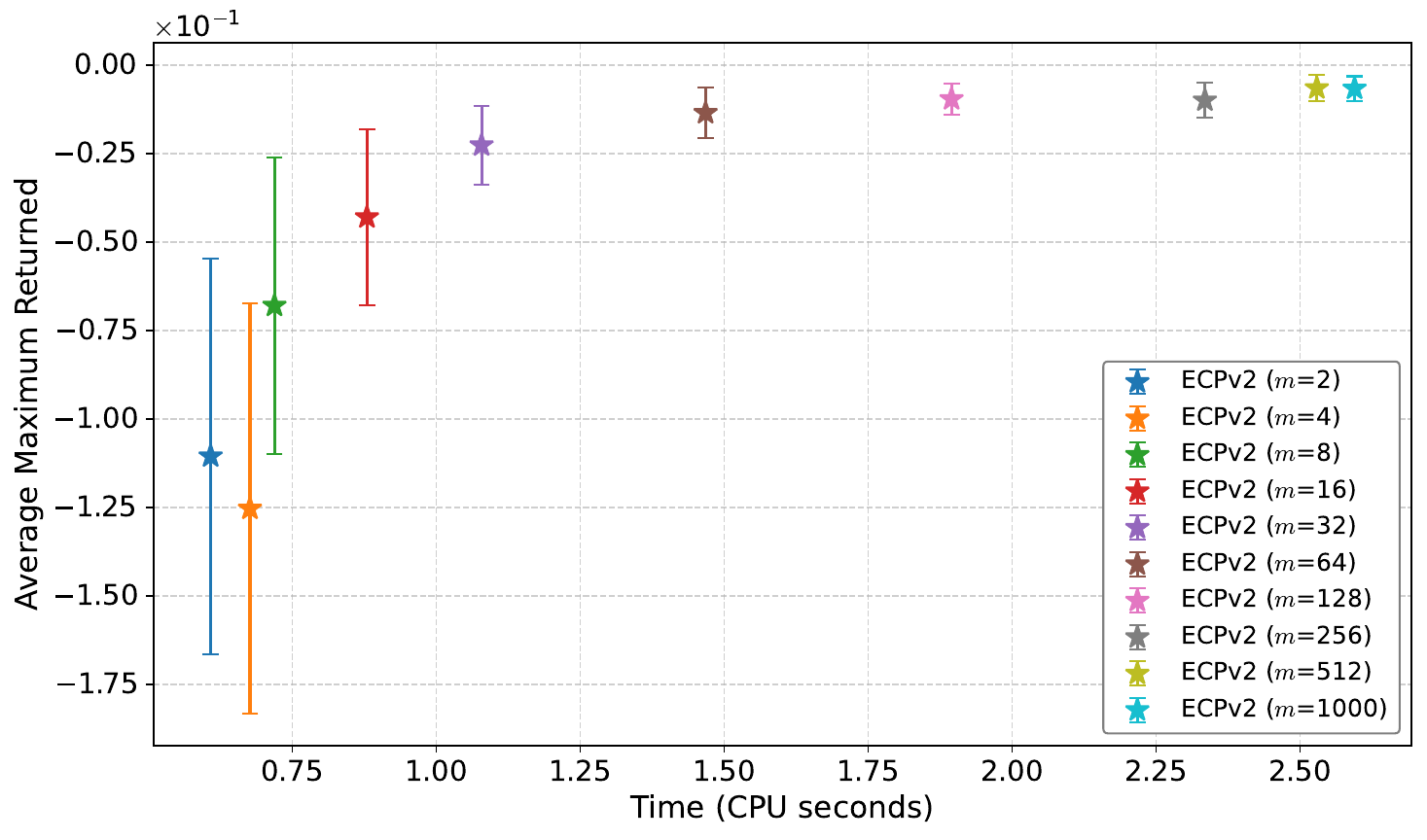}
        \caption{Himmelblau}
    \end{subfigure}
    \hspace{0.06\textwidth}
    \begin{subfigure}[t]{0.46\textwidth}
        \centering
        \includegraphics[width=\linewidth]{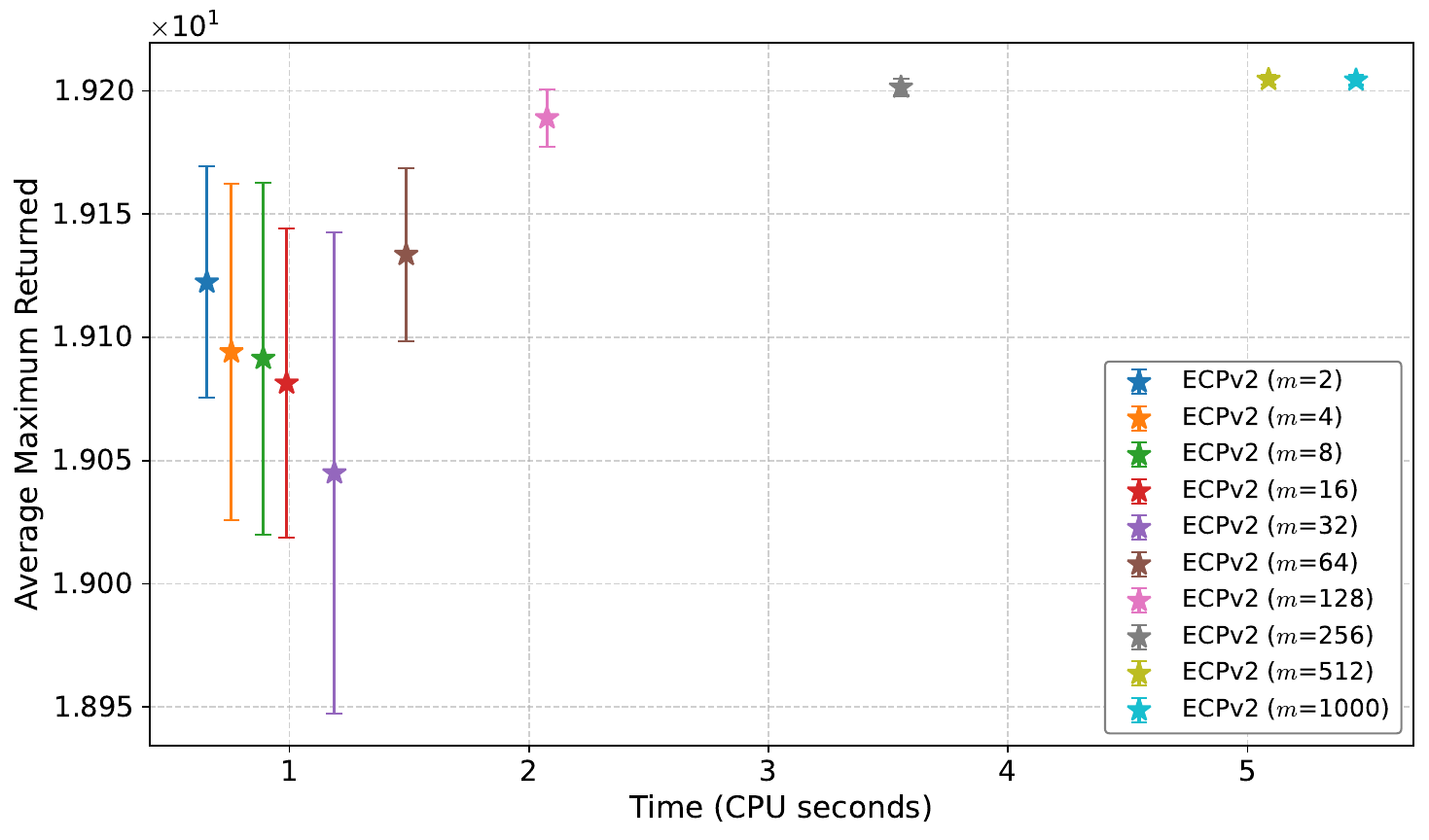}
        \caption{Hölder}
    \end{subfigure}

    \caption{Ablation study on the projection dimension $m$ in ECPv2, using fixed parameters $\delta = 2/3$ and $\beta = 5$. For each benchmark function, every method is allocated $n = 1000$ evaluations and performance is averaged over 100 independent runs. Each point reports the mean final best value with $\pm$ half a standard 
    deviation.}
    \label{fig:ecpv2_ablation_m_main}
\end{figure*}

\section{Principled Hyperparameter Choices ($\beta$, $\delta$)}

The random projection component introduces the distortion tolerance \( \delta \) and the confidence scaling factor \( \beta \). These parameters jointly determine the reduced projection dimension \( d' \) via the bound established in Lemma~\ref{lemma:distoriton_bound}.
The parameter \( \beta \) governs the probability that all pairwise distances are preserved within the specified distortion interval. 

To ensure high-confidence embeddings, with a 96\% success rate, we fix \( \beta = 5 \). Furthermore, we analytically derive the value of \( \delta \) that minimizes the required projection dimension, and find the optimal setting to be \( \delta = \frac{2}{3} \). This choice yields the most dimension-efficient embedding. Full derivations and detailed sensitivity analysis are provided in Appendix~D.

\section{Empirical Evaluation}

We evaluate the performance of ECPv2 on high-dimensional, non-convex global optimization tasks and compare it against Lipschitz-based global optimization methods (Table~\ref{tab:ecp-comparison}), including ECP~\cite{fourati25ecp}, AdaLIPO~\cite{malherbe2017global}, and AdaLIPO+~\cite{serre2024lipo+}. 

Comparisons with general-purpose optimizers such as SMAC3~\cite{lindauer2022smac3}, DIRECT~\cite{jones2021direct}, and Dual Annealing~\cite{xiang1997generalized} are provided in Appendix~I, implementation details in Appendix~J, examination of the sensitivity to the
projection dimension and the influence of $\delta$ and $\beta$ in Appendix D, and extensive ablations in Appendices~E--G.

Figure~\ref{fig:rosenbrock_ecpv2_comparison} presents an ablation of the projection dimension $m$ on Himmelblau and Hölder using fixed parameters $\delta = 2/3$ and $\beta = 5$, illustrating how $m$ influences the behavior of ECPv2. Appendix~E extends this study to ten additional benchmark functions and shows that the Worst-$m$ approximation provides substantial computational savings: evaluating candidates against only a small subset of the worst-performing points reduces the acceptance-test cost by roughly $4$--$5\times$. Remarkably, small values such as $m \in {8, 16, 32, 64, 128}$ often match and sometimes even exceed the performance of the full method by allowing more permissive early exploration. Larger choices of $m$ yield diminishing returns and eventually recover both the behavior and the runtime overhead of the original ECP.

Appendix~F demonstrates that the adaptive lower bound $\varepsilon_t^{\oslash}$ is crucial for preventing empty acceptance regions. Enabling it alone yields roughly $2\times$ faster runtimes at low evaluation budgets and can improve performance.

Appendix~G shows that combining random projections with the Worst-$m$ rule produces the fastest ECPv2 variant. On the 500-dimensional Rosenbrock function, all ECPv2 configurations outperform ECP in runtime, with the full version achieving the best speed–quality tradeoff.

Figure~\ref{fig:ecpv2_ablation_m_main} reports results on Rosenbrock functions with dimensions $d \in \{3, 100, 200, 300, 500\}$, comparing ECPv2 (using the default hyperparameters $\beta=5$, $\delta=2/3$, and $m=8$) to Lipschitz-based baselines from Table~\ref{tab:ecp-comparison}. ECPv2 attains competitive performance while significantly reducing runtime (Appendix~H), even at moderate budgets ($n=200$). Additional results on high-dimensional settings---Rosenbrock with $d=500$ and Powell with $d=1000$---are presented in Appendix~H. In both cases, ECPv2 not only converges roughly $2\times$ faster than ECP but also reaches higher optimization scores. These gains grow more pronounced with increasing dimensionality, illustrating the scalability and efficiency of our approach.

\section{Conclusion}

We introduced ECPv2, a fast and scalable algorithm for global optimization of Lipschitz-continuous functions with unknown constants. The method builds on the ECP framework, which already provides strong empirical performance and theoretical guarantees with minimal tuning. ECPv2 retains these guarantees while reducing computational and memory costs through an adaptive lower bound, a Worst-$m$ selective memory, and a projection-based acceleration. These components, used individually or in combination, offer additional benefits in large-dimensional settings and under large evaluation budgets. Across a wide range of benchmarks, ECPv2 matches or exceeds the performance of several black-box optimization methods while delivering significant wall-clock speedups, demonstrating the robustness and scalability of the ECP family for black-box optimization.

\bibliography{aaai2026}

\clearpage
\appendix

\section*{Appendix A: Random Projection Guarantees}

To support our analysis, we restate and adapt key results from \cite{arriaga2006algorithmic} that characterize the behavior of random Gaussian projections. These results guarantee the concentration of norms and distances after dimensionality reduction, which is critical to our use of random projections in Lipschitz optimization.

\medskip
\noindent\textbf{Norm Preservation for a Single Vector.}  
The following lemma shows that projecting a fixed vector using a random Gaussian matrix approximately preserves its squared norm, with high probability:

\begin{lemma}[Adapted from Lemma 2 in \cite{arriaga2006algorithmic}]
Let $R = (r_{ij})$ be a random $n \times d'$ matrix, where each entry $r_{ij}$ is drawn independently from $\mathcal{N}(0,1)$. For any fixed vector $u \in \mathbb{R}^n$ and any $\delta \in (0, 1)$, let $u' = \frac{1}{\sqrt{d'}} R^T u$. Then $\mathbb{E}[\|u'\|_2^2] = \|u\|_2^2$, and:
\[
\mathbb{P}\left[\|u'\|_2^2 > (1 + \delta)\|u\|_2^2\right] \leq e^{-(\delta^2 - \delta^3)\frac{d'}{4}},
\]
\[
\mathbb{P}\left[\|u'\|_2^2 < (1 - \delta)\|u\|_2^2\right] \leq e^{-(\delta^2 - \delta^3)\frac{d'}{4}}.
\]
\end{lemma}

\medskip
\noindent\textbf{Distance Preservation Between Pairs of Vectors.}  
Using the lemma above, we immediately obtain the following corollary, which states that pairwise Euclidean distances are preserved under the same random projection with high probability. This result is essential for our dimensionality reduction technique:

\begin{corollary}[Adapted from Theorem 2 (Neuronal RP) in \cite{arriaga2006algorithmic}]
\label{corr:neuronal}
Let $R$ be a random $n \times d'$ matrix with entries drawn independently from $\mathcal{N}(0,1)$. For any $u, v \in \mathbb{R}^n$, define the projections $u' = \frac{1}{\sqrt{d'}} R^T u$ and $v' = \frac{1}{\sqrt{d'}} R^T v$. Then for any $\delta \in (0, 1)$,
\begin{align}
&\mathbb{P}\left[(1 - \delta)\|u - v\|_2^2 \leq \|u' - v'\|_2^2 \leq (1 + \delta)\|u - v\|_2^2\right] \nonumber \\
&\quad \geq 1 - 2e^{-(\delta^2 - \delta^3)\frac{d'}{4}}.   \nonumber
\end{align}
\end{corollary}

\medskip
\noindent\textbf{Distance Preservation for a Finite Set of Vectors.}  
The following theorem shows that a random Gaussian projection approximately preserves all pairwise distances among a set of $n$ vectors, with high probability. This is crucial when applying dimensionality reduction to all the sampled points by ECPv2.

\subsection{Proof of Lemma~\ref{lemma:distoriton_bound}}

\begin{proof}
When \( \delta = 0 \), by design of the projection matrix in Equation~\eqref{projection_definition}, it is simply the identity, and the result holds trivially. For \( \delta \in (0, 1) \), we proceed as follows. There are ${n \choose 2} < \frac{n^2}{2}$ distinct pairs of vectors. From the Corollary~\ref{corr:neuronal}, the probability that a single pair deviates outside the $(1 \pm \delta)$ bounds is at most $2\exp\left( -(\delta^2 - \delta^3)\frac{d'}{4} \right) \leq 2\exp\left( -2 \log(\beta \cdot n)\right) = 2\exp\left( -\log(\beta^2 \cdot n^2)\right)$. By the union bound over all pairs, we have the probability that any pair has a large distortion is at most $\frac{1}{\beta^2}.
$
\end{proof}

\subsection{Proof of Lemma~\ref{lemma:jl-projection-superset}}

\begin{proof}
Let $x \in \mathcal{A}_t(\varepsilon_t, t, m)$, which implies
\[
\min_{i \in \mathcal{I}_t^{m}}  \bigl(f(X_i) + \max\{\varepsilon_t, \varepsilon^\oslash_t\} \cdot \|x - X_i\|_2 \bigr) \geq \max_j f(X_j).
\]
By the Lemma~\ref{lemma:distoriton_bound}, for all $i$, with a probability at least $1-\frac{1}{\beta^2}$:
\[
\|\mathbf{P} x - \mathbf{P} X_i\|_2 \leq \frac{\|x - X_i\|_2}{\sqrt{1 - \delta}}.
\]
Multiplying both sides by $\tilde{\varepsilon}_t = \max\{\varepsilon_t, \varepsilon^\oslash_t\} / \sqrt{1 - \delta}$ gives
\[
\tilde{\varepsilon}_t \cdot \|\mathbf{P} x - \mathbf{P} X_i\|_2 \geq \max\{\varepsilon_t, \varepsilon^\oslash_t\} \cdot \|x - X_i\|_2.
\]
Thus, for all i, with probability at least $1- \frac{1}{\beta^2}$:
\[
f(X_i) + \tilde{\varepsilon}_t \cdot \|\mathbf{P} x - \mathbf{P} X_i\|_2 \geq f(X_i) + \max\{\varepsilon_t, \varepsilon^\oslash_t\} \cdot \|x - X_i\|_2.
\]
Taking the minimum over $i \in \mathcal{I}_t^{m}$ yields with probability at least $1- \frac{1}{\beta^2}$:
\begin{align*}
\min_{i \in \mathcal{I}_t^{m}} (f(X_i) + \tilde{\varepsilon}_t \cdot \|\mathbf{P} x - \mathbf{P} X_i\|_2) \geq \\
 \min_{i \in \mathcal{I}_t^{m}} (f(X_i) + \max\{\varepsilon_t, \varepsilon^\oslash_t\} \cdot \|x - X_i\|_2).    
\end{align*}
Hence,
\[
\min_{i \in \mathcal{I}_t^{m}} \bigl(f(X_i) + \tilde{\varepsilon}_t \cdot \|\mathbf{P} x - \mathbf{P} X_i\|_2 \bigr) \geq \max_j f(X_j).
\]
Hence, $x \in \mathcal{A}_{ECPv2}(\varepsilon_t, t, m, \mathbf{P})$.
\end{proof}

\section{Appendix B: No-Regret Guarantees}
\label{section:acceptance_condition_analysis}

The acceptance region of ECPv2, similar to ECP, is motivated by the defintion of consistent functions.

\begin{definition}
\label{def:consistent_functions}
{\sc (Consistent functions)}
The active subset of Lipschitz functions, with a Lipschitz constant $k$, consistent with the black-box function $f$ over $t\geq1$ evaluated samples 
$(x_1,f(x_1)), \cdots,(x_t,f(x_t))$ is:
$
\mathcal{F}_{k,t} \triangleq \left\{ g \in \text{Lip}(k) : \forall i \in\{ 1, \cdots, t\},
~ g(x_i) = f(x_i) \right\}.
$
\end{definition}

Using the above definition of a consistent function, we define the subset of points that can maximize at least some function $g$ within that subset of consistent functions and possibly maximize the target $f$.

\begin{definition}
\label{def:potential_maximizers}
{\sc (Potential Maximizers)}
For a Lipschitz function \( f \) with a Lipschitz constant \( k \geq 0 \), let \( \mathcal{F}_{k,t} \) be the set of consistent functions with respect to \( f \), as defined in Definition~\ref{def:consistent_functions}. For any iteration \( t \geq 1 \), the set of potential maximizers is defined as follows:
$
\mathcal{P}_{k,t} \triangleq \left\{ x \in \mathcal{X} : \exists g \in \mathcal{F}_{k,t} \text{ where } x \in \underset{x \in \mathcal{X}}{\arg \max}~g(x) \right\}.
$
\end{definition}

We can then show the relationship between the potential maximizers and our proposed acceptance region. 

\begin{lemma}[Lemma 8 in \cite{malherbe2017global}] 
\label{lem:potential_k}
If $\mathcal{P}_{k,t}$ denotes the set of potential maximizers of the function $f$, as defined in Definition \ref{def:consistent_functions},
then we have $
 \mathcal{P}_{k,t} = \{x \in \mathcal{X}: 
 \min_{i=1, \cdots, t}  (f(x_i) + k \cdot ||x-x_i||_2)
\geq \max_{j=1, \cdots, t}f(x_j) \}
$. 
\end{lemma}

\begin{proposition}
{\sc (Potential Optimality)} 
\label{prop:potential}
For any iteration $t$, if $\mathcal{P}_{k,t}$ denotes the set of potential maximizers of $f \in \text{Lip}(k)$, as in Definition~\ref{def:consistent_functions}, and $\mathcal{A}_{ECPv2}(\varepsilon_t, t, m, \mathbf{P})$ denotes our acceptance region, defined in Equation~\eqref{new_acceptance_condition_with_proj},
then with probability at least $1-1/\beta^2$:
$
\begin{aligned}
\forall \varepsilon_t > k, \quad \mathcal{P}_{k,t} \subseteq \mathcal{A}_{ECPv2}(\varepsilon_t, t, m, \mathbf{P}).
\end{aligned}
$
\end{proposition}

\begin{proof}
From Proposition 2 of \cite{fourati25ecp}, we have $\mathcal{P}_{k,t} \subseteq \mathcal{A}_{ECP}(\varepsilon_t, t)$ for any $\varepsilon_t > k$.  
By Corollary~\ref{corollary:ecp2_superset_of_ecp}, $\mathcal{A}_{ECP}(\varepsilon_t, t) \subseteq \mathcal{A}_{ECPv2}(\varepsilon_t, t, m, \mathbf{P})$.  
Therefore, by transitivity,
$\mathcal{P}_{k,t} \subseteq \mathcal{A}_{ECPv2}(\varepsilon_t, t, m, \mathbf{P})$ for any $\varepsilon_t > k.$
\end{proof}

Let $\mathcal{A}_{ECPv2}(\varepsilon_t, t, m, \mathbf{P})$ denote the acceptance region of ECPv2, as defined in Definition~\ref{def:potentially_accepted_v2}, and let $\mathcal{A}_{ECP}(\varepsilon_t, t)$ denote the region defined in Definition~\ref{def:potentially_accepted}. Let $\mathbf{P}$ be a random projection matrix with reduced dimension $d'$, as in Lemma~\ref{lemma:distoriton_bound}. Then, for any $m \leq n$, with probability at least $1 - 1/\beta^2$, the following inclusion holds:
\[
\mathcal{A}_{ECP}(\varepsilon_t, t) \subseteq \mathcal{A}_{ECPv2}(\varepsilon_t, t, m, \mathbf{P}). 
\]

Therefore, when $\varepsilon_t$ reaches or exceeds $k$, i.e., $\varepsilon_t \geq k$, which is unavoidable with a growing number of evaluations, the acceptance space does not exclude any potential maximizer, as all potential maximizers remain within the acceptance condition, which is crucial to guarantee the no-regret property of ECPv2.

First, we define the $i^\star$ as the hitting time, after which $\varepsilon_t$ reaches or overcomes the Lipschitz constant $k$.

\begin{definition}{\sc (Hitting Time)}
\label{i_star_definition}
For the sequence $\left(\varepsilon_i\right)_{i \in \mathbb{N}}$ and the unknown Lipschitz constant $k > 0$, we can define 
$
i^\star \triangleq \min \left\{i \in \mathbb{N}^\star: \varepsilon_i \geq k\right\}.
$
\end{definition}

In the following lemma, we upper-bound the time $t$ after which $\varepsilon_t$ is guaranteed to reach or exceed $k$. This shows that for sufficiently large $t$, the event of reaching $k$ is inevitable, with proof provided in Appendix A.

\begin{lemma}{\sc (Hitting Time Upper-bound)}
\label{hitting_time_upperbound}
For any function $f \in \text{Lip}(k)$, for any coefficient $\tau_{n,d}>1$, any initial value $\varepsilon_1 > 0$, any constant $C>1$, any integer $m$, any distortion $\delta \in [0,1)$, and any integer $m \geq 1$, the hitting time $i^\star$ is upperbounded as follows:
\begin{equation*}
\forall \varepsilon_1 > 0,  \quad i^\star \leq \max\left(\left\lceil\log_{\tau_{n,d}}\left(\frac{k}{\varepsilon_1}\right)\right\rceil,1\right).
\end{equation*}
\end{lemma}
\begin{proof}
Notice that for all $t \geq 1$, for any choice of $C > 1$, $\varepsilon_1 > 0$, $\tau_{n,d} > 1$, $\delta \in [0,1)$, and $m \geq 1$, $\varepsilon_t$ grows throughout the iterations when the stochastic growth condition is satisfied and when a point is evaluated. While the growth condition may or may not occur, the latter happens deterministically after every evaluation. Therefore, $\varepsilon_t \geq \varepsilon_1 \tau_{n,d}^{t-1}$. Hence, the result follows from the non-decreasing and diverging geometric growth of $\varepsilon_1 \tau_{n,d}^{t-1}$.
\end{proof}

\subsection{Proof of Theorem~\ref{thm:proba_convergence}}
\label{app:prop:proba_convergence}

For any given function $f$, with some fixed unknown Lipschitz constant $k$, and for any chosen constants $\varepsilon_1 > 0$, $\tau_{n,d} > 1$, $C > 1$, $\delta \in [0,1)$, and $m\geq1$, as shown in Lemma~\ref{hitting_time_upperbound}, there exists a constant $L = \left\lceil\log _{\tau_{n,d}}\left(\frac{k}{\varepsilon_1}\right)\right\rceil$, not depending on $n$, such that for $t \geq L$, we have $t \geq i^\star$. Hence, by Definition~\ref{i_star_definition}, for $t \geq L$, $\varepsilon_t$ reaches and exceeds $k$. Therefore, it is guaranteed that as $t$ tends to infinity, $\varepsilon_t$ surpasses $k$. Furthermore, by Proposition~\ref{prop:potential}, we know that for all $\varepsilon_t > k$, $\mathcal{P}_{k,t} \subseteq \mathcal{A}_t(\varepsilon_t, t, m, \mathbf{P})$. Thus, as $t$ tends to infinity, the search space uniformly recovers all the potential maximizers and beyond.

\begin{lemma}[Corollary 13 in \cite{malherbe2017global}]
\label{prop:cvg_prs}
Let $\mathcal{X} \subset \mathbb{R}^d$ be
a compact and convex set with non-empty interior
and let $f \in \text{Lip}(k)$ be a $k$-Lipschitz functions defined on $\mathcal{X}$ for some $k\geq0$.
Then, 
for any $n \in \mathbb{N}^{\star}$ and
$\xi \in(0,1)$, we have with probability at least $1-\xi$,
\[
\max_{x \in \mathcal{X}}f(x) -\max_{i=1, \cdots, n} f(x_i) 
\leq k \cdot \operatorname{Diam}(\mathcal{X}) \cdot \left(  \frac{\ln(1/\xi)}{n} \right)^{\frac{1}{d}}
\]
where $x_1, \cdots, x_n$ denotes a sequence of $n$ independent copies of
$x \sim \mathcal{U}(\mathcal{X})$.
\end{lemma}

\begin{proposition}
\label{prop:fasterprs}
{\sc (ECPv2 Faster than Pure Random Search)} 
Consider the ECPv2 algorithm tuned with any initial value $\varepsilon_1>0$, any constant $ \tau_{n,d}>1 $, any constant $C > 1$, any integer $m \geq 1$, any $\delta \in [0,1)$, and any $\beta > 1$.
Then, for any $f \in \normalfont{\text{Lip}}(k)$ and
$n \geq i^\star$, with probability $1-1/\beta^2$
\[
\mathbb{P}\left(\max_{i=i^\star, \cdots,  n } f(x_i) \geq y \right) \geq 
\mathbb{P}\left(\max_{i=i^\star, \cdots,  n } f(x'_i) \geq y \right)
\]
for all $y \geq \max_{i=1, \cdots, i^{\star} -1} f(x_i)$, where $x_{1}, \cdots,  x_n$ are $n$ evaluated points by ECPv2
and $x_{i^\star}', \cdots,  x'_n$ are $n$ 
independent uniformly distributed points over $\mathcal{X}$.
\end{proposition}

\begin{proof}
We proceed by induction. Let \( x_1, \cdots, x_{i^\star} \) be a sequence of evaluation points generated by ECPv2 after \( i^\star \) iterations, and let \( x_{i^\star}' \) be an independent point randomly sampled over \( \mathcal{X} \). Consider any \( y \geq \max_{i=1, \cdots, i^\star - 1} f(x_i) \), and define the corresponding level set \( \mathcal{X}_y = \{ x \in \mathcal{X} : f(x) \geq y \} \). Assume, without loss of generality, that \( \mu(\mathcal{X}_y) > 0 \) (otherwise, \( \mathbb{P}(f(x_{i^\star}) \geq y) = 0 \), and the result trivially holds).

For \( t = i^\star \), we have \( \varepsilon_{i^\star} \geq k \), therefore by Proposition~\ref{prop:potential} we have with probability at least $1-1/\beta^2$ \( \mathcal{P}_{k,i^\star} \subseteq \mathcal{A}_{ECPv2}(\varepsilon_{i^\star}, i^\star, m, \mathbf{P})\subseteq \mathcal{X} \). Moreover, since \( y \geq \max_{i=1, \cdots, i^\star - 1} f(x_i) \), if \( \mathcal{X}_y \) is non-empty, its elements are potential maximizers. Hence, \( \mathcal{X}_y \subseteq \mathcal{P}_{k,i^\star} \). If \( \mathcal{X}_y \) is empty, the result holds trivially. Next, we compute the following probabilities:
\[
\begin{aligned}
\mathbb{P}(f(x_{i^\star}) \geq y) &= \mathbb{E}\left[\mathbb{I}\{ x_{i^\star} \in \mathcal{X}_y \} \right] \\
&=\mathbb{E}\left[\frac{\mu(\mathcal{A}_{ECPv2}(\varepsilon_{i^\star}, i^\star, m, \mathbf{P}) \cap \mathcal{X}_y)}{\mu(\mathcal{A}_{ECPv2}(\varepsilon_{i^\star}, i^\star, m, \mathbf{P}))} \right] \\&\geq \mathbb{E}\left[\frac{\mu(\mathcal{P}_{k,i^\star} \cap \mathcal{X}_y)}{\mu(\mathcal{X})} \right] \\
&= \mathbb{E}\left[\frac{\mu(\mathcal{X}_y)}{\mu(\mathcal{X})} \right] \\
&= \mathbb{P}(f(x_{i^\star}') \geq y).
\end{aligned}
\]

Now, suppose the statement holds for some \( n \geq i^\star \). Let \( x_1, \cdots, x_{n+1} \) be a sequence of evaluation points generated by ECPv2 after \( n+1 \) iterations, and let \( x_1', \cdots, x_{n+1}' \) be a sequence of \( n+1 \) independent points sampled over \( \mathcal{X} \).

As before, assume \( \mu(\mathcal{X}_y) > 0 \), and let \( \mathcal{A}_{ECPv2}(\varepsilon_{n}, n, m, \mathbf{P})\) denote the sampling region of \( x_{n+1} \mid x_1, \cdots, x_n \). Then, on the event \( \{ \max_{i=i^\star, \cdots, n} f(x_i) < y \} \), we have, with probability $1-1/\beta^2$, \( \mathcal{X}_y \subseteq \mathcal{P}_{k,n} \subseteq \mathcal{A}_{ECPv2}(\varepsilon_{n}, n, m, \mathbf{P}) \subseteq \mathcal{X} \). We note $\mathcal{T} = \mathbb{I}\left\{\max_{i=i^{\star}, \cdots, n} f(x_i) \geq y\right\}$ and for briefness we note $\mathcal{A}_{ECPv2, n} = \mathcal{A}_{ECPv2}(\varepsilon_{n}, n, m, \mathbf{P})$ and compute:
\[
\begin{aligned}
&\mathbb{P}\left(\max_{i=i^{\star}, \cdots, n+1} f(x_i) \geq y\right) \\
& = \mathbb{E}\left[\mathcal{T} + \mathbb{I}\left\{\max_{i=i^{\star}, \cdots, n} f(x_i) < y, x_{n+1} \in \mathcal{X}_y\right\}\right] \\
& = \mathbb{E}\left[\mathcal{T} + \mathbb{I}\left\{\max_{i=i^{\star}, \cdots, n} f(x_i) < y\right\} \frac{\mu\left(\mathcal{A}_{ECPv2, n} \cap \mathcal{X}_y\right)}{\mu\left(\mathcal{A}_{ECPv2, n}\right)}\right] \\
& \geq \mathbb{E}\left[\mathcal{T} + \mathbb{I}\left\{\max_{i=i^{\star}, \cdots, n} f(x_i) < y\right\} \frac{\mu\left(\mathcal{P}_{k,n} \cap \mathcal{X}_y\right)}{\mu\left(\mathcal{A}_{ECPv2, n}\right)}\right] \\
& \geq \mathbb{E}\left[\mathcal{T} + \mathbb{I}\left\{\max_{i=i^{\star}, \cdots, n} f(x_i) < y\right\} \frac{\mu\left(\mathcal{X}_y\right)}{\mu(\mathcal{X})}\right] \\
& \geq \mathbb{E}\left[\mathcal{T} + \mathbb{I}\left\{\max_{i=i^{\star}, \cdots, n} f(x_i') < y\right\} \frac{\mu\left(\mathcal{X}_y\right)}{\mu(\mathcal{X})}\right] \\
&= \mathbb{E}\left[\mathcal{T} + \mathbb{I}\left\{\max_{i=i^{\star}, \cdots, n} f(x_i') < y, x_{n+1}' \in \mathcal{X}_y\right\}\right] \\
& = \mathbb{P}\left(\max_{i=i^{\star}, \cdots, n+1} f(x_i') \geq y\right)
\end{aligned}
\]
where the third inequality follows from the fact that \( x \mapsto \mathbb{I}\{x \geq y\} + \mathbb{I}\{x < y\} \frac{\mu(\mathcal{X}_y)}{\mu(\mathcal{X})} \) is non-decreasing, and the induction hypothesis implies that \( \max_{i=i^\star, \cdots, n} f(x_i) \) stochastically dominates \( \max_{i=i^\star, \cdots, n} f(x_i') \). Thus, by induction, the statement holds for all \( n \geq i^\star \), completing the proof.    
\end{proof}

\subsection{Proof of Theorem~\ref{thm:ecpupperbound}}
\label{proof:thm:ecpupperbound}

Fix any \(\xi \in \left(0, 1 - \frac{1}{\beta^2}\right)\), and let \(i^\star\) be defined as in Definition~\ref{i_star_definition}. Consider any iteration budget \(n > i^\star\).

By Proposition~\ref{prop:fasterprs}, for all \(t \geq i^\star\), the precision threshold \(\varepsilon_t \geq k\) with probability at least \(1 - \frac{1}{\beta^2}\), where \(k\) is the (unknown) Lipschitz constant of \(f\). This implies that, starting from iteration \(i^\star\), ECPv2 behaves at least as efficiently as Pure Random Search (PRS) in improving over the best value observed so far.

Let \(x^\star \in \arg\max_{x \in \mathcal{X}} f(x)\) and define the simple regret after \(n\) evaluations as \(\mathcal{R}_{\text{ECPv2}, f}(n) := f(x^\star) - \max_{i=1,\ldots,n} f(x_i)\). From Lemma~\ref{prop:cvg_prs}, with total probability at least \(1 - \frac{1}{\beta^2} - \xi\), ECPv2 enjoys the same regret bound:
\begin{align*}
\mathcal{R}_{\text{ECPv2}, f}(n)
&\leq k \cdot \operatorname{diam}(\mathcal{X}) \cdot \left( \frac{\ln(1/\xi)}{n - i^\star + 1} \right)^{1/d} \\
& = k \cdot
\operatorname{diam}(\mathcal{X}) \cdot \left(  \frac{n}{n-i^\star + 1} \right)^{\frac{1}{d}} \cdot
\left(  \frac{\ln(\frac{1}{\xi})}{ n } \right)^{\frac{1}{d}}  \\
& \leq k \cdot
\operatorname{diam}(\mathcal{X}) \cdot \left(i^\star \right)^{\frac{1}{d}}
\left(  \frac{\ln(1/\xi)}{ n } \right)^{\frac{1}{d}}\\
&\leq k \cdot
\operatorname{diam}(\mathcal{X}) \cdot \log_{\tau_{n,d}}\left(\frac{k}{\varepsilon_1}\right)^{\frac{1}{d}}
\left(  \frac{\ln(1/\xi)}{ n } \right)^{\frac{1}{d}}
\end{align*}
where the last two inequalities follow from Lemma~\ref{hitting_time_upperbound}, which ensures \( \frac{n}{n - i^\star + 1} \leq i^\star \) for all \(n > i^\star\).

For the case \(n \leq i^\star\), the bound is trivially satisfied since:
\[
\mathcal{R}_{\text{ECPv2}, f}(n) \leq k \cdot \operatorname{diam}(\mathcal{X}). 
\]
\hfill\(\square\)

\section{Appendix C: Missing Proofs}

\subsection{Proof of Lemma~\ref{lem:lb_eps}}

\begin{proof}
Suppose there exists $x \in \mathcal{X}$ such that:
\[
\min_{i=1,\dots,t} \left( f(X_i) + \varepsilon_t \cdot \|x - X_i\|_2 \right) \geq f_{\max_t}.
\]
Then for all $i \in \{1, \dots, t\}$, we must have:
\[
f(X_i) + \varepsilon_t \cdot \|x - X_i\|_2 \geq f_{\max_t},
\]
which implies:
\[
\varepsilon_t \cdot \|x - X_i\|_2 \geq f_{\max_t} - f(X_i), \quad \forall i \in [t].
\]
Hence:
\[
\varepsilon_t \geq \frac{f_{\max_t} - f(X_i)}{\|x - X_i\|_2}, \quad \forall i \in [t].
\]
Taking the maximum over all $i \in [t]$:
\[
\varepsilon_t \geq \max_{i=1,\dots,t} \left( \frac{f_{\max_t} - f(X_i)}{\|x - X_i\|_2} \right).
\]
Finally, since $\|x - X_i\|_2 \leq \operatorname{diam}(\mathcal{X})$ for all $x \in \mathcal{X}$, we get:
\[
\varepsilon_t \geq \max_{i=1,\dots,t} \left( \frac{f_{\max_t} - f(X_i)}{\operatorname{diam}(\mathcal{X})} \right) = \frac{f_{\max_t} - f_{\min_t}}{\operatorname{diam}(\mathcal{X})}.
\]
\end{proof}

\subsection{Proof of Lemma~\ref{lemma:lb_superset}}
\begin{proof}
Note that for a given value of $\varepsilon_t$, ECPv2 instead uses $\max\{\varepsilon_t, \varepsilon^\oslash\} \geq \varepsilon_t$. Thus the result follows directly from Lemma 1 in \cite{fourati25ecp}. 
\end{proof}

\subsection{Proof of Lemma~\ref{lemma:acceptance-region-with-m}}
\begin{proof}
When $t \leq m$, we have $\mathcal{I}_t^{m} = \{1, \dots, t\}$, so the Worst-$m$ condition reduces exactly to the condition in Equation~\eqref{new_acceptance_condition_with_lb}, implying $\mathcal{A}_{\text{ECP}}(\varepsilon_t, t) \subseteq \mathcal{A}_t(\varepsilon_t, t) = \mathcal{A}_t(\varepsilon_t, t, m)$, as shown by Lemma~\ref{lemma:lb_superset}.

When $t > m$, the set $\mathcal{I}_t^{m}$ is a strict subset of $\{1,\dots,t\}$ with cardinality $m$. Thus, any point $x$ that satisfies Equation~\eqref{new_acceptance_condition_with_lb} automatically satisfies Equation~\eqref{new_acceptance_condition_with_m}, establishing $\mathcal{A}_t(\varepsilon_t, t) \subseteq \mathcal{A}_t(\varepsilon_t, t, m)$.
\end{proof}

\section{Appendix D: Analysis of Projection Dimension Sensitivity (impact of $\delta$ and $\beta$)}

The Lemma~\ref{lemma:distoriton_bound} states that any set of $n$ points in high-dimensional space can be embedded into $\mathbb{R}^{d'}$ with pairwise distances preserved within a multiplicative factor of $1 \pm \delta$, provided:
\[
d' \geq \frac{8 \log(\beta n)}{\delta^2 - \delta^3}
\]
where $\beta > 1$ controls the failure probability. The term $\delta^2 - \delta^3$ governs the required projection dimension $d'$ in terms of the distortion tolerance $\delta$. To understand its behavior, we define the distortion scaling function:
\[
f(\delta) = \frac{1}{\delta^2 - \delta^3} = \frac{1}{\delta^2 (1 - \delta)} \quad \text{for } \delta \in (0, 1)
\]

\vspace{0.5em}
\noindent
As $\delta \to 0$, the required projection dimension diverges: $f(\delta) \sim 1/\delta^2$. Similarly, as $\delta \to 1$, the denominator vanishes and $f(\delta) \to \infty$. Thus, there exists an intermediate value of $\delta$ that minimizes the bound.

\subsection*{On the choice of $\delta$}

We analytically compute the minimum of $f(\delta)$ by solving:
\[
f(\delta) = \frac{1}{\delta^2 (1 - \delta)} \quad \Rightarrow \quad
f'(\delta) = \frac{-2(1 - \delta) + \delta}{\delta^3 (1 - \delta)^2}
\]
Setting $f'(\delta) = 0$ yields:
\[
-2(1 - \delta) + \delta = 0 \quad \Rightarrow \quad 3\delta = 2 \quad \Rightarrow \quad \delta^* = \frac{2}{3}
\]
At this optimal distortion level:
\[
f\left(\frac{2}{3}\right) = \frac{1}{\left(\frac{2}{3}\right)^2 - \left(\frac{2}{3}\right)^3}
= \frac{1}{\frac{4}{9} - \frac{8}{27}} = \frac{1}{\frac{4}{27}} = \frac{27}{4} = 6.75
\]
Hence, the projection dimension is minimized when $\delta = \frac{2}{3}$, and the bound simplifies to:
\[
d' \geq \frac{8 \log(\beta n)}{6.75}
\]

\subsubsection*{Visualization}

Figure~\ref{fig:jl-scaling} shows the scaling term $f(\delta)$ over the interval $(0.01, 0.99)$. The y-axis is plotted logarithmically to highlight the divergence at the boundaries. The unique minimum at $\delta = \frac{2}{3}$ is marked on the curve.

\begin{figure}[h]
  \centering
  \includegraphics[width=0.44\textwidth]{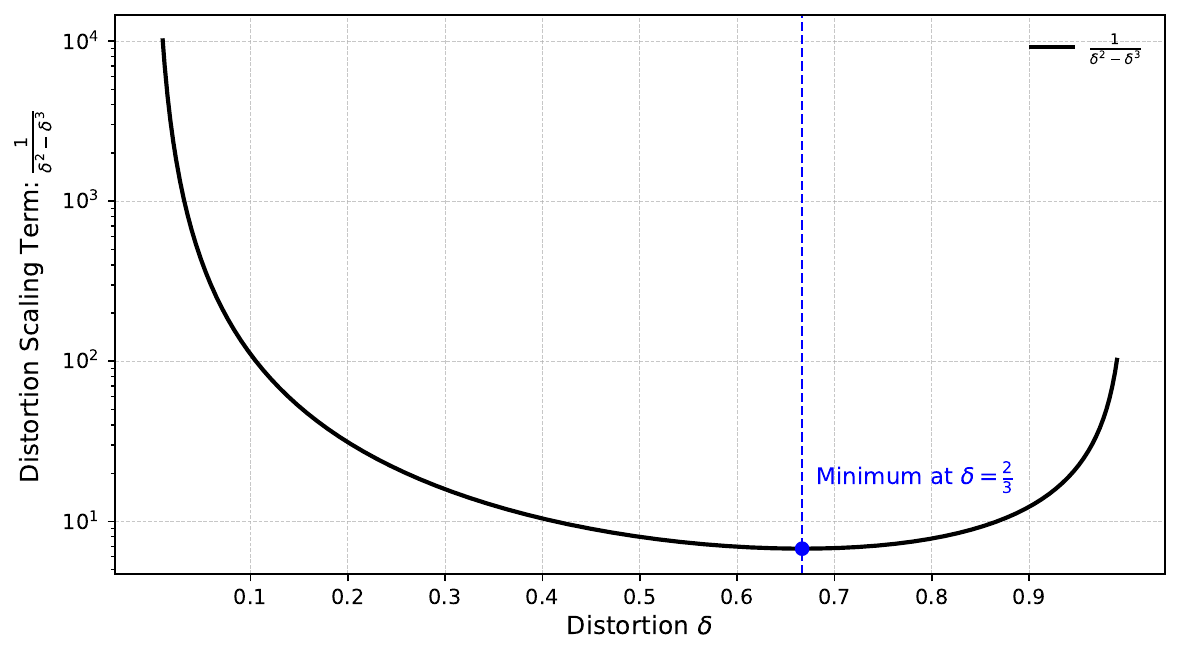}
  \caption{JL scaling term $f(\delta) = \frac{1}{\delta^2 - \delta^3}$ as a function of distortion $\delta$. The function attains a minimum at $\delta = \frac{2}{3}$, yielding the most distortion-efficient embedding.}
  \label{fig:jl-scaling}
\end{figure}

\subsection{Projection Dimension Scaling}

Figure~\ref{fig:projection-dimension} illustrates how the required projection dimension \( d' \) scales with the number of evaluations \( n \) for various values of the allowed distortion \( \delta \):
\[
d' = \frac{8 \log(\beta n)}{\delta^2 - \delta^3},
\]
where \( \beta \) is a constant (set to 5 in our case to ensure the guarantee holds with probability above 96\%). Smaller values of \( \delta \) lead to larger required projection dimensions due to the tighter approximation constraints. The plot uses a log-log scale to capture behavior over several orders of magnitude in \( n \), ranging from \( 10^1 \) to \( 10^4 \).

\begin{figure}[h]
    \centering
    \includegraphics[width=0.8\linewidth]{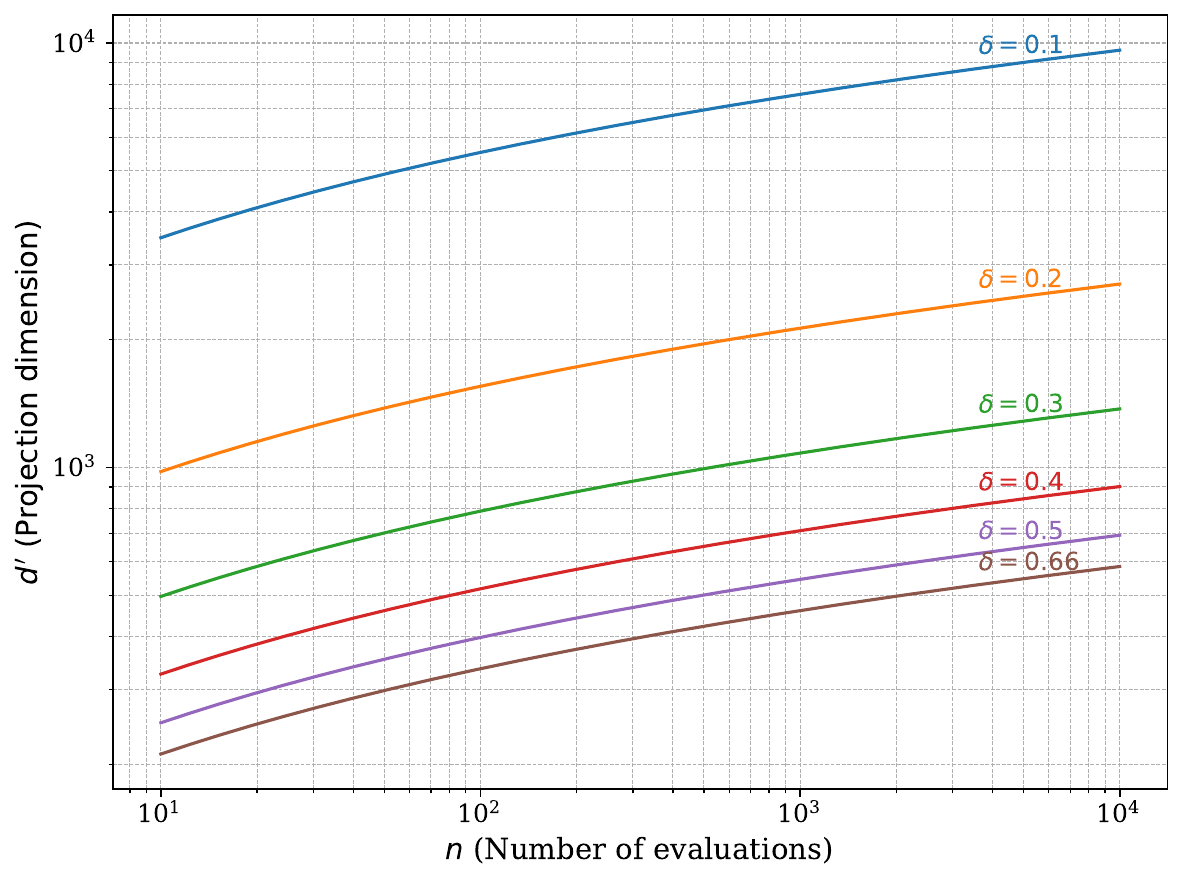}
    \caption{Projection dimension \( d' \) as a function of the number of evaluations \( n \) for various values of \( \delta \in \{0.1, 0.2, 0.3, 0.4, 0.5, 0.6 \} \), with \( \beta = 5 \) fixed.}
    \label{fig:projection-dimension}
\end{figure}

\subsection{On the choice of $\beta$}

The Lemma provides the following guarantee on the success probability of the embedding:
\[
\mathbb{P} \left[\text{pairwise distances are preserved within } (1 \pm \delta)\right] \geq 1 - \frac{1}{\beta^2}
\]
To achieve a success probability of at least $p = 0.95$, we solve:
\[
1 - \frac{1}{\beta^2} \geq 0.95 \quad \Longrightarrow \quad \frac{1}{\beta^2} \leq 0.05 \quad \Longrightarrow \quad \beta^2 \geq 20
\]
This gives the threshold:
\[
\beta \geq \sqrt{20} \approx 4.47
\]

In practice, we round up for simplicity and robustness. Choosing $\beta = 5$ yields:
\[
\Pr[\text{Success}] \geq 1 - \frac{1}{25} = 0.96
\]

\vspace{0.5em}
\noindent
Thus, we adopt $\beta = 5$ in our empirical evaluations, which ensures that the embedding succeeds with at least 96\% probability.

\subsection*{Effect of \(\beta\) on Distortion Concentration}

To empirically validate the bound in Lemma~\ref{lemma:distoriton_bound}, we investigate how the choice of the confidence parameter \( \beta \) affects the preservation of pairwise distances after random projection. Recall that a larger \( \beta \) results in a larger projection dimension \( d' \), thereby increasing the probability that all pairwise distortions fall within the acceptable range \( [1 - \delta,\ 1 + \delta] \).

We generate \( n = 100 \) data points in \( \mathbb{R}^{1000} \) from a uniform distribution over \([0, 1]^d\), and project them to \( \mathbb{R}^{d'} \), where \( d' \) is determined according to the JL-type bound:
\[
d' \geq \frac{8 \log(\beta n)}{\delta^2 - \delta^3}, \quad \text{with } \delta = 0.5.
\]
For each value of \( \beta \in \{0.05,\ 0.5,\ 5\} \), we repeat the projection process over 1000 independent trials and collect all pairwise distortion ratios.

Figure~\ref{fig:distortion-histogram-avg} presents the resulting histograms. As expected, the distortion distributions become more concentrated around 1 as \( \beta \) increases. A horizontal line segment indicates the theoretical distortion bound \( [1 - \delta,\ 1 + \delta] \), which serves as a visual reference for acceptable approximation quality. These results empirically support the theoretical guarantee: larger \( \beta \) leads to more reliable distance preservation due to increased projection dimension.

\begin{figure}[htbp]
    \centering
    \includegraphics[width=0.85\linewidth]{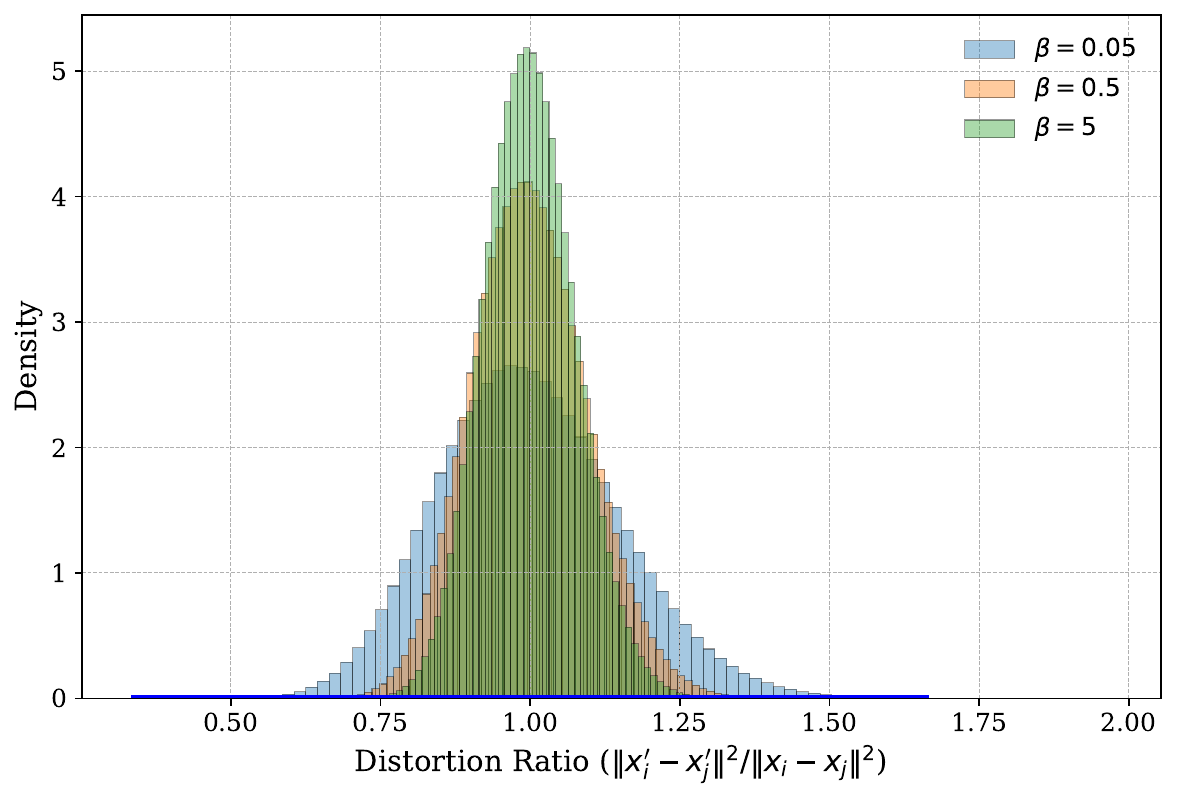}
    \caption{Histogram of pairwise distance distortions aggregated over 1000 random trials. Each histogram corresponds to a different value of \( \beta \), with \( n = 100 \), \( d = 1000 \), and distortion bound \( \delta = 0.5 \). A thick blue horizontal line marks the acceptable distortion interval \( [1 - \delta,\ 1 + \delta] \). As \( \beta \) increases (and with it the projection dimension \( d' \)), the distortion distribution becomes more concentrated around 1, consistent with Lemma~\ref{lemma:distoriton_bound}.}
    \label{fig:distortion-histogram-avg}
\end{figure}

\subsection{Empirical Success Rates vs. Projection Dimension}

To further assess the validity of the theoretical projection bound from Lemma~\ref{lemma:distoriton_bound}, we empirically estimate the success rate of random projections across varying target dimensions \( d' \). 

We fix \( n = 100 \), \( d = 1000 \), \( \delta = \frac{2}{3} \), and vary \( d' \) around the theoretical minimum:
\[
d'_{\text{bound}} = \left\lceil \frac{8 \log(\beta n)}{\delta^2 - \delta^3} \right\rceil,
\quad \text{with } \beta = 5.
\]
For each \( d' \), we generate 1000 random projections and compute the proportion of trials in which all pairwise distances are preserved within the interval \( [1 - \delta,\ 1 + \delta] \). 

Figure~\ref{fig:success-vs-dimension} reports the empirical success rate as a function of \( d' \). The horizontal dashed line marks the theoretical success threshold \( 1 - 1/\beta^2 = 0.96 \), while the vertical dashed line marks the theoretical projection bound. The plot confirms that success rates converge to the theoretical threshold around the predicted dimension \( d'_{\text{bound}} \), supporting the tightness of the bound.

\begin{figure}[h]
    \centering
    \includegraphics[width=0.75\linewidth]{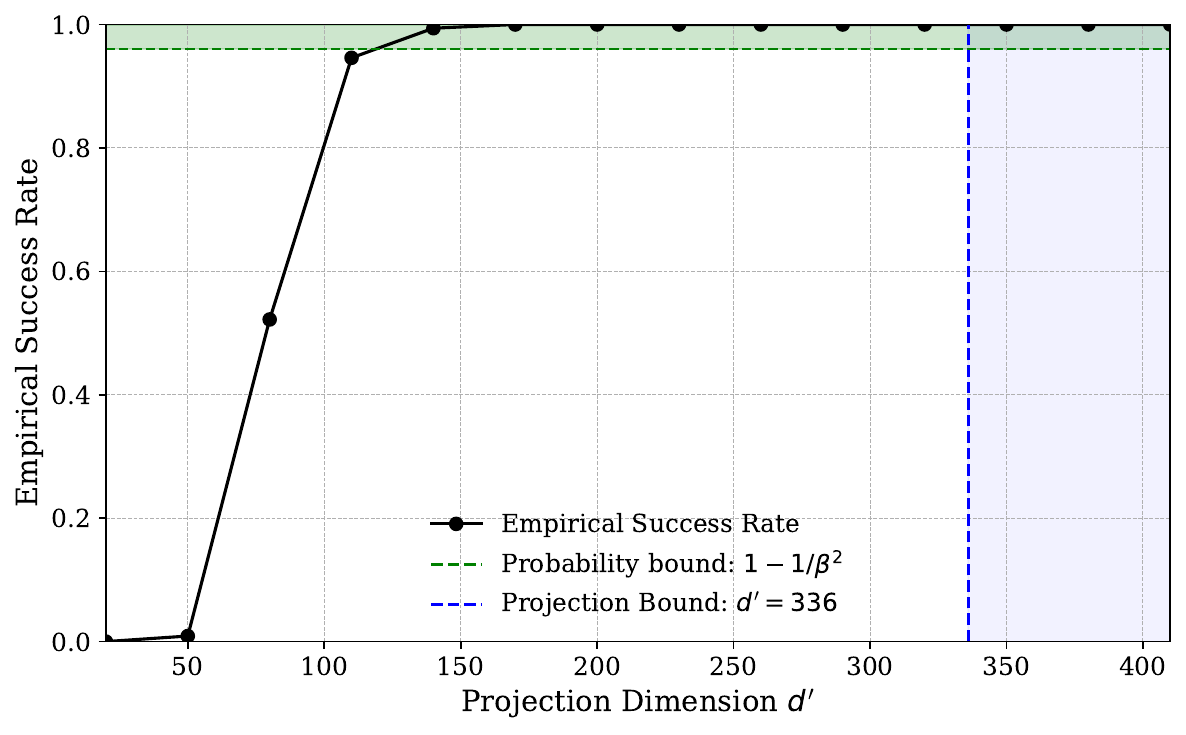}
    \caption{Empirical success rate of random projections preserving all pairwise distances within \( (1 \pm \delta) \), with \( \delta = \frac{2}{3} \), \( n = 100 \), and \( \beta = 5 \). Each point averages 1000 trials. The theoretical projection bound and success threshold are marked with dashed lines.}
    \label{fig:success-vs-dimension}
\end{figure}

\section{Appendix E: Analysis of the Worst-$m$ Approximation Strategy}

\subsection{Effect of \( m \) on the Minimization}

To alleviate the increasing conservatism in the acceptance condition as the number of evaluations grows, ECPv2 adopts a subsampling approach: it considers only the $m$ worst (i.e., lowest) function values from the evaluation history. This subset is defined as:
\[
\mathcal{I}_t^{m} = \arg \min_{\substack{S \subseteq \{1, \dots, t\} \\ |S| = m}} \sum_{i \in S} f(x_i).
\]
The hyperparameter $m \geq 1$ controls the aggressiveness of the approximation. Notably, for $m \geq n$, we recover the standard ECP acceptance condition. For smaller $m$, the approximation becomes less conservative and computationally cheaper. The modified acceptance condition becomes:
\begin{equation*}
\min_{i \in \mathcal{I}^{m}_t} \left( f(x_i) + \max\{\varepsilon_t, \varepsilon^\oslash_t\} \cdot \|  x -X_i \|_2 \right) \geq \max_j f(x_j).    
\end{equation*}

We denote the left-hand side of this inequality as $\phi_m(x)$ and compare it to its full-data counterpart, $\phi_{\text{full}}(x)$:
\[
\phi_{\text{full}}(x) = \min_{i} \left[ f(x_i) + \epsilon_t \cdot \|x - x_i\|_2 \right],
\]
\[
\phi_m(x) = \min_{i \in \mathcal{I}_m} \left[ f(x_i) + \epsilon_t \cdot \|x - x_i\|_2 \right],
\]
where $\mathcal{I}_m$ is the index set of the $m$ lowest function values, and we fix $\epsilon_t = 1$.

To quantify the quality of the Worst-$m$ approximation, we evaluate the ratio $\phi_{\text{full}}(x)/\phi_m(x)$ as a function of the sampling fraction $m/n$. Specifically, we sample $n = 100$ points uniformly within each function’s domain, select a random test point $x_{\text{test}}$, and compute the ratio for varying values of $m \in [1, 100]$. Each experiment is repeated over 1000 independent trials per function.

We fit the observed mean ratios to a logarithmic model of the form:
\[
\frac{\phi_{\text{full}}(x)}{\phi_m(x)} \approx a \cdot \log(m/n) + b,
\]
where $(a, b)$ are fitted parameters for each benchmark function. This functional form reflects the intuition that adding more points provides diminishing returns in surrogate quality. We perform this evaluation across four standard non-convex test functions.

\begin{figure}[h]
    \centering
    \includegraphics[width=0.9\linewidth]{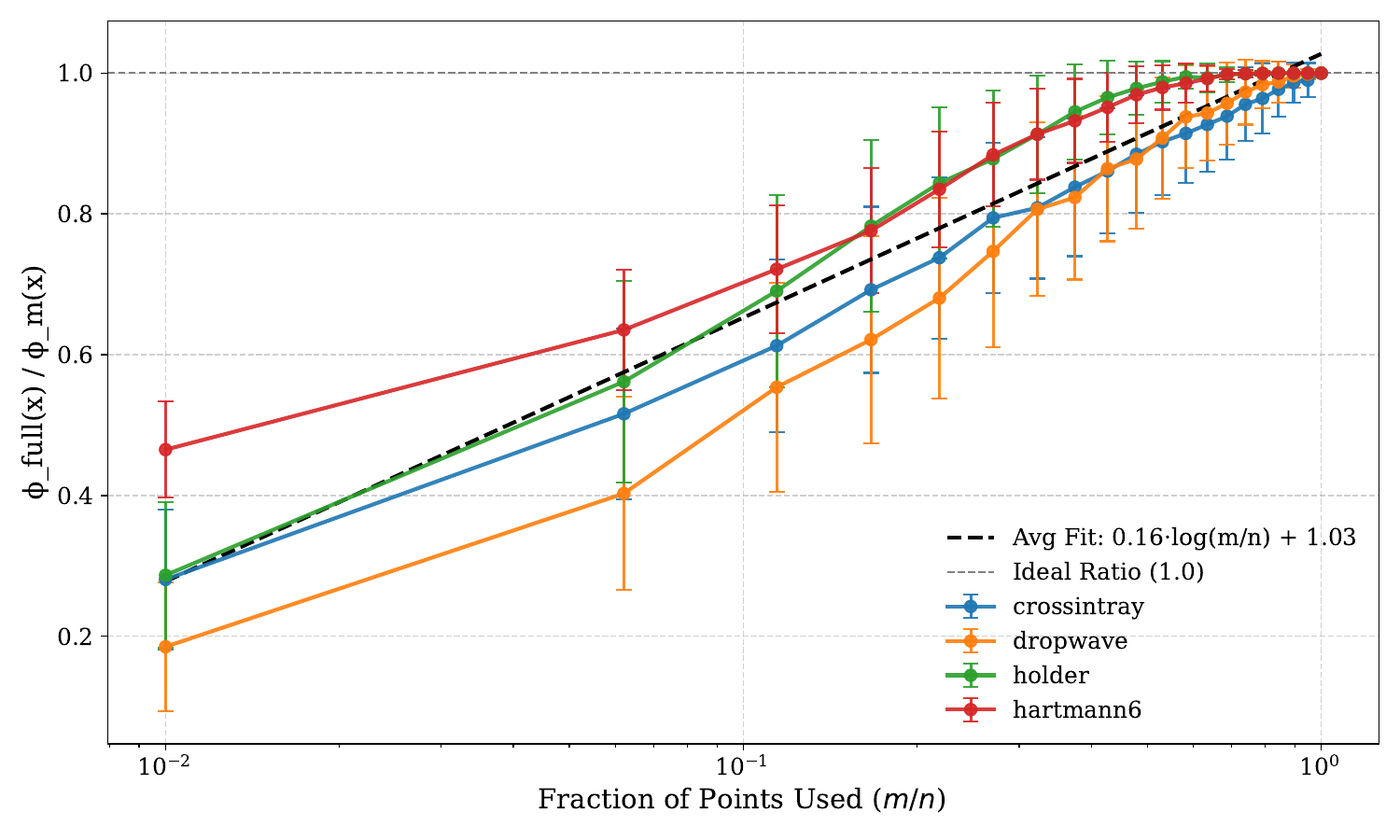}
    \caption{
        Ratio $\phi_{\text{full}}(x)/\phi_m(x)$ as a function of $m/n$, the fraction of retained evaluations, for four benchmark functions. Each curve represents the average over 1000 trials using $n = 100$ points. Error bars indicate half the standard deviation across trials. The black dashed line shows the average fit $\bar{a} \cdot \log(m/n) + \bar{b}$, with $(\bar{a}, \bar{b}) \approx (0.16, 1.03)$. The horizontal dashed line marks the ideal ratio of $1.0$.}
    \label{fig:phi_ratio}
\end{figure}

This empirical study reveals a consistent logarithmic relationship between $\phi_{\text{full}}(x)/\phi_m(x)$ and $m/n$ across all benchmark functions. The key insight is that even a small, well-chosen subset of evaluations ($m \ll n$) retains most of the quality of the full acceptance condition. This observation supports the use of the Worst-$m$ strategy to balance surrogate fidelity and computational efficiency, with practical implications for large-scale and real-time optimization scenarios.

\subsection{Effect of \( m \) in ECPv2}

To isolate the impact of \( m \) in the ECPv2 optimizer, we conducted an ablation study with fixed parameters \( \delta = \frac{2}{3} \) and \( \beta = 5 \).

We benchmarked ECPv2 across 12 diverse test functions, varying in dimension and landscape. For each function, we ran 100 independent optimization trials, each for $n = 1000$ evaluations. The tested values for \( m \) were:
\[
m \in \{2, 4, 8, 16, 32, 64, 128, 256, 512, 1000\}.
\]

We recorded the best-so-far function value over time and averaged across runs. Each figure reports the final best-so-far mean (with ± half standard deviation error bars) as a function of total CPU time.

\begin{remark}
ECPv2 only performs projection when the function’s intrinsic dimensionality \( d \) exceeds \( d' \):
$
d > d' = \left\lceil \frac{8 \log(5n)}{\delta^2 - \delta^3} \right\rceil,
$
where \( n \) is the evaluation budget. When \( d \leq d' \), projection is skipped regardless of the \( m \) setting.   
\end{remark}

Figures~\ref{fig:ecpv2_ablation_m_part1}, \ref{fig:ecpv2_ablation_m_part2}, \ref{fig:ecpv2_ablation_m_part3} illustrate the impact of varying \( m \) on optimization performance across all functions. As expected, increasing \( m \) leads to longer computational times due to the higher cost of minimization (especially for very large values of $m$). At the same time, larger \( m \) values tend to achieve better optima on average (not always), though with diminishing returns beyond a certain point.

In almost all experiments across diverse functions and dimensionalities, the three choices of \( m \in \{256, 512, 1000\} \) are the slowest, with larger values of \( m \) generally leading to slower performance. Notably, setting \( m = 1000 \) effectively disables the worst-\( m \) projection since \( n = 1000 \), thereby recovering the ECP minimization. This observation highlights the impact of using larger projection dimensions on wall-clock time.

These results support the findings in Appendix~E, where the Worst-$m$ strategy closely approximated the exact solution even for modest values of \( m \). We conclude that very large values of \( m \) are not necessary in practice: smaller settings can achieve comparable or even superior results in significantly less time (up to 4×–5× faster in CPU seconds).

Interestingly, in some cases, smaller values of \( m \) outperform larger ones or even the exact method. For example, in the Eggholder problem, \( m = 128 \) achieves the best score while being 2.5\(\times\) faster than exact minimization. In the Rosenbrock function (500 dimensions), \( m = 16 \) yields the best performance while being 1.5\(\times\) faster. We attribute this to the relaxed acceptance condition used when \( m \) is small: it is more permissive, allowing better candidate points to be accepted when the acceptance region, controlled by \( \varepsilon_t \), is still small. In contrast, with larger \( m \), the acceptance region tightens—especially when coupled with lower values of \( \varepsilon_t \)—potentially rejecting high-quality candidates prematurely.

These insights further motivate the use of small projection dimensions \( m \ll n \), not only for computational efficiency but also for effective optimization performance.

\begin{figure*}[t]
    \centering

    \begin{subfigure}[t]{0.47\textwidth}
        \centering
        \includegraphics[width=\linewidth]{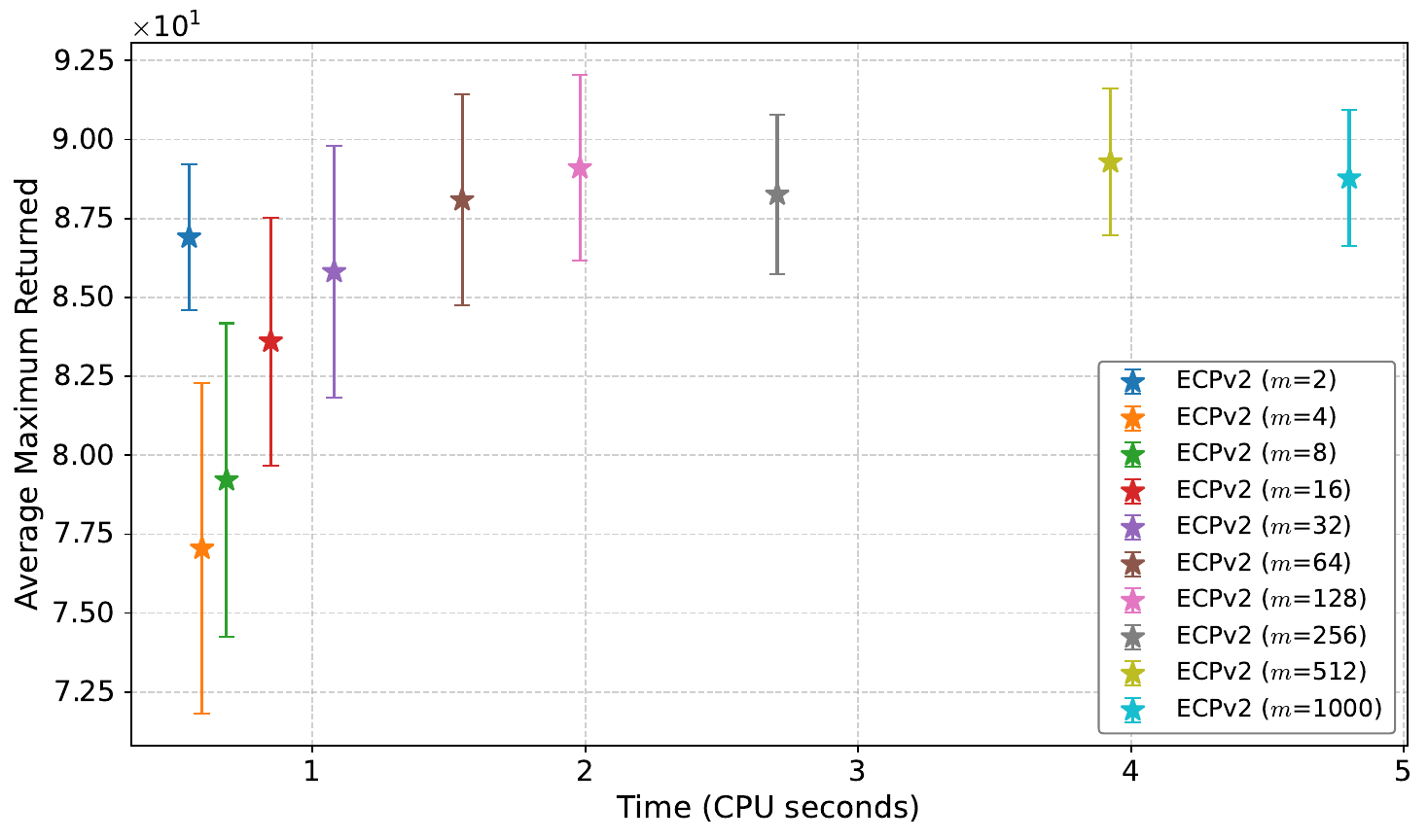}
        \caption{Eggholder}
    \end{subfigure}
    \hfill
    \begin{subfigure}[t]{0.47\textwidth}
        \centering
        \includegraphics[width=\linewidth]{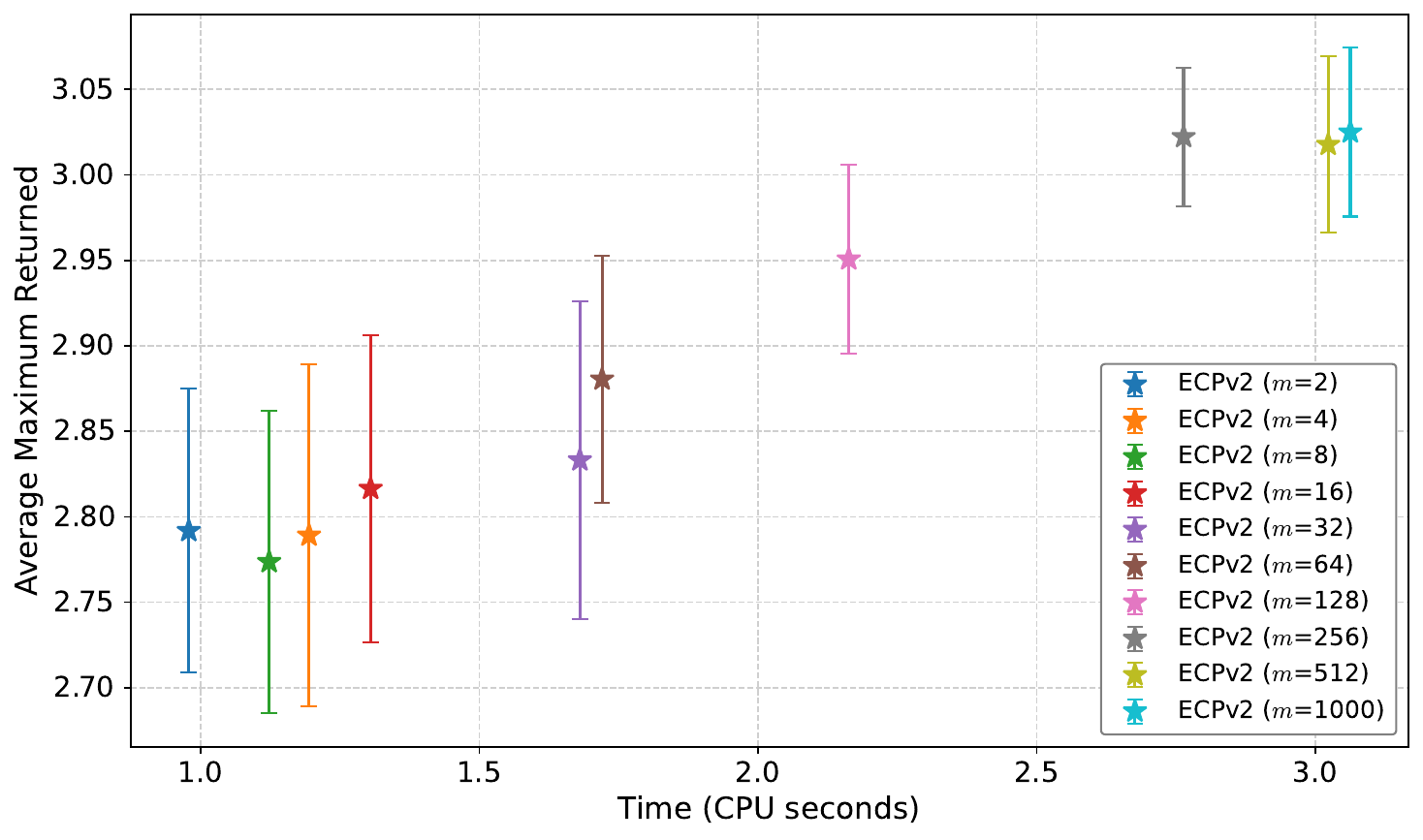}
        \caption{Hartmann (6D)}
    \end{subfigure}

    \vspace{1em}

    \begin{subfigure}[t]{0.47\textwidth}
        \centering
        \includegraphics[width=\linewidth]{figures/ablations_m/final_points_with_std_himmelblau.pdf}
        \caption{Himmelblau}
    \end{subfigure}
    \hfill
    \begin{subfigure}[t]{0.47\textwidth}
        \centering
        \includegraphics[width=\linewidth]{figures/ablations_m/final_points_with_std_holder.pdf}
        \caption{Holder}
    \end{subfigure}

    \caption{Ablation study on the projection dimension \( m \) in ECPv2 (Part 1 of 3). Each point represents the final average of the best score after \( n = 1000 \) evaluations, averaged over 100 runs, with \( \pm \) half the standard deviation. Projection is skipped as the function dimension \( d < d' = \left\lceil \frac{8 \log(5n)}{\delta^2 - \delta^3} \right\rceil \).}
    \label{fig:ecpv2_ablation_m_part2}
\end{figure*}

\begin{figure*}[h]
    \centering
    \vspace{1em}

    \begin{subfigure}[t]{0.48\textwidth}
        \centering
        \includegraphics[width=\linewidth]{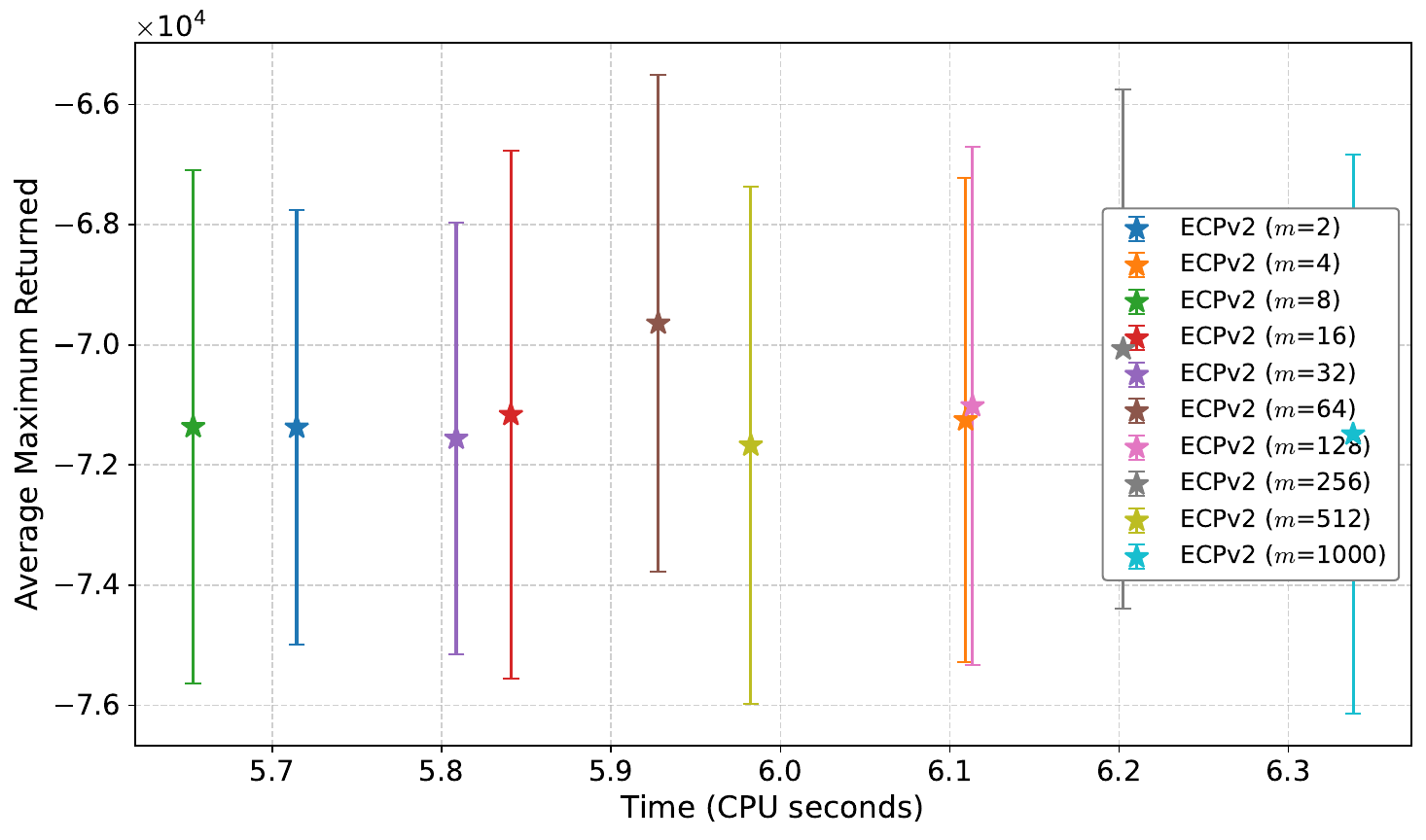}
        \caption{Powell (100D)}
    \end{subfigure}
    \hfill
    \begin{subfigure}[t]{0.48\textwidth}
        \centering
        \includegraphics[width=\linewidth]{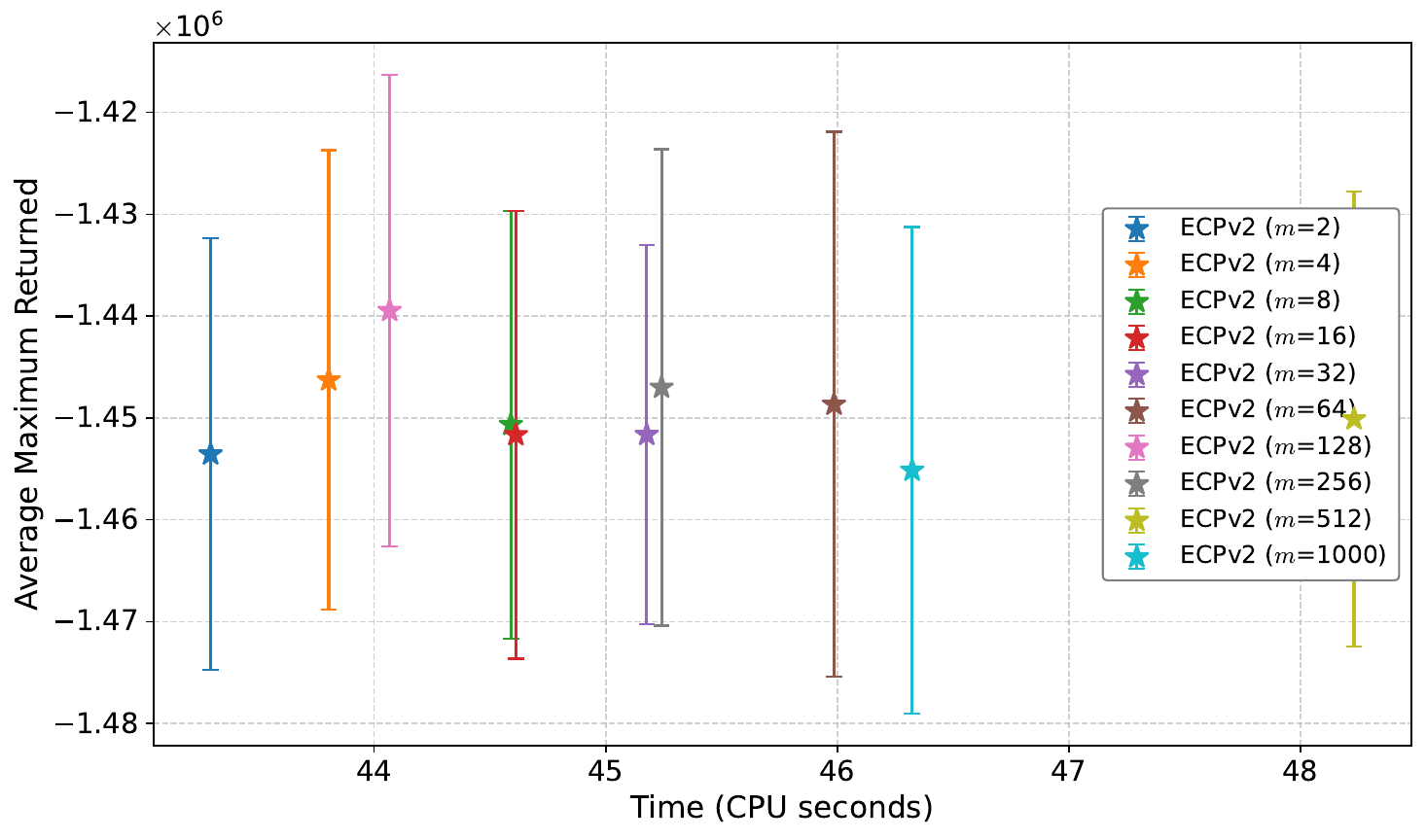}
        \caption{Powell (1000D)}
    \end{subfigure}
    \vspace{1em}

    \begin{subfigure}[t]{0.48\textwidth}
        \centering
        \includegraphics[width=\linewidth]{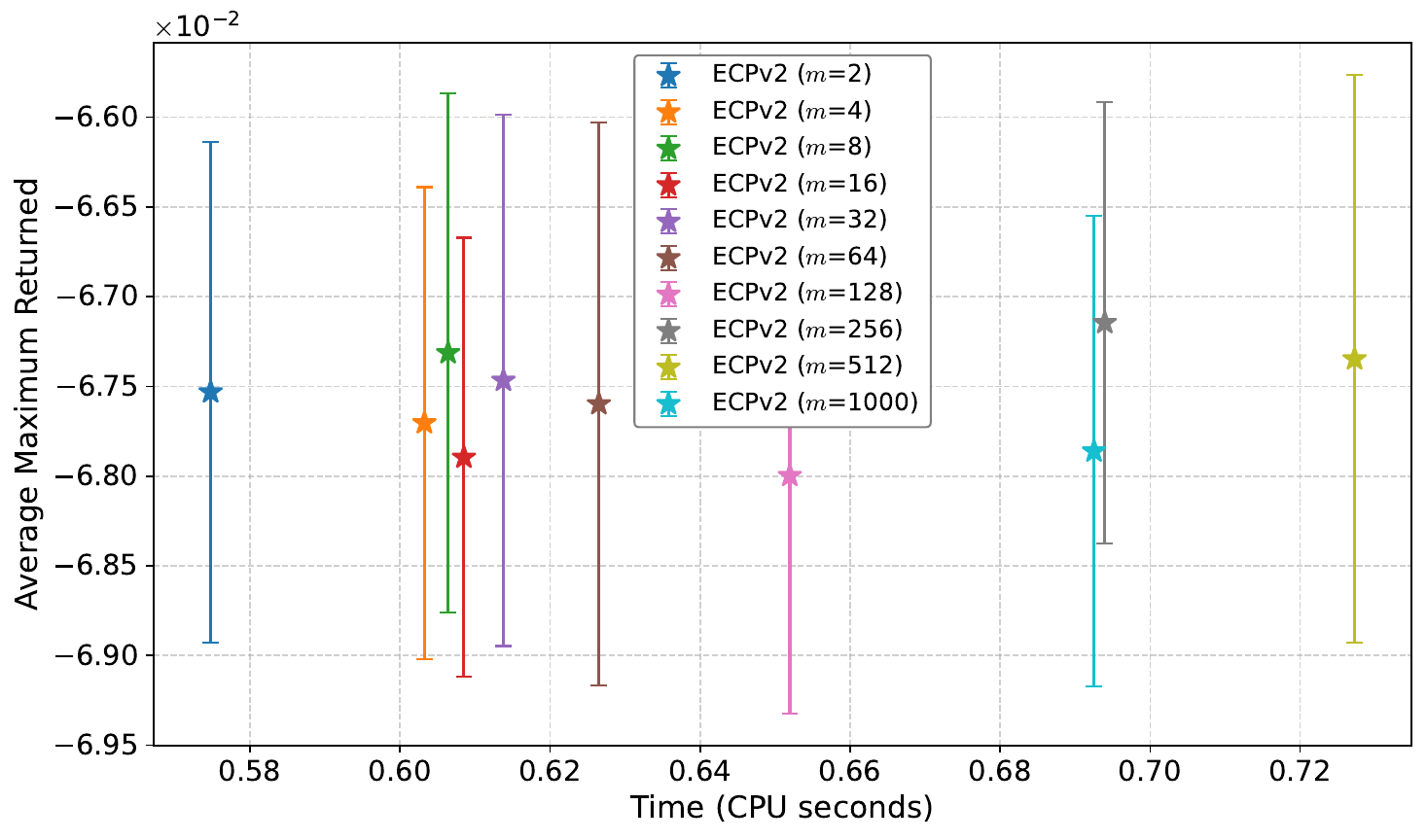}
        \caption{Rosenbrock (100D)}
    \end{subfigure}
    \hfill
    \begin{subfigure}[t]{0.48\textwidth}
        \centering
        \includegraphics[width=\linewidth]{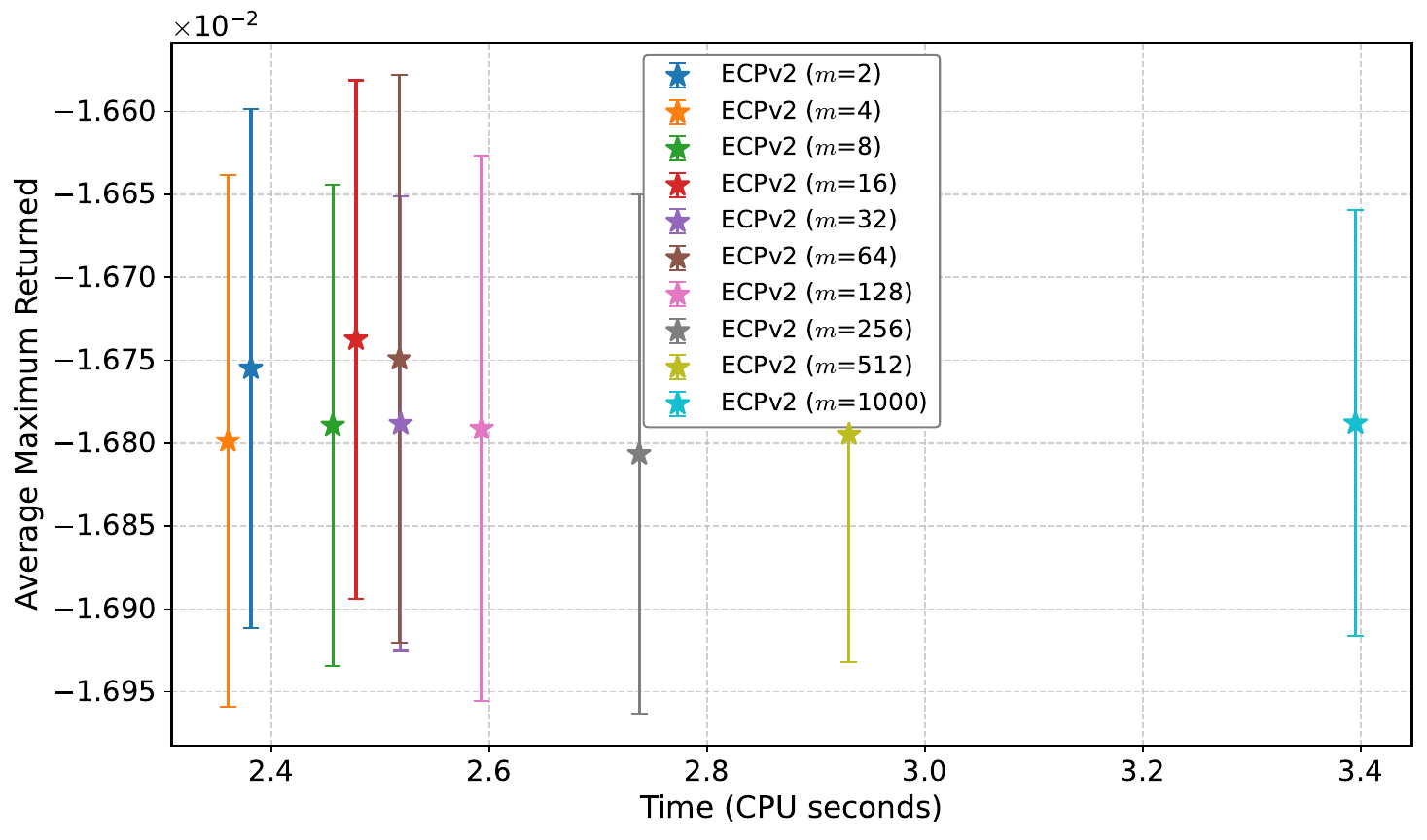}
        \caption{Rosenbrock (500D)}
    \end{subfigure}

    \vspace{1em}

    \caption{Ablation study on the projection dimension \( m \) in ECPv2 (Part 2 of 3). Each point represents the final average of the best score after \( n = 1000 \) evaluations, averaged over 100 runs, with \( \pm \) half the standard deviation.}
    \label{fig:ecpv2_ablation_m_part3}
\end{figure*}

\begin{figure*}[t]
    \centering

    \begin{subfigure}[t]{0.47\textwidth}
        \centering
        \includegraphics[width=\linewidth]{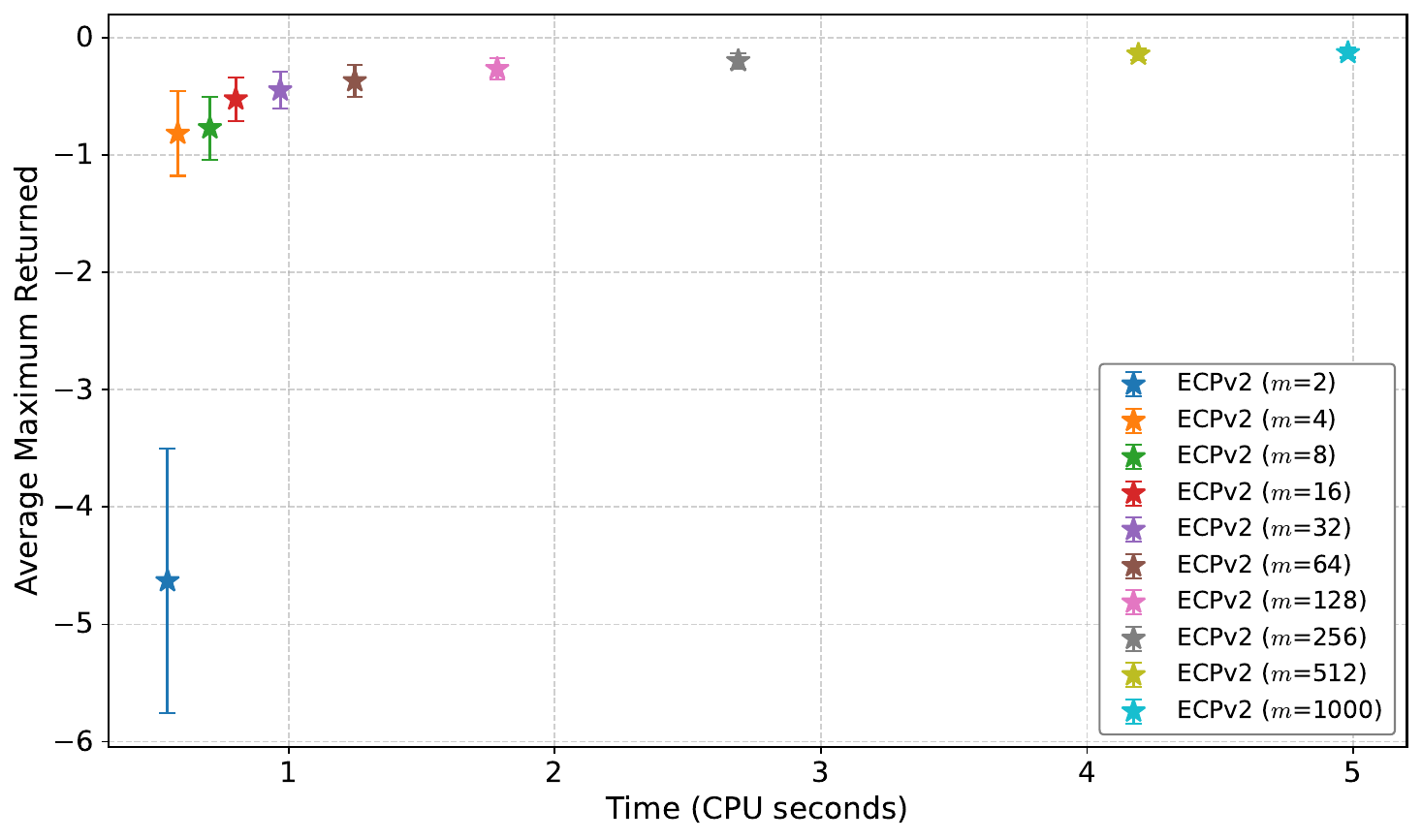}
        \caption{Ackley}
    \end{subfigure}
    \hfill
    \begin{subfigure}[t]{0.47\textwidth}
        \centering
        \includegraphics[width=\linewidth]{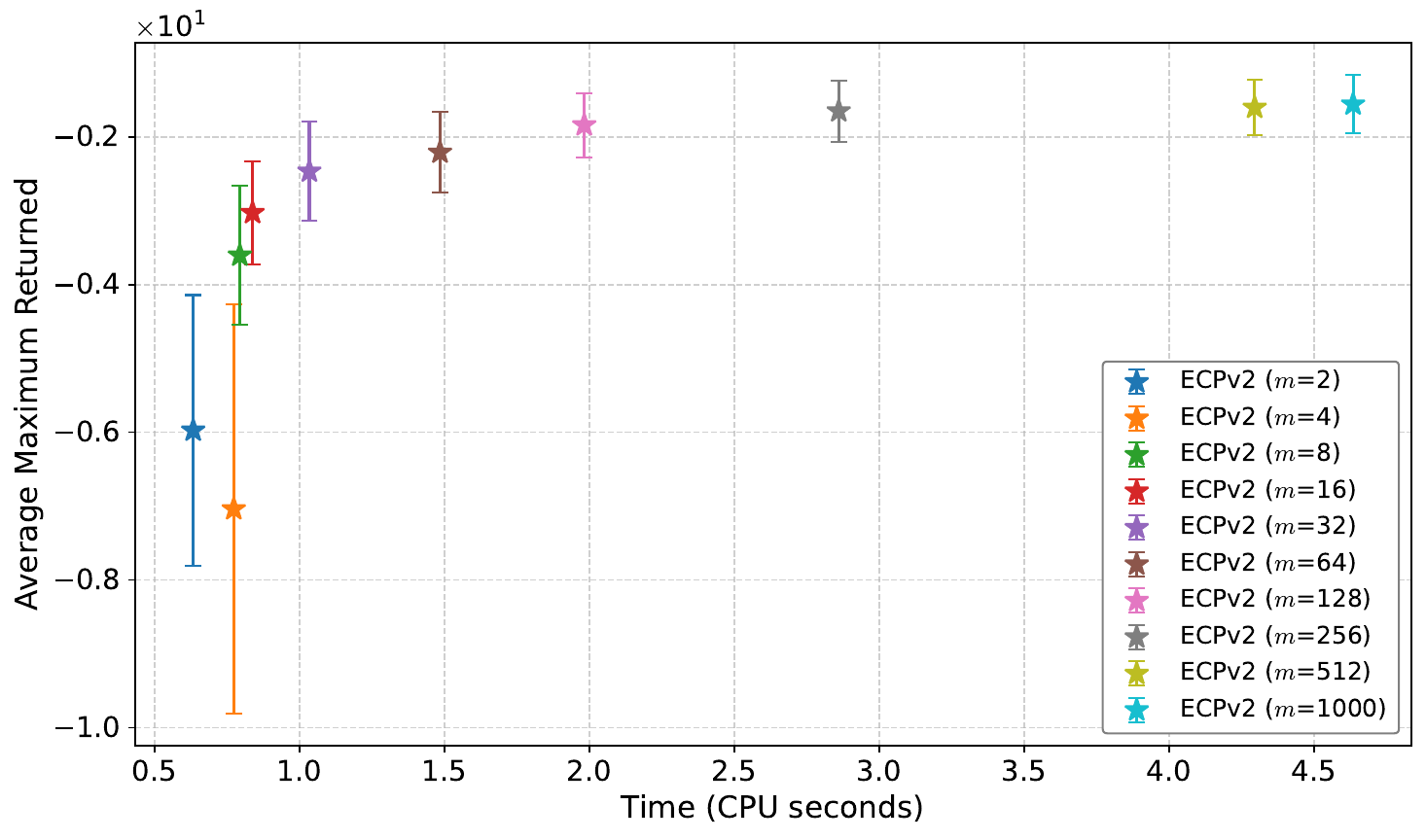}
        \caption{Bukin}
    \end{subfigure}

    \vspace{1em}

    \begin{subfigure}[t]{0.47\textwidth}
        \centering
        \includegraphics[width=\linewidth]{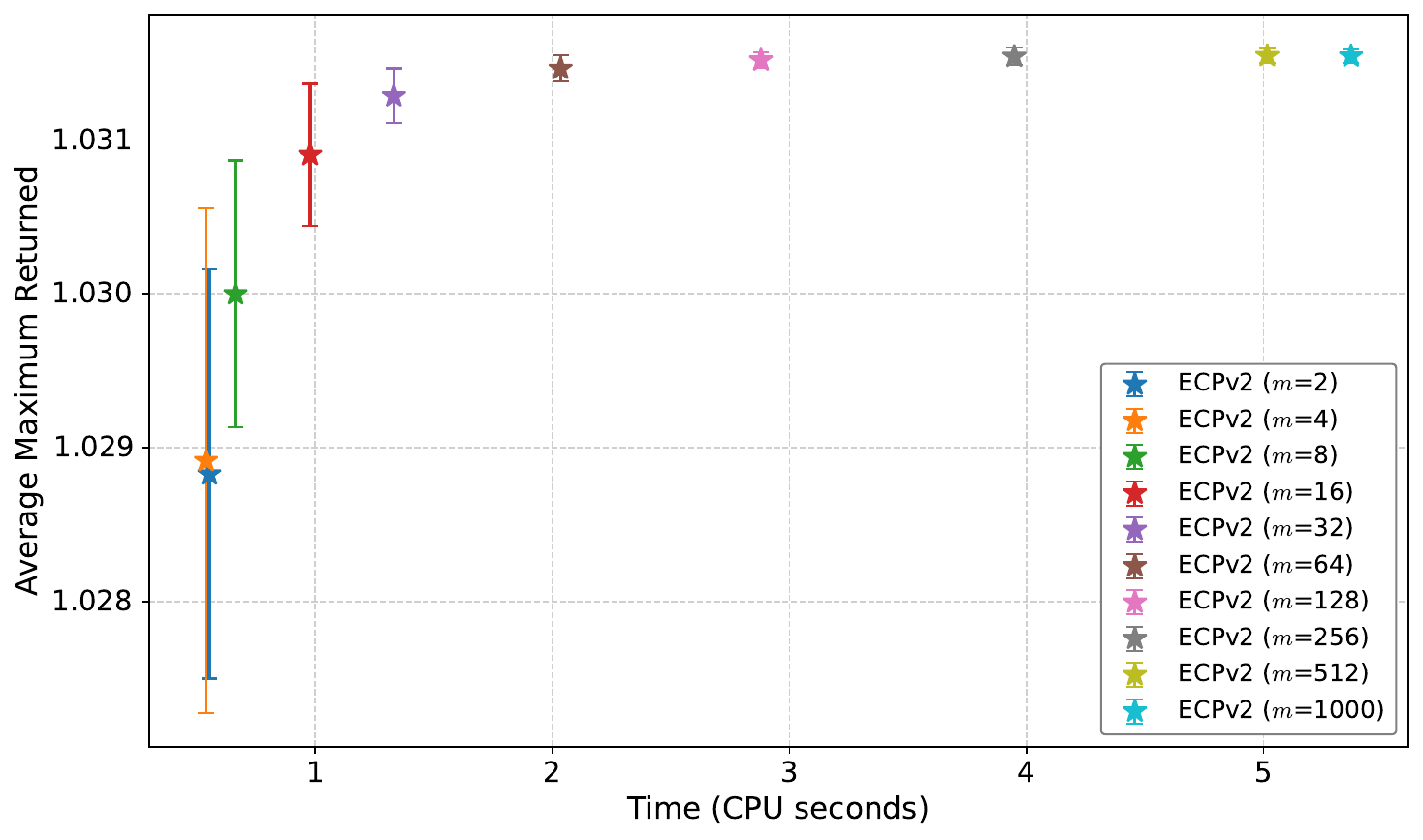}
        \caption{Camel}
    \end{subfigure}
    \hfill
    \begin{subfigure}[t]{0.47\textwidth}
        \centering
        \includegraphics[width=\linewidth]{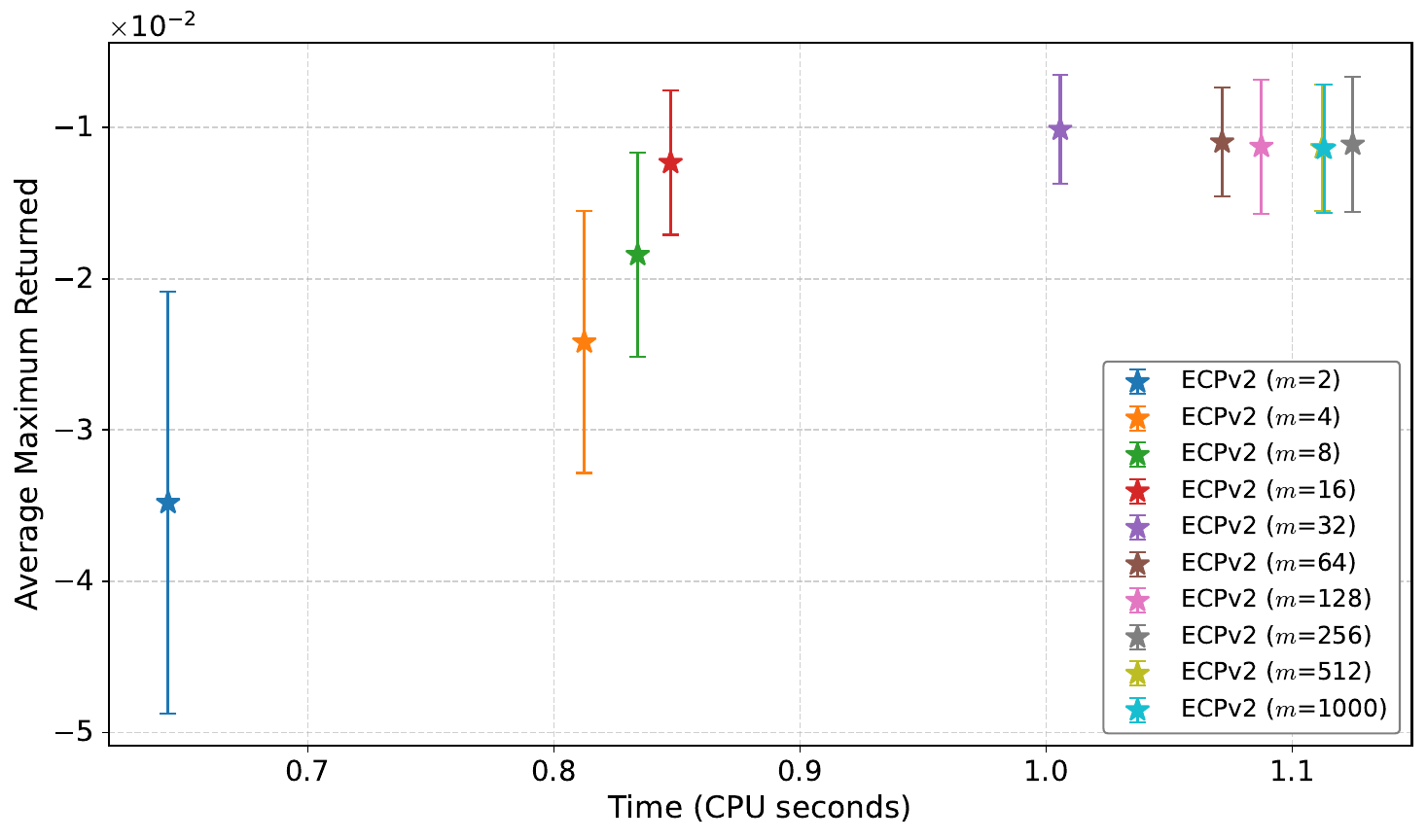}
        \caption{Colville}
    \end{subfigure}

    \vspace{1em}

    \begin{subfigure}[t]{0.47\textwidth}
        \centering
        \includegraphics[width=\linewidth]{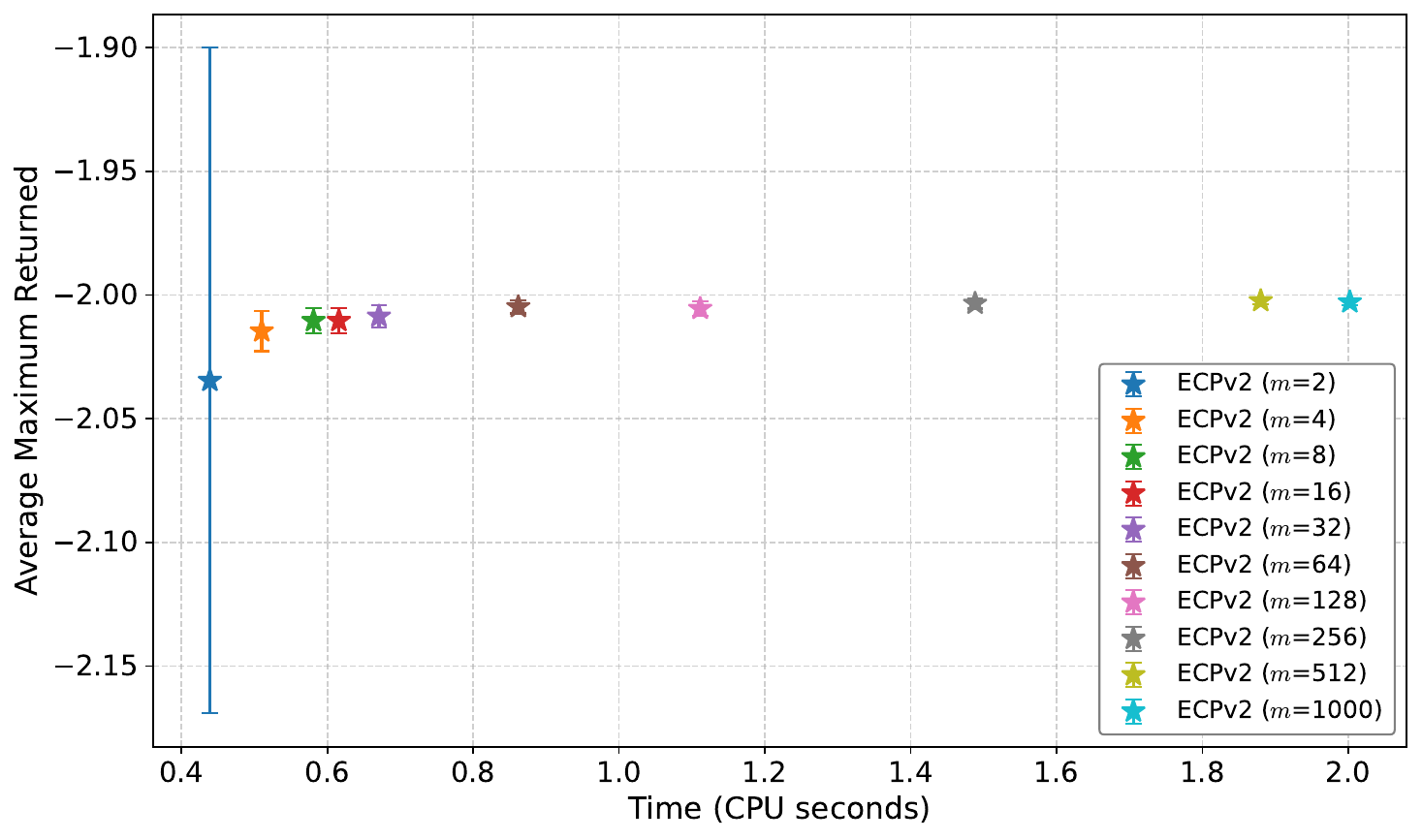}
        \caption{Damavandi}
    \end{subfigure}
    \hfill
    \begin{subfigure}[t]{0.47\textwidth}
        \centering
        \includegraphics[width=\linewidth]{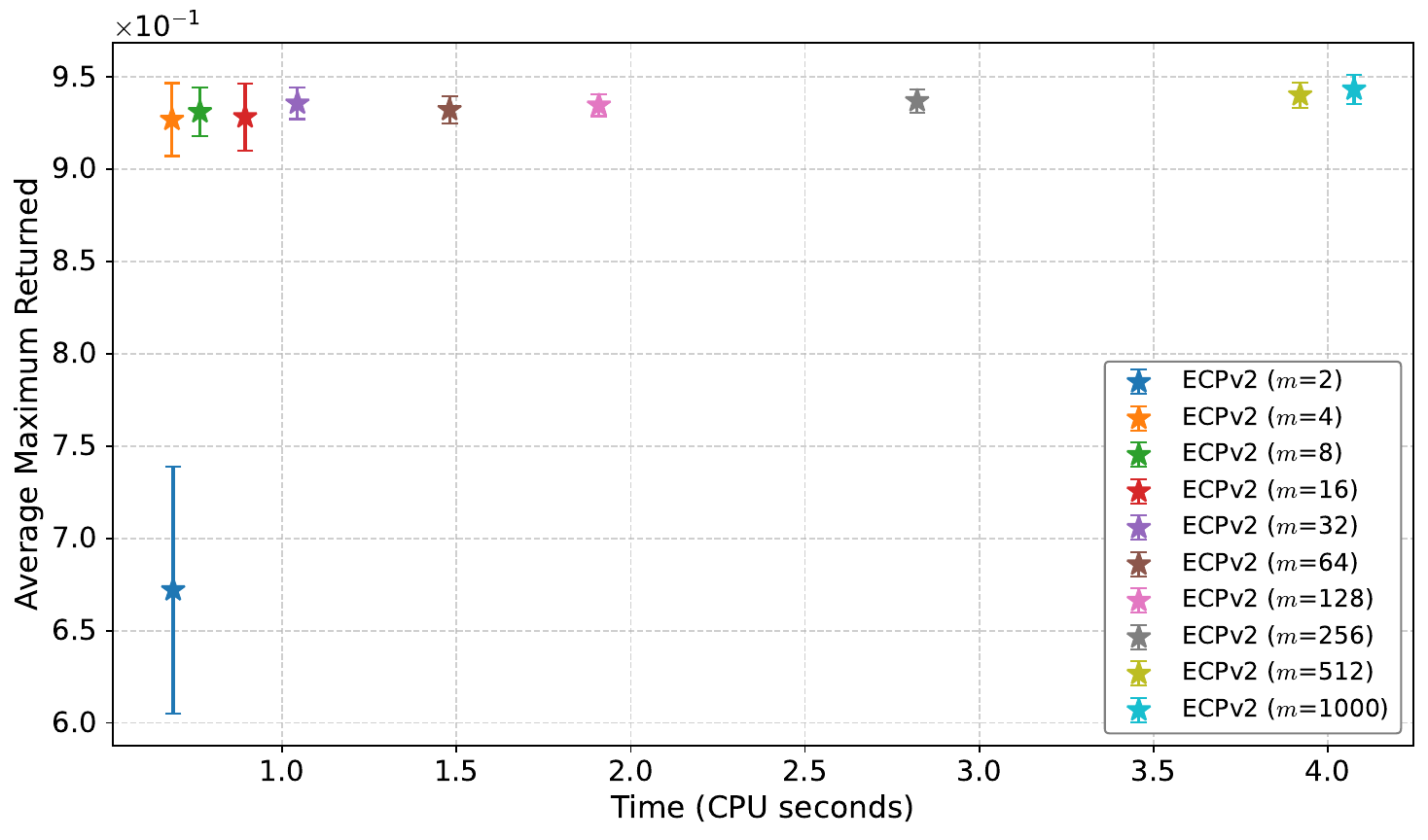}
        \caption{Drop-Wave}
    \end{subfigure}

    \caption{Ablation study on the projection dimension \( m \) in ECPv2 (Part 3 of 3). Each point represents the final average of the best score after \( n = 1000 \) evaluations, averaged over 100 runs, with \( \pm \) half the standard deviation. Projection is skipped as the function dimension \( d < d' = \left\lceil \frac{8 \log(5n)}{\delta^2 - \delta^3} \right\rceil \).}
    \label{fig:ecpv2_ablation_m_part1}
\end{figure*}

\section{Appendix F: Empirical Validation of the Adaptive Lower Bound $\varepsilon_t^\oslash$}

\begin{figure*}[t]
    \centering

    \begin{subfigure}[t]{0.47\textwidth}
        \centering
        \includegraphics[width=\linewidth]{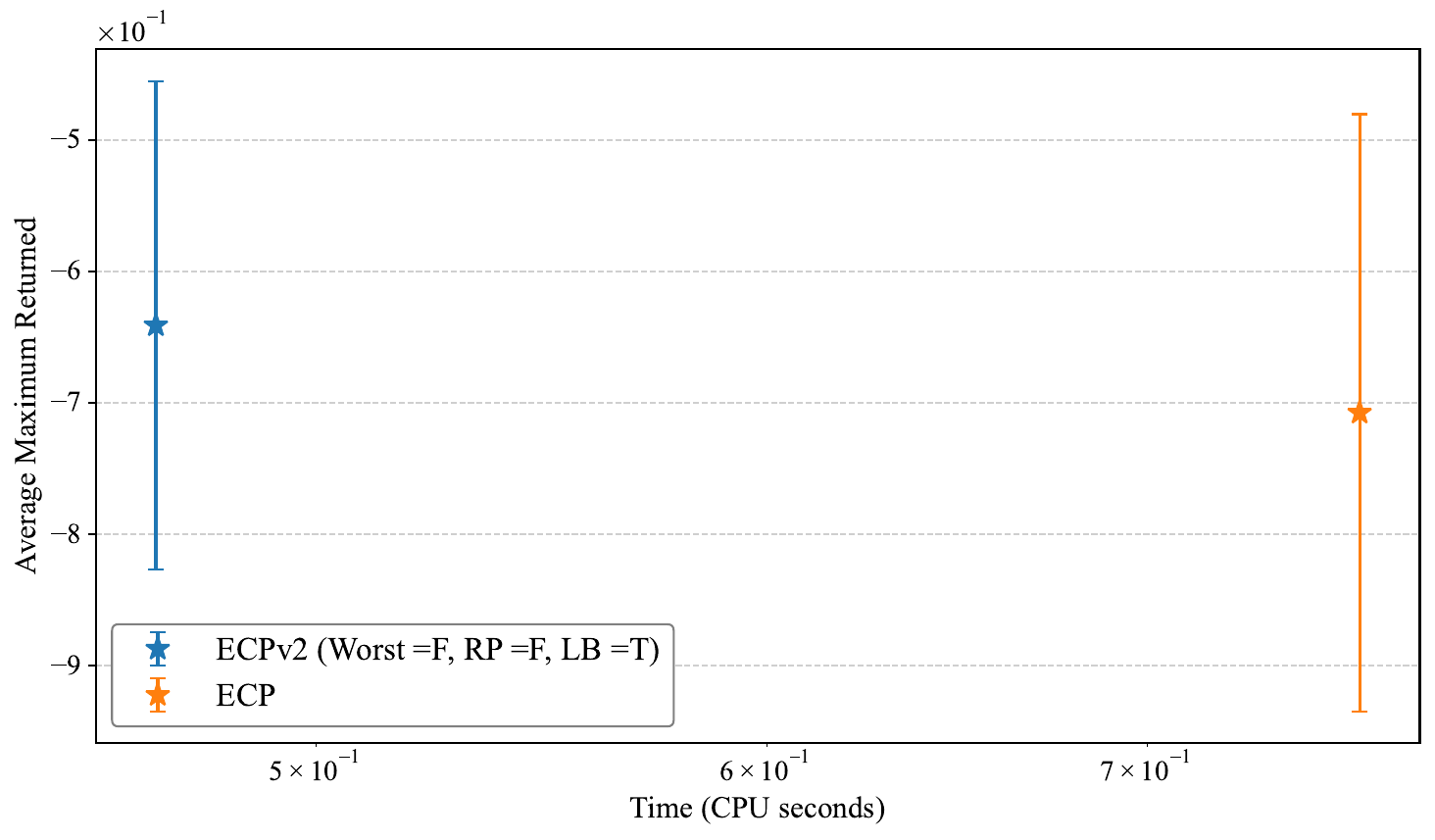}
        \caption{Ackley}
    \end{subfigure}
    \hfill
    \begin{subfigure}[t]{0.47\textwidth}
        \centering
        \includegraphics[width=\linewidth]{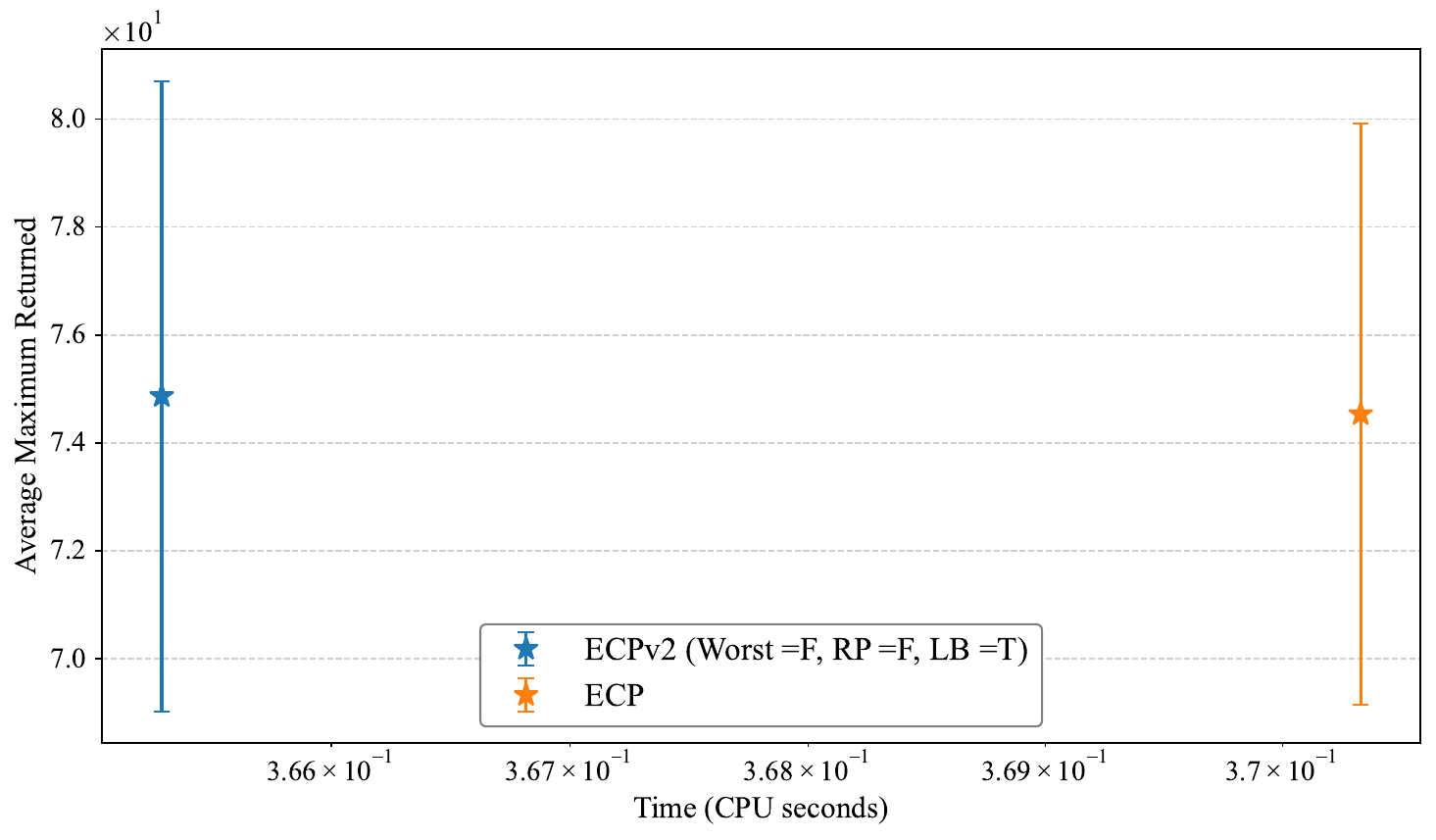}
        \caption{Eggholder}
    \end{subfigure}

    \vspace{1em}

    \begin{subfigure}[t]{0.47\textwidth}
        \centering
        \includegraphics[width=\linewidth]{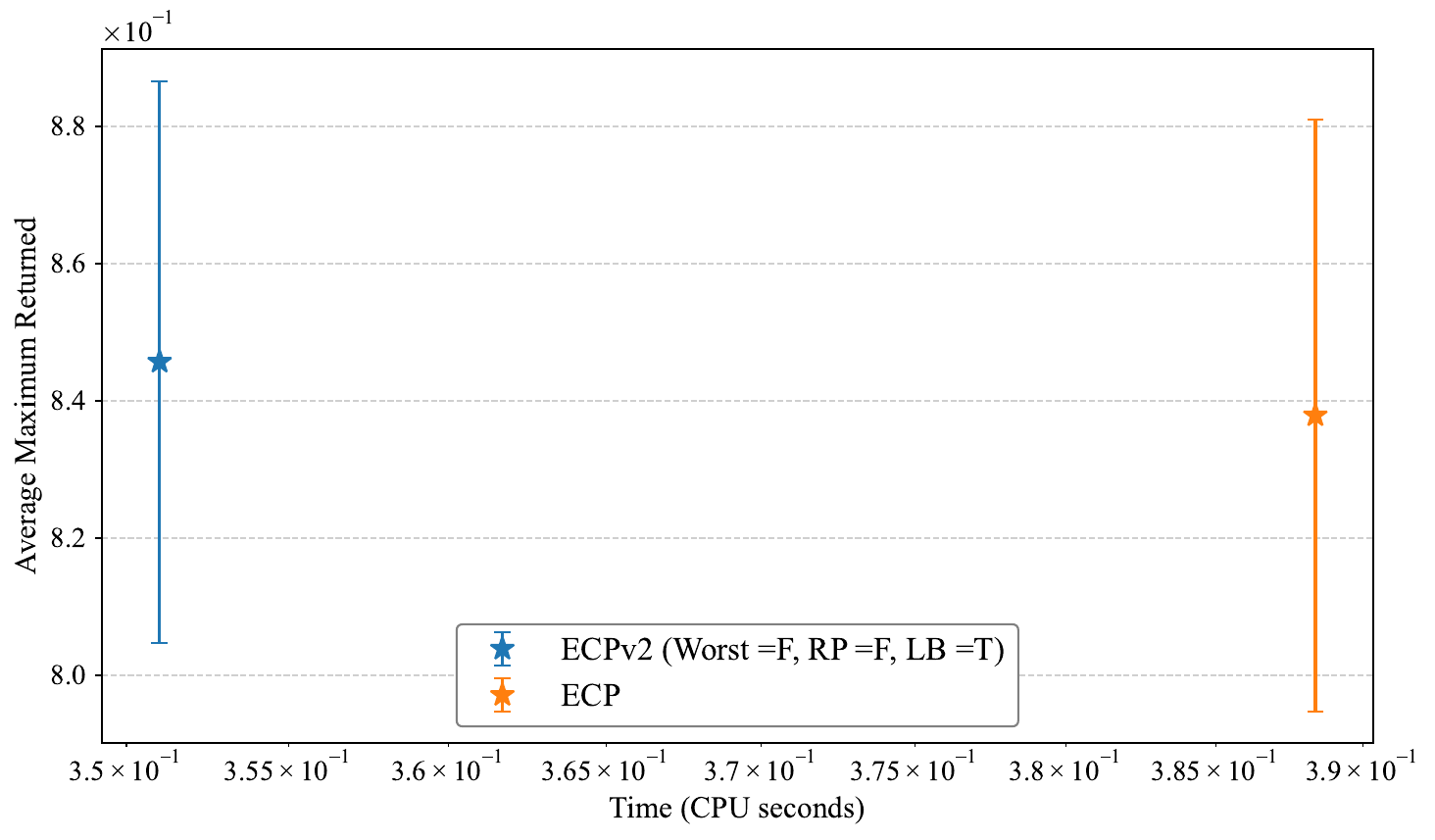}
        \caption{Dropwave}
    \end{subfigure}
    \hfill
    \begin{subfigure}[t]{0.47\textwidth}
        \centering
        \includegraphics[width=\linewidth]{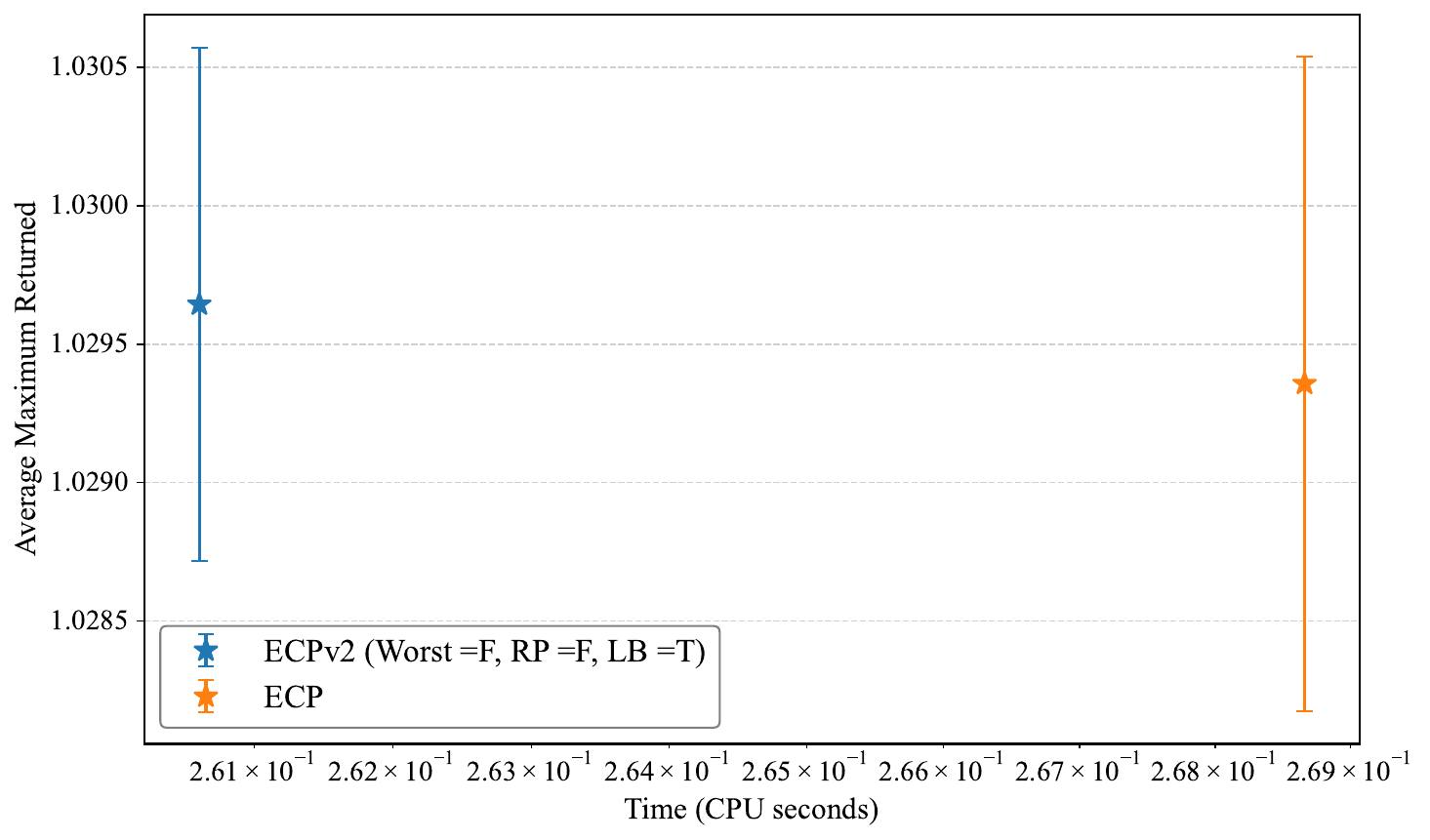}
        \caption{Three-Hump Camel}
    \end{subfigure}

    \caption{Ablation comparison between ECP and the \texttt{ECPv2 (LB only)} variant on four benchmark functions. Each point shows the final average of the best score after 100 evaluations, averaged over 100 runs, with $\pm \frac{1}{2}$ standard deviation. The lower-bound enhancement alone provides up to $2\times$ speedups with consistent solution quality.}
    \label{fig:ecpv2_ablation_lb_multi}
\end{figure*}

\subsection{Validation of Necessity}

To empirically validate the necessity and effectiveness of the adaptive lower bound $\varepsilon_t^\oslash$ introduced in ECPv2, we conduct a focused experiment using the Ackley function.

We fix the algorithm at iteration $t = 5$ and simulate a scenario where five previously accepted points $\{x_1, \dots, x_5\}$ are sampled uniformly from the domain $\mathcal{X}$. We evaluate these points to obtain their function values $\{f(x_1), \dots, f(x_5)\}$ and denote the current maximum as $f_{\max}$.

For a grid of $\varepsilon$ values in the range $[0.01, 1.5]$, we draw $1000$ independent candidate points $\tilde{x} \sim \mathcal{U}(\mathcal{X})$ and compute the empirical \emph{acceptance ratio}, i.e., the fraction of points satisfying the ECP acceptance condition:
\[
\min_{i=1,\dots,t} \left( f(x_i) + \varepsilon \cdot \|\tilde{x} - x_i\|_2 \right) \geq f_{\max}.
\]

As shown in Lemma~1, a necessary condition for any point to be accepted is that $\varepsilon \geq \varepsilon_t^\oslash$, where:
\[
\varepsilon_t^\oslash = \frac{f_{\max_t} - f_{\min_t}}{\mathrm{diam}(\mathcal{X})}.
\]
This value ensures that the acceptance region is not trivially empty.

Figure~\ref{fig:acceptance-ratio} illustrates the empirical acceptance ratio as a function of $\varepsilon$. We observe a sharp phase transition: for $\varepsilon < \varepsilon_t^\oslash$, the acceptance ratio is virtually zero, indicating that all candidate points are rejected. Once $\varepsilon$ surpasses the threshold, the acceptance ratio increases rapidly, confirming that $\varepsilon_t^\oslash$ effectively marks the boundary of a non-trivial search region.

\begin{figure}[H]
    \centering
    \includegraphics[width=0.43\textwidth]{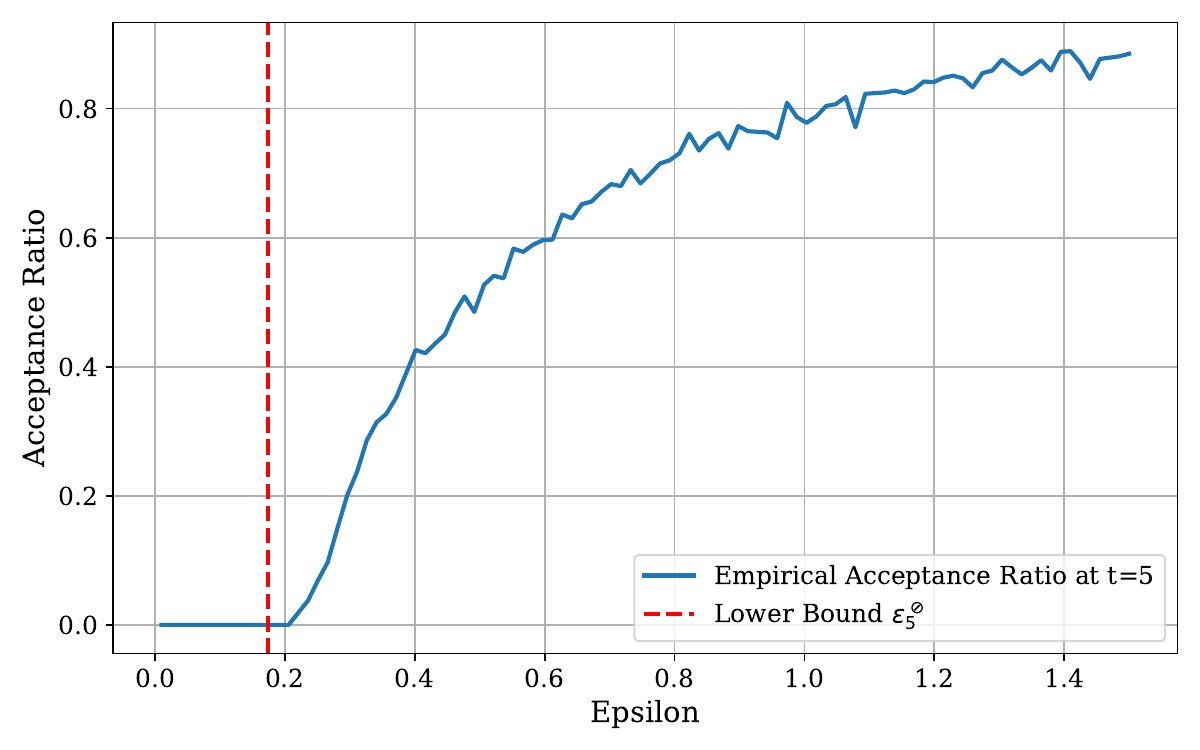}
    \caption{Empirical acceptance ratio at iteration $t=5$ on the 10D Ackley function. The red dashed line denotes the theoretical lower bound $\varepsilon_5^\oslash$. This confirms the bound’s role in preventing empty acceptance regions.}
    \label{fig:acceptance-ratio}
\end{figure}

This experiment offers empirical support for the theoretical lower bound $\varepsilon_t^\oslash$, demonstrating its critical role in maintaining optimization progress. Without it, the algorithm may experience excessive early rejections, stalling exploration. ECPv2’s adaptive enforcement of this threshold ensures that each iteration maintains a viable and meaningful acceptance region.

\subsection{Effect of the Lower Bound on ECPv2}

To better understand the contribution of the lower bound enhancement in ECPv2, we performed an ablation study comparing two variants: ECP and ECPv2 (LB only): a variant where only the lower bound enhancement is enabled (Worst-$m$ = \textbf{False}, Random Projections = \textbf{False}, Lower Bound = \textbf{True}).

Both algorithms were benchmarked on four benchmark functions over 100 independent runs, each with a budget of 100 evaluations. Figure~\ref{fig:ecpv2_ablation_lb_multi} reports the final average performance and runtime.

We observe that enabling the lower bound mechanism alone yields a significant runtime improvement, approximately a $2\times$ speedup and even outperforming the optimization performance. This demonstrates that even without projection-based exploration or Worst-$m$ sampling, incorporating the lower bound check provides a strong computational advantage, especially with lower-budgets (such as $n=100$).

\begin{figure*}[h]
    \centering

    \begin{subfigure}[t]{0.47\linewidth}
        \centering
        \includegraphics[width=\linewidth]{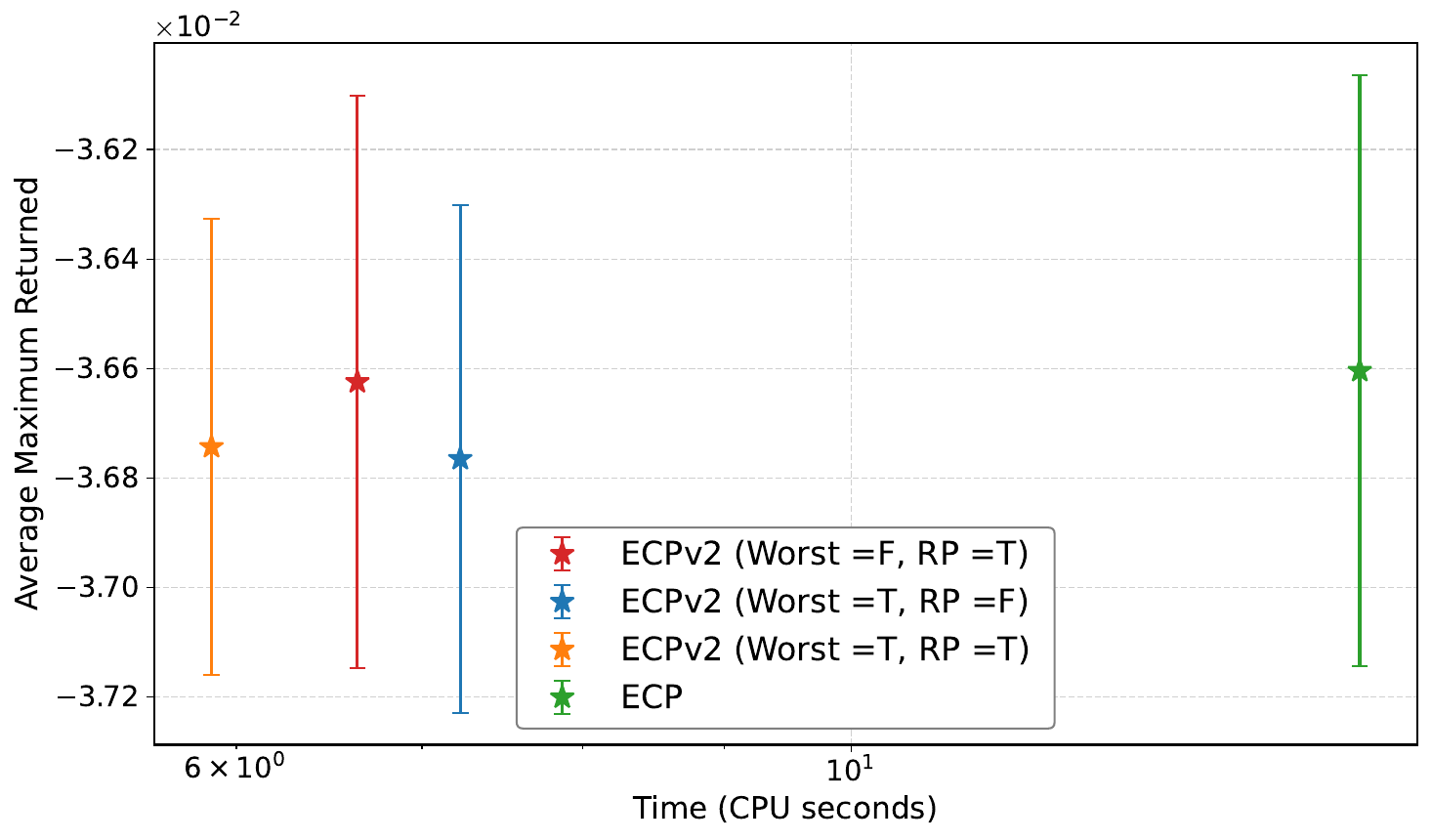}
        \caption{Rosenbrock (200D)}
    \end{subfigure}
    \hfill
    \begin{subfigure}[t]{0.47\linewidth}
        \centering
        \includegraphics[width=\linewidth]{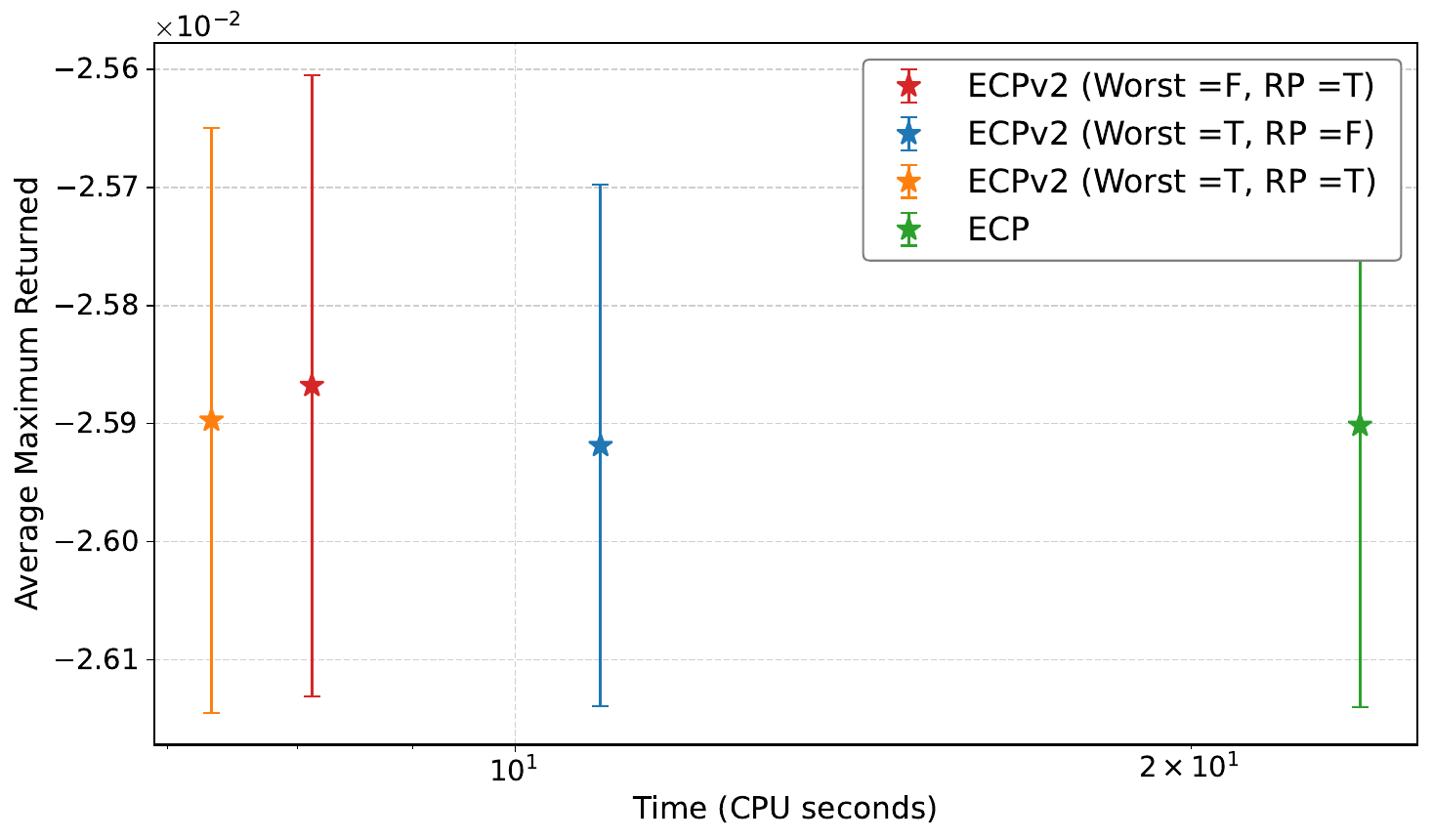}
        \caption{Rosenbrock (300D)}
    \end{subfigure}

    \caption{Final performance vs. CPU time for ECPv2 ablations on Rosenbrock in higher dimensions. Error bars represent $\pm \frac{1}{2}$ standard deviation across 100 runs.}
    \label{fig:ecpv2_ablation_rosenbrock_multi}
\end{figure*}

\section{Appendix G: On the Combination of Random Projection and Worst-$m$}

To understand the impact of combining the proposed Random Projection and Worst-$m$ components in \textbf{ECPv2}, we perform an ablation study by enabling their different combinations.

Each configuration is evaluated on the rosenbrock function (with 500 dimensions), with $100$ independent runs and a budget of $n = 5000$ function evaluations. The figure below shows the final performance (average maximum value returned) plotted against CPU time (seconds), with error bars denoting $\pm \frac{1}{2}$ standard deviation.

\paragraph{Findings.}
\begin{itemize}
    \item All ECPv2 variants are significantly faster than the ECP, while maintaining comparable or better solution quality.
    \item \textbf{ECPv2 (full)} (all true) is the \emph{fastest} across runs, showing that combining all components leads to the fastest performance.
    \item The \emph{Worst-$m$ selection} is especially effective in high-dimensional problems like Rosenbrock (300D and 400D) and substantially reducing computational cost.
\end{itemize}

These results confirm that ECPv2 is a robust, scalable, and computationally efficient extension of ECP, with strong empirical performance.

\section{Appendix H: Additional Results on ECPv2 vs. ECP}

\subsection{Statistical Significance}

To assess whether the performance improvements in ECPv2 over ECP are statistically significant, we conducted a Wilcoxon signed-rank test on the per-objective rankings of average runtimes across 13 benchmark functions (\texttt{michalewicz}, \texttt{perm}, \texttt{perm10}, \texttt{powell100}, \texttt{powell1000}, \texttt{rastrigin}, \texttt{rosenbrock}, \texttt{rosenbrock100}, \texttt{rosenbrock200}, \texttt{rosenbrock300}, \texttt{rosenbrock500}, \texttt{schaffer}, and \texttt{schubert}). This non-parametric test is suitable for comparing paired samples when the assumption of normality may not hold.

The test yielded a statistic of $W = 7.0$ with a p-value of $0.0034$, indicating a statistically significant difference ($p < 0.01$) in favor of ECPv2. The improved method consistently achieved faster average runtimes across most objectives. These results suggest that the enhancements introduced in ECPv2 contribute to a meaningful reduction in computational cost.

\subsection{High-dimensional Problems}

We evaluated ECPv2 against the ECP on two challenging high-dimensional benchmark problems: Rosenbrock (500 dimensions) and Powell (1000 dimensions). In both cases, ECPv2 not only demonstrated approximately 2× faster convergence compared to ECP, but also achieved higher optimization scores. These results highlight ECPv2's scalability and efficiency in very high-dimensional settings. Detailed results are shown in Figure~\ref{fig:high_dim_ecp}.

\begin{figure*}[t]
    \centering

    \begin{subfigure}[t]{0.48\linewidth}
        \centering
        \includegraphics[width=\linewidth]{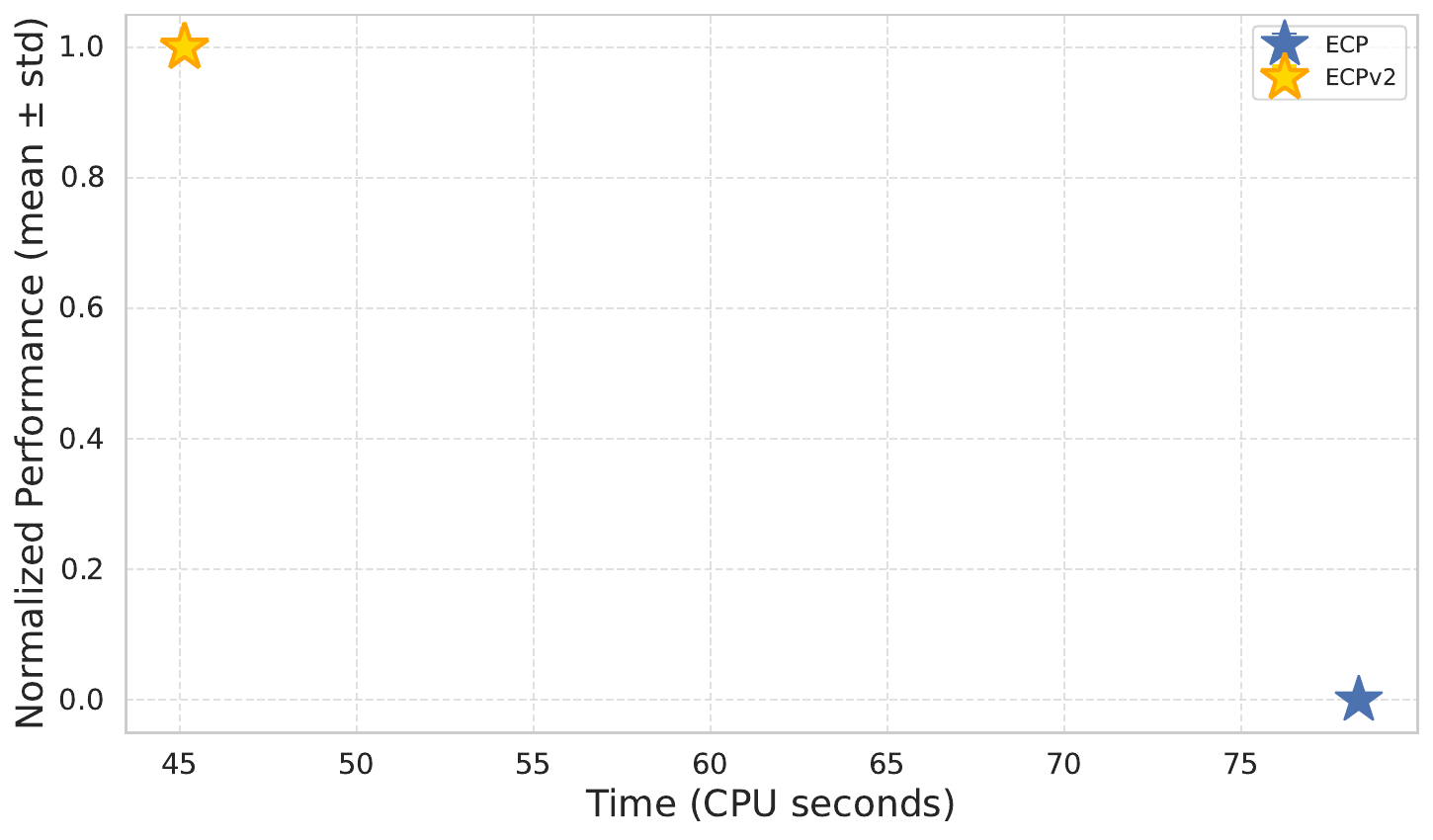}
        \caption{Powell function, $d = 1000$}
        \label{fig:powell}
    \end{subfigure}
    \hfill
    \begin{subfigure}[t]{0.48\linewidth}
        \centering
        \includegraphics[width=\linewidth]{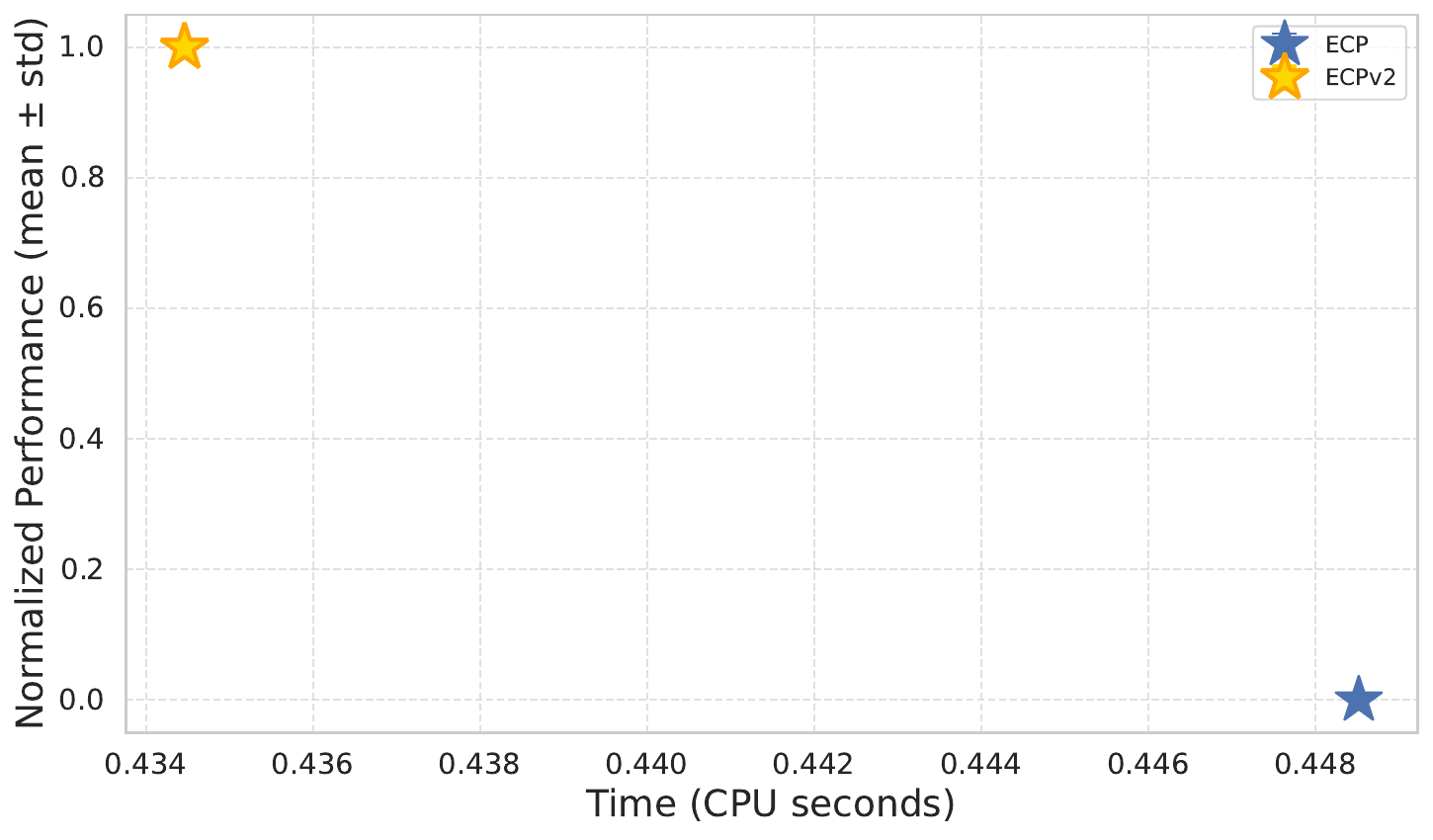}
        \caption{Rosenbrock function, $d = 500$}
        \label{fig:rosenbrock}
    \end{subfigure}
    \caption{Performance comparison of ECP and ECPv2 on high-dimensional optimization problems with a budget $n = 200$, averaged over 100 repetitions.}
    \label{fig:high_dim_ecp}
\end{figure*}

\section{Appendix I: Additional Comparison with Other Global Optimization Algorithms}

While the primary focus of this work is on algorithms tailored for Lipschitz optimization, along with AdaLIPO, AdaLIPO+, PRS, and DIRECT, we also evaluate several widely used global optimization baselines including SMAC3 and Dual Annealing under the same experimental setup as described in the main paper.

\begin{figure*}[t]
    \centering

    \begin{subfigure}[t]{0.48\linewidth}
        \centering
        \includegraphics[width=\linewidth]{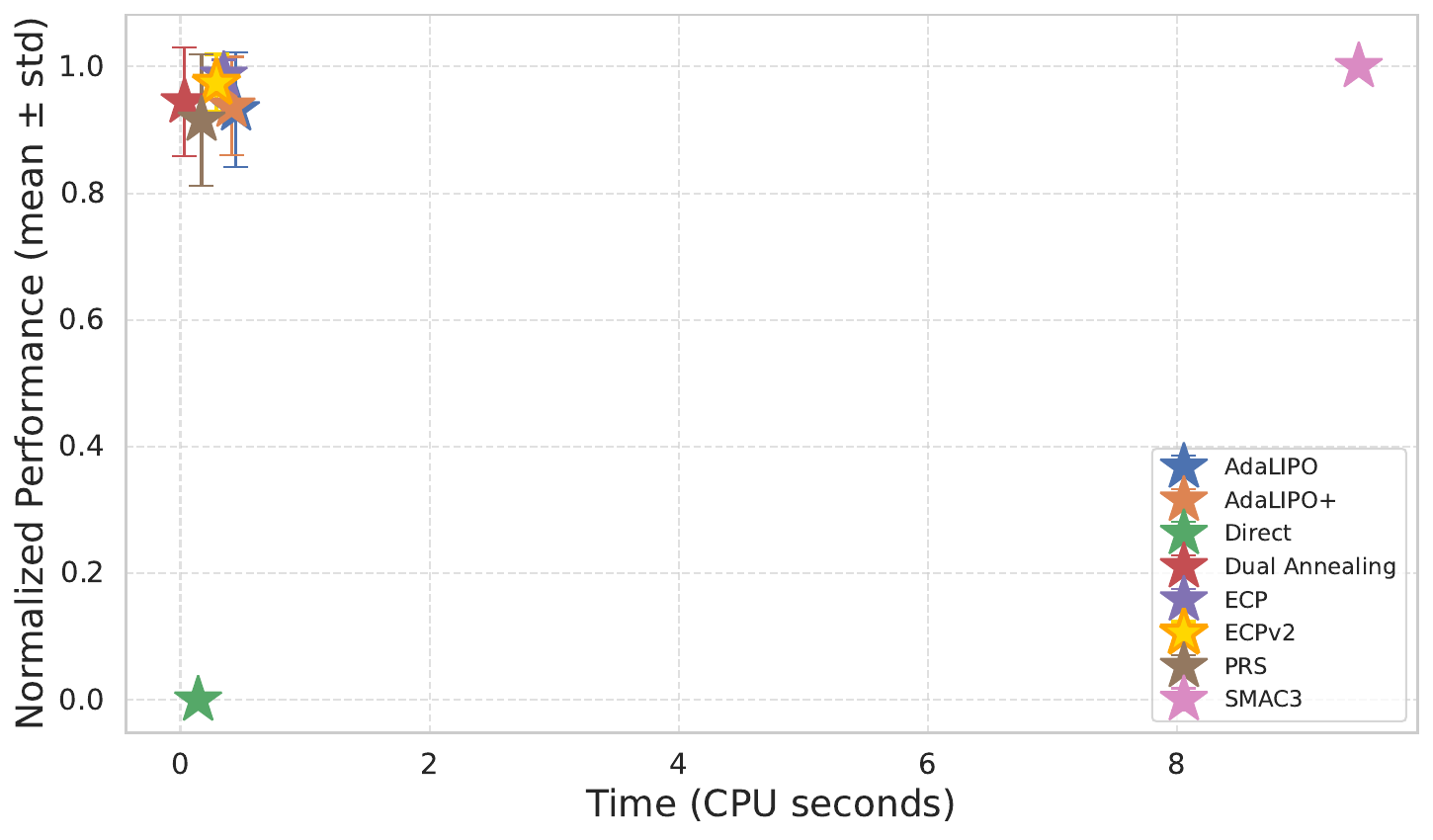}
        \caption{Rosenbrock-$d$, $d \in \{3, 100, 200, 300, 500\}$}
        \label{fig:rosenbrock-baselines}
    \end{subfigure}
    \hfill
    \begin{subfigure}[t]{0.48\linewidth}
        \centering
        \includegraphics[width=\linewidth]{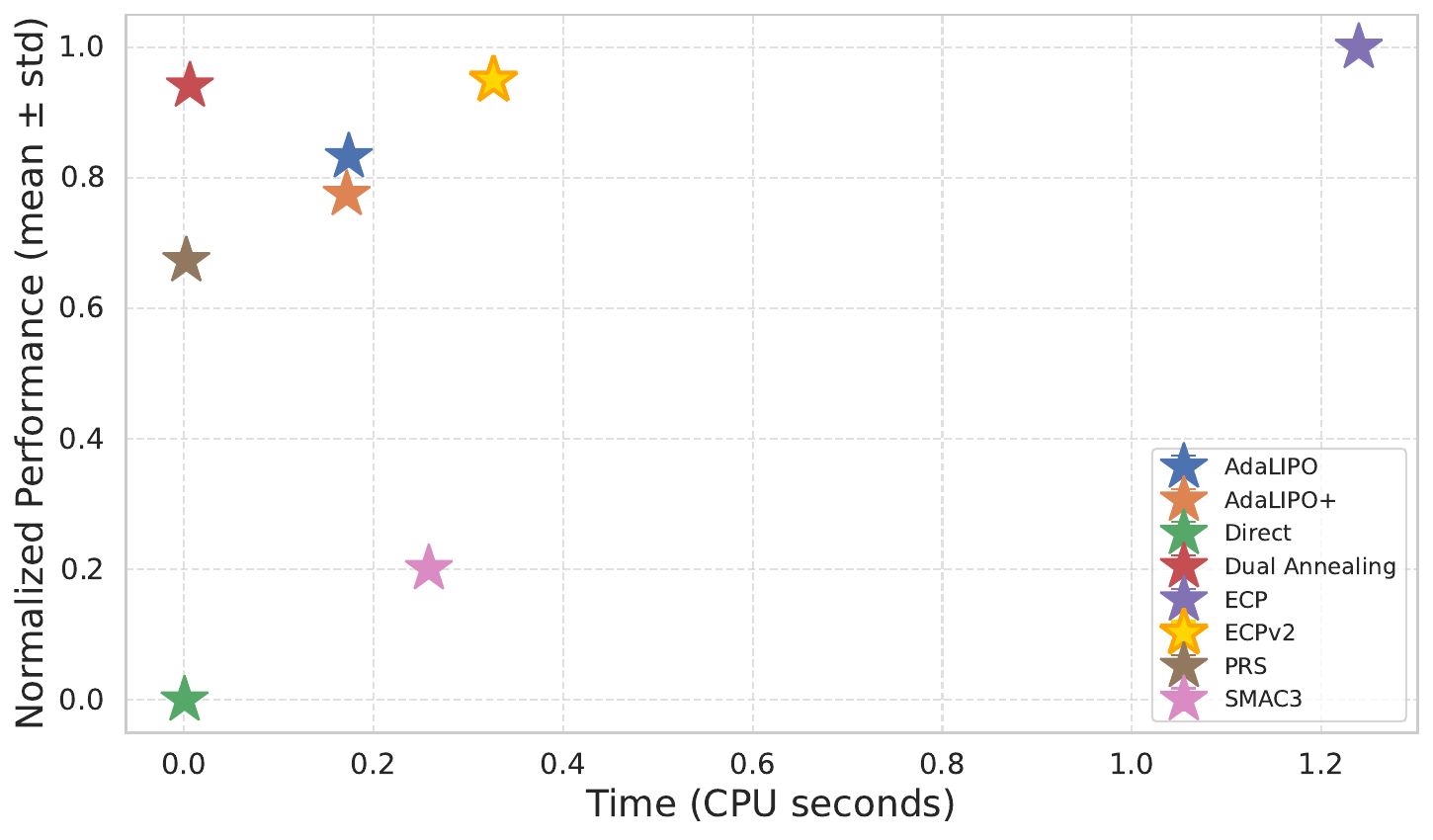}
        \caption{Rastrigin, $d = 2$}
        \label{fig:rastrigin-2D}
    \end{subfigure}

    \caption{
    Comparison of additional global optimization baselines on two benchmark functions: 
    (a) Rosenbrock across increasing dimensionalities, where dimensionality reduction improves performance; 
    (b) Rastrigin in 2D, where dimensionality reduction is skipped since \( d = 2 < d' = \left\lceil \frac{8 \log(5n)}{\delta^2 - \delta^3} \right\rceil \). 
    In this case, acceleration primarily comes from the use of Worst-$m$ and the lower-bound term \( \varepsilon_t^\oslash \).
    Results are aggregated over 100 runs per optimizer per problem. Each point represents the average final best objective value after a modest budget of \( n = 200 \) evaluations.
    }
    \label{fig:appendix-baselines}
\end{figure*}

Among the evaluated methods, SMAC3 demonstrates strong optimization capability but suffers from significant overhead due to its model-based nature, making it much slower than other approaches. DIRECT and PRS by contrast, are extremely fast but tends to perform poorly. DIRECT suffers with high-dimensional settings. 

Dual Annealing, ECP, and ECPv2 all strike a good balance between computational efficiency and optimization performance. Notably, ECPv2 achieves the best overall trade-off, consistently identifying high-quality solutions with minimal wall-clock time. ECPv2 is remarkably much faster than ECP in the rastrigin case, while providing a much competitive wall-clock time.

Both AdaLIPO and AdaLIPO+ also perform competitively, especially on the Rosenbrock family, with a slightly lower performance on rastrigin. Full quantitative results for these comparisons are provided in Figure~\ref{fig:appendix-baselines}.

\section{Appendix J: Experimental Details}

\subsection{Default Hyper-parameters}

We provide recommended default values for the main hyper-parameters of ECPv2, based on both theoretical analysis and empirical validation. These choices are robust across a wide range of optimization problems and were used in all our experiments unless otherwise stated.

To ensure a fair comparison, we avoid fine-tuning the hyper-parameters that are shared with the ECP algorithm. Instead, we retain the same values as those studied in ECP~\cite{fourati25ecp}, thereby preventing any performance advantage that could arise from tuning them.

\begin{itemize}
    \item \textbf{Initial precision} $\varepsilon_1$: Set to $0.01$ (same as ECP). This defines the initial threshold for the acceptance region. A smaller value yields a more conservative start.
    
    \item \textbf{Growth factor} $\tau_{n,d}$: Set as $\tau_{n,d} = \max(1 + \frac{1}{nd}, 1.001)$ (same as ECP). This controls how quickly the acceptance threshold expands after a series of rejections.
    
    \item \textbf{Exploration patience} $C$: Set to $1000$ (same as ECP). This parameter determines the number of consecutive rejections allowed before increasing $\varepsilon_t$.
    
    \item \textbf{Worst-$m$ memory} $m$: Set to $m = 8$. 
    
    \item \textbf{Distortion parameter} $\delta$: Set to $\delta = \frac{2}{3}$. This value minimizes the required projection dimension $d'$ for a given distortion tolerance, as discussed in Appendix D.
    
    \item \textbf{Projection confidence parameter} $\beta$: Set to $\beta = 5$, ensuring that pairwise distances are preserved under random projection with at least 96\% probability.
\end{itemize}

These settings were chosen to provide strong performance without requiring problem-specific tuning. Sensitivity analyses for $\delta$ and $\beta$ are presented in Appendix~D, and analysis for $m$ are presented in Appendix~E and Appendix F.

\subsection{Optimization Algorithms}

For ECP, we use the implementation provided by \cite{fourati25ecp} with the default parameters (https://github.com/fouratifares/ECP).

For DIRECT and Dual Annealing, we use the implementations from SciPy \citep{virtanen2020scipy}, with standard hyperparameters and necessary modifications to adhere to the specified budgets.

For AdaLIPO and AdaLIPO+, we use the implementation provided by \citep{serre2024lipo+}. To ensure fairness, we run AdaLIPO+ without stopping, thereby maintaining the same budget across all methods. Furthermore, since AdaLIPO requires an exploration probability \( p \), we fix it at \( 0.1 \), as done by the authors \citep{malherbe2017global}.

\subsection{Optimization Objectives}

We evaluate the proposed method on various global optimization problems. The problems were designed to challenge global optimization methods due to their highly non-convex curvatures \citep{molga2005test, simulationlib, fourati25ecp}.

\subsection{Comparison Protocol}

We use the same hyperparameters across all optimization tasks without fine-tuning them for each task, as this may not be practical when dealing with expensive functions.

We allocate a fixed budget of function evaluations, denoted by $n$, for all methods. The maximum value over the $n$ iterations is recorded for each algorithm. This maximum is then averaged over 100 repetitions, and both the mean and standard deviation are tracked. Furthermore, the wall-clock time (seconds) spent to finsh the budget $n$ is further tracked and averaged over the runs. We then report the average time spent to achieve the average maximum. 

More evaluations increase the likelihood of finding better points. To ensure that all methods fully utilize the budget, we eliminate any unnecessary stopping conditions, such as waiting times. For example, in AdaLIPO+, we use the variant AdaLIPO+(ns), which continues running even when large rejections occur.

\subsection{Compute and Implementations}

We implemented our method using open-source libraries in \texttt{Python 3.9} and \texttt{NumPy 1.23.4}. Experiments were conducted on two systems: a local machine with an 11th Gen Intel\textsuperscript{\textregistered} Core\texttrademark{} i7-1165G7 CPU @ 2.80\,GHz and 16\,GB RAM, and a high-performance computing (HPC) system with an Intel\textsuperscript{\textregistered} Xeon\textregistered{} Gold 6148 CPU @ 2.40\,GHz and 16\,GB RAM. All experiments were run on CPUs without the use of GPUs. In each plot, all methods are tested under exactly the same settings to ensure fairness.

\end{document}